\documentclass[twoside]{article}

\usepackage[accepted]{aistats2022}
\usepackage{graphicx}
\usepackage{xcolor}
\usepackage{amsmath}
\usepackage{amssymb}
\usepackage{amsthm}
\usepackage{hyperref}
\usepackage{IEEEtrantools}
\usepackage{nicefrac}
\usepackage{lineno}
\usepackage{multirow}
\usepackage{longtable}
\usepackage{makecell}
\usepackage{algorithm}
\usepackage[noend]{algpseudocode}
\usepackage{thmtools, thm-restate}
\usepackage{cancel}
\DeclareMathOperator{\E}{\mathbb{E}}
\DeclareMathOperator{\V}{\mathbb{V}}
\DeclareMathOperator{\Cov}{\textrm{Cov}}

\DeclareMathOperator{\diag}{\textrm{diag}}

\DeclareMathOperator{\sgn}{\textrm{sgn}}

\DeclareMathOperator{\cvec}{\bf c}

\DeclareMathOperator{\evec}{\bf e}
\DeclareMathOperator{\fvec}{\bf f}
\DeclareMathOperator{\gvec}{\bf g}

\DeclareMathOperator{\qvec}{\bf q}
\DeclareMathOperator{\rvec}{\bf r}

\DeclareMathOperator{\wvec}{\bf w}
\DeclareMathOperator{\xvec}{\bf x}
\DeclareMathOperator{\yvec}{\bf y}

\DeclareMathOperator{\Ivec}{\bf I}

\DeclareMathOperator{\Wvec}{\bf W}
\DeclareMathOperator{\Xvec}{\bf X}
\DeclareMathOperator{\Yvec}{\bf Y}

\DeclareMathOperator{\onevec}{\boldsymbol{1}}
\DeclareMathOperator{\zerovec}{\boldsymbol{0}}
\DeclareMathOperator{\thetavec}{\boldsymbol{\theta}}

\DeclareMathOperator{\etavec}{\boldsymbol{\eta}}
\DeclareMathOperator{\sigmavec}{\boldsymbol{\sigma}}
\DeclareMathOperator{\pivec}{\boldsymbol{\pi}}

\DeclareMathOperator{\omegavec}{\boldsymbol{\omega}}

\begin{document}

%

%
\runningauthor{Shivam Garg, Samuele Tosatto, Yangchen Pan, Martha White, A. Rupam Mahmood}

\twocolumn[

  \aistatstitle{An Alternate Policy Gradient Estimator for Softmax Policies}

  \aistatsauthor{Shivam Garg$^{1}$ $\;$ Samuele Tosatto$^{1}$ $\;$ Yangchen Pan$^{2}$ $\;$ Martha White$^{1 *}$ $\;$ A.~Rupam Mahmood$^{1 *}$}

  \aistatsaddress{$^{1}$University of Alberta. $^{2}$Noah's Ark Lab, Huawei (work done while at UofA). \\ $^{*}$CIFAR AI Chair, Alberta Machine Intelligence Institute (Amii). } ]

\begin{abstract}
  Policy gradient (PG) estimators are ineffective in dealing with softmax policies that are sub-optimally saturated, which refers to the situation when the policy concentrates its probability mass on sub-optimal actions. Sub-optimal policy saturation may arise from bad policy initialization or sudden changes in the environment that occur after the policy has already converged. Current softmax PG estimators require a large number of updates to overcome policy saturation, which causes low sample efficiency and poor adaptability to new situations. To mitigate this problem, we propose a novel PG estimator for softmax policies that utilizes the bias in the critic estimate and the noise present in the reward signal to escape the saturated regions of the policy parameter space. Our theoretical analysis and experiments, conducted on bandits and various reinforcement learning environments, show that this new estimator is significantly more robust to policy saturation.
\end{abstract}

\section{INTRODUCTION} \label{sec: introduction}

Policy gradient (PG) algorithms aim to optimize a sequential decision making problem defined on a set of parameterized policies: the policy parameters are optimized using standard stochastic gradient-based update to maximize the net reward received by the agent. A salient feature of the PG methods is that they can be used with both discrete and continuous action spaces, and have been successfully deployed in a variety of real world tasks such as robotics (Haarnoja et al., 2018; Mahmood, et al., 2018) and large scale simulated problems (Berner et al., 2019). 

When the action-space is discrete, policies are typically represented using a categorical distribution parameterized by a softmax function. However, softmax policies have some inherent issues. Even with access to the true gradients, they can be slow in responding to non-stationarity (Hennes et al., 2019) and their performance heavily depends on the initialization of the policy parameters (Li et al., 2021; Mei et al. 2020a; Mei et al. 2020b). Both these problems arise from an issue, which we call \textit{sub-optimal policy saturation}.

Sub-optimal policy saturation refers to the situation when the policy places a high probability mass on sub-optimal actions. An agent with a saturated policy will not be able to explore other actions and may continue to remain in the sub-optimal region if an appropriate measure is not taken. 
Sub-optimal policy saturation is a significant issue that arises in multiple scenarios: (1) Non-stationarity: when the agent is acting in a constantly changing environment, and what was once an optimal strategy may no longer work well.\ 
(2) Pre-training / transfer learning: deep reinforcement learning systems are often pre-trained on different tasks before being used on the main task. The subtle differences in the structure of these tasks might lead to a bad policy initialization which is sub-optimally saturated. And (3) Stochastic updates: policy gradient updates suffer from high variance (Williams, 1992) and therefore, it is possible for a policy to become saturated on sub-optimal actions during the course of learning.

PG methods with softmax policies are particularly susceptible to the issue of saturation. Prior work has suggested different approaches to mitigate this issue: entropic regularization (Ding et al., 2021; Ahmed et al., 2019; Haarnoja et al., 2018; Peters et al., 2010), natural PG algorithms (Kakade, 2001; Peters and Schaal, 2008), and log barrier regularization (Agarwal et al., 2021) can all make the policy more explorative. Explicit reward bonuses can also encourage the exploration of low-density state-action pairs (Auer et al. 2002). However, entropic regularization and explorative bonuses introduce additional terms in the objective, and therefore the resulting optimal policy can be different from the original one. Another solution is to replace the softmax with a function such as the escort transform (Mei et al., 2020b).
In this article, we take a different approach to address this issue. 
Instead of augmenting the optimization objective with bonuses or regularizers, we introduce a PG estimator that inherently helps to escape the sub-optimally saturated regions. Furthermore, the above augmentations can still be applied on top of our estimator.

Our proposed estimator is a simple yet effective approach for dealing with sub-optimally saturated policies thereby making PG algorithms more robust. The classic likelihood ratio estimator for softmax policies, which we call the \textit{regular estimator}, takes a frustratingly large amount of experience to escape sub-optimally saturated policy regions. The regular estimator produces near zero gradients at saturation as both its expectation and variance, even with reward noise, are vanishingly small at those regions. 
In contrast, our proposed estimator, which we call the \textit{alternate estimator}, has a non-zero variance in the same scenario that can be utilized to escape the sub-optimal regions. 
Further, the alternate estimator naturally utilizes the bias in the critic estimate to increase the policy's entropy, thereby encouraging exploration. As the critic estimate improves, this effect reduces, allowing the policy to saturate towards the optimal actions.

\textbf{Contributions:} We present the alternate PG estimator for softmax policies and show that it works by utilizing reward noise as well as the bias in the critic to escape saturated regions in the policy space. We theoretically justify this behavior in the bandit setting complemented by simple and intuitive examples. We also provide empirical verification\footnote{Code for the experiments is available at \url{https://github.com/svmgrg/alternate_pg}.} of these properties on (1) problems with artificially saturated policies and (2) on non-stationary tasks where sub-optimal policy saturation naturally arises. Our results show that this new estimator works well in both the bandit and the MDP settings with various PG algorithms such REINFORCE and Actor-Critic (Sutton et al., 2000), and is superior (or competitive) to the regular estimator under tabular, linear, and neural representations.

\section{PRELIMINARIES} \label{sec: preliminaries}
In reinforcement learning, the decision making task is described using Markov decision processes (MDPs) $\mathcal{M} := (\mathcal{S}, \mathcal{A}, \mathcal{R}, \mu, p, \gamma)$, where $\mathcal{S}$ is the set of states, $\mathcal{A}$ is the discrete set of actions, $\mathcal{R} \subset \mathbb{R}$ is the set of  rewards, $\mu \in \Delta(\mathcal{S})$\footnote{The object $\Delta(\mathcal{X})$ denotes the set of all possible probability distributions over the set $\mathcal{X}$.} is the start state distribution, $p: \mathcal{S} \times \mathcal{A} \to \Delta(\mathcal{S} \times \mathcal{R})$ is the transition dynamics, and $\gamma\in[0,1]$ is the discount factor. The agent maintains a policy $\pi: \mathcal{S} \rightarrow \Delta(\mathcal{A})$ that describes its interaction with the environment. This interaction is represented using the history $H_{T} = \{S_0, A_0, R_1, S_1, A_1, \ldots,  R_T, S_T\}$ where $S_0 \sim \mu$, $A_t \sim \pi(\cdot | S_t)$, and $S_{t+1}, R_{t+1} \sim p(\cdot, \cdot | S_t, A_t)$ for $t \in \{0, \dots, T-1\}$\footnote{In this article, we only consider episodic tasks with the random variable $T$ denoting the episode termination length. However, all the arguments presented here can be extended for continuing tasks as well.}. The agent uses such interactions to learn a policy that maximizes the expected return $\mathcal{J}=\E_\pi[\sum_0^{T-1}\gamma^t R_{t+1}] $. PG methods provide one way to accomplish this task. We now describe these methods in two typical settings.  

\subsection*{Gradient Bandits}
Discrete bandits are a simplification of the full MDP framework. They have no states and at each timestep the agent picks an action $A_t$ from a discrete action set $\mathcal{A}$ and obtains a reward $R_t \sim p(\cdot | A_t)$. In this setting, the agent's goal reduces to maximizing the expected immediate reward $\mathcal{J} := \E_\pi[R_t] =: r_\pi$. Gradient bandit algorithms maximize $\mathcal{J}$ using gradient ascent. Given a softmax policy $\pi_{\thetavec}(a) = e^{\theta_a} / \sum_{b \in \mathcal{A}} e^{\theta_b}$, where $\theta_a$ is the action preference for action $a$, the gradient $\nabla \mathcal{J}$ (\S 2.8, Sutton and Barto, 2018) is given by 
\begin{equation}
  [\nabla_{\thetavec} \mathcal{J}_{\pi}]_i = r(a_i) \cdot \pi(a_i) \cdot (1 - \pi(a_i)), \label{eq: gradient_bandit}
\end{equation}
where $r(a) := \E[R | a]$ is the reward function. However, the agent usually does not have access to $r$ for all actions at the same time and must resort to using sample based gradient estimators $\hat{\gvec}$ such that $\nabla_{\thetavec} \mathcal{J}_\pi = \E_{A \sim \pi; R \sim p(\cdot | A)} [\hat{\gvec}(A, R)]$. A typical choice is to use what we call the \textsl{regular estimator}
\begin{equation}
  [\hat{\gvec}^{\text{REG}}(A, R)]_a = (R - b) (\mathbb{I}(A = a) - \pi(a)), \label{eq: regular_bandit_estimator_scalar} 
\end{equation}
where $\mathbb{I}$ is the indicator function and $b$ is an estimate of the average reward $r_\pi$.

\subsection*{Policy Gradient Methods}
In an episodic MDP, the PG objective is $\mathcal{J}=\E_\pi[\sum_0^{T-1}\gamma^t R_{t+1}]$. The policy gradient theorem (Sutton et al., 2000) states that
\begin{equation}
  \nabla \mathcal{J} = \sum_{s \in \mathcal{S}} \nu_\pi(s) \sum_{a \in \mathcal{A}} \nabla \pi(a | s) q_{\pi}(s, a), \label{eq: pg_main}
\end{equation}
where $\nu_\pi(s) := \sum_{k=0}^\infty \gamma^k \mathbb{P}(S_k = s)$ is the $\gamma$-discounted state occupancy measure under policy $\pi$. Eq. \ref{eq: pg_main} can also be written as an expectation,
\begin{align}
  \nabla \mathcal{J} = \E_{\nu_\pi, \pi} \big[ \underbrace{ \nabla \log \pi(A | S) \big( q_\pi(S, A) - v_\pi(S) \big)}_{=: \gvec^\text{REG}(S, A)} \big]. \label{eq: pg_expectation}
\end{align}
This expression gives us the sample based gradient estimator $\gvec^\text{REG}$. Again, the agent will use $\hat{\gvec}^{\text{REG}}(S, A)$ with value function estimates $\hat{q}_\pi$ and $\hat{v}_\pi$.

\section{ALTERNATE GRADIENT BANDITS}
In this section, we introduce and analyze the alternate PG estimator. We begin by describing the vector notation which makes our calculations succinct and easier to follow. For a $k$-armed bandit problem, let $\pivec, \rvec \in \mathbb{R}^k$ denote the policy and the reward vectors with $[{\pivec}]_k = \pi(a_k)$ and $[\rvec]_k = r(a_k)$. Using these vectors, the expectations of scalar quantities become vector inner products: $\mathcal{J}_{\pi} \equiv r_\pi = \E_\pi[R] = \sum_a \pi(a) r(a) = \pivec^\top \rvec$. Further, the softmax policy can be written as $\pivec = e^{\thetavec} / (\boldsymbol{1}^\top e^{\thetavec})$, where the vector $e^{\thetavec} \in \mathbb{R}^k$ is defined as $[e^{\thetavec}]_a := e^{\theta_a}$. And $\nabla_{\thetavec} \pivec = (\Ivec - \pivec \boldsymbol{1}^\top) \diag(\pivec)$. 
Hence, the regular policy gradient can be derived as follows:
\begin{align}
  \nabla_{\thetavec} \mathcal{J} &= \nabla_{\thetavec} (\pivec^\top \rvec) = (\Ivec - \pivec \boldsymbol{1}^\top) \diag(\pivec) \rvec \nonumber \\
  &= \E_{A \sim \pi} [r(A) (\evec_A - \pivec ) ] \nonumber \\
  &= \E_{A \sim \pi, R \sim p(\cdot | A)} [\underbrace{(R - r_\pi) (\evec_A - \pivec )}_{=: \gvec^\text{REG}(A, R)}], \label{eq: regular_gradient_bandit} 
\end{align}
where $\evec_A \in \mathbb{R}^k$ represents the basis  vector with the $A$th element equal to one. The approximate version of it, as given before in Eq. \ref{eq: regular_bandit_estimator_scalar}, is:
\begin{equation}
  \hat{\gvec}^\text{REG}(A, R) = (R - b) (\evec_A - \pivec), \label{eq: regular_gradient_bandit_with_baseline} 
\end{equation}
where $b$ is an estimate of $r_\pi$. Also note that this estimator remains unbiased: $\nabla \mathcal{J} = \E[\hat{\gvec}^\text{REG}(A, R)]$.

\textbf{The Alternate Gradient Bandit Estimator} is derived by using an alternate form of the true gradient:
\begin{align}
  \nabla_{\thetavec} \mathcal{J}& = (\Ivec - \pivec \boldsymbol{1}^\top) \diag(\pivec) \rvec = \left( \diag(\rvec) - \pivec \rvec^\top \right) \pivec  \nonumber \\
  &= \diag(\rvec) \pivec - \pivec^\top \rvec \pivec = \left( \diag(\rvec) - \pivec^\top \rvec \Ivec \right) \pivec \nonumber \\
  &= \E_{A \sim \pi, R \sim p(\cdot | A)} [( R  - r_\pi ) \evec_A], \label{eq: alternate_stochastic_softmax_grad_vector}
\end{align}
where  $\gvec^\text{ALT}(A, R) := (R  - r_\pi ) \evec_A $ is the \textit{alternate} gradient bandit estimator. Using a baseline $b$ to estimate the average reward $r_\pi$, gives the following estimator:
\begin{equation}
  \hat{\gvec}^\text{ALT}(A, R) = (R - b) \evec_A. \label{eq: alternate_stochastic_softmax_grad_vector_with_baseline}
\end{equation}

\textbf{The alternate estimator is faster per-timestep than the regular estimator:} At each time-step, the regular estimator updates all the components of the action preference vector $\thetavec$, whereas the alternate estimator only updates a single component. Due to this, a gradient bandit algorithm employing the regular estimator will have a per-timestep computational complexity of $\mathcal{O}(|\mathcal{A}|)$. However, as we show in \S \ref{app: tree_sampling}, gradient bandits utilizing the alternate estimator can use a tree based sampling procedure to have a per-time-step complexity of $\mathcal{O}(\log |\mathcal{A}|)$. When the action space is large, this difference becomes substantial. Now we discuss how the alternate estimator utilizes the reward noise and the bias in the critic estimate to escape saturated regions in the policy space. Proofs of the results are given in \S \ref{app: zero_at_corner} and \S \ref{app: attractor_repellor}.

\subsection*{Alternate Estimator Utilizes Reward Noise}

\begin{figure*}[!tbp]
  \centering
  \includegraphics[scale=0.70]{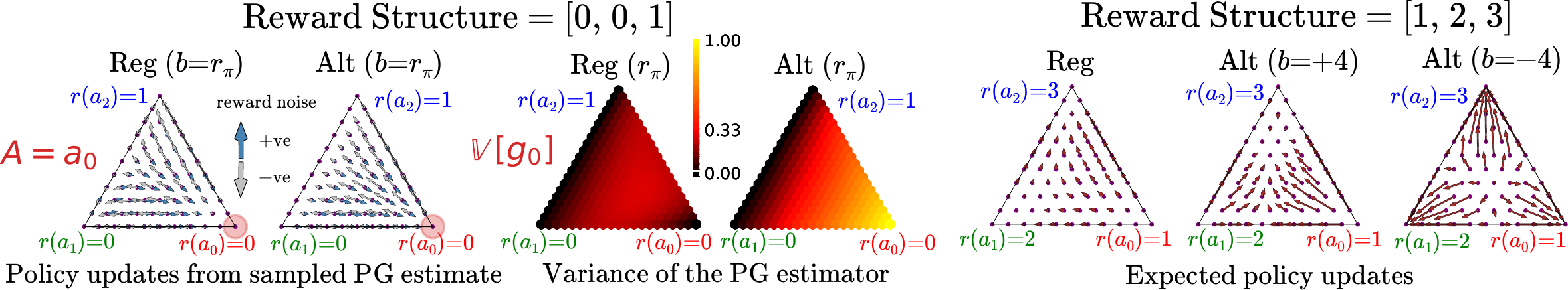}
  \caption{Policy updates and the variance for PG estimators on 3-armed bandits plotted on the probability simplex. The policy-update plots were drawn by updating the policy using the corresponding stochastic PG estimator, assuming that action $a_0$ was taken and a noise of constant magnitude was added to the reward. The blue and grey arrows correspond to the policy change direction for $+ve$ and $-ve$ reward noises. The variance plots show the heatmap for the zeroth element of the gradient estimator. Last three plots show policy updates using expectation of the estimators. We include more of these plots for other settings in \S \ref{app: bandits_intuition}.}
  \label{fig: bandit_intuition}
\end{figure*}

Whenever the policy is saturated, i.e. it places a high probability mass on some actions, the expected gradient becomes close to zero, and consequently the policy weights are not updated. Further, for the regular estimator, the variance at these ``corners'' of the probability simplex is also zero. Therefore, there is neither any gradient signal nor any noise, and the agent is unable to escape the sub-optimal region. In contrast, the alternate estimator, despite being zero in expectation, has a non-zero variance and therefore the agent can utilize the reward noise to escape the sub-optimal region. We illustrate this point with Figure \ref{fig: bandit_intuition} (left and middle) and an accompanying example:
\begin{restatable}{eg}{saturated_pg} \label{eg: saturated_pg}
  Consider a 3-armed bandit with $\mathcal{A} = \{a_0, a_1, a_2\}$, the expected reward $\rvec = [0\;0\;1]^\top$, and the policy $\pivec = [1\;0\;0]^\top$. The rewards are perturbed with a normally distributed noise $\epsilon \sim \mathcal{N}(0, \sigma^2)$, so that upon picking action $A$, the agent receives the reward $r(A) + \epsilon$. Since $r_\pi = 0$,  $R - r_\pi = 0 + \epsilon - 0 = \epsilon$. And because the agent samples $A = a_0$ at each timestep, we obtain
  \begin{align*}
    \gvec^\text{REG}(A, R) &= (R - r_\pi) (\evec_A - \pivec) \\
    &= \epsilon \cdot \Big( [1\;0\;0]^\top - [1\;0\;0]^\top \Big) = [0\;0\;0]^\top, \; \text{and} \\
    \gvec^\text{ALT}(A, R) &= (R - r_\pi) \evec_A = \epsilon \cdot [1\;0\;0]^\top = [\epsilon\;0\;0]^\top.
  \end{align*}
\end{restatable}
This example shows that at a sub-optimal corner, the agent with the alternate estimator effectively does a random walk until it escapes that region and can choose other actions to get a non-zero reward signal. Further, doing a random walk at the optimal corner is not problematic, since there is an attractive gradient signal as soon as the agent moves away from this corner. The following propositions formalize these points.
\begin{restatable}{prop}{expectation_saturated_pg} \label{thm: expectation_saturated_pg}
  Define $\mathcal{I}_c := \{a \;|\; r(a) = c\}$ for some constant $c \in \mathbb{R}$. Assume that $\exists c$ such that $\mathcal{I}_c \neq \emptyset$ and that the policy is saturated on the actions in the set $\mathcal{I}_c$: $\sum_{a \in \mathcal{I}_c} \pi(a) = 1$. Then the expected policy gradient for softmax policies is zero: $\nabla_{\thetavec} \mathcal{J} = \E [\gvec^{\text{REG}}(A, R)] = \E [\hat{\gvec}^{\text{REG}}(A, R)] = \E [\gvec^{\text{ALT}}(A, R)] = \zerovec$.
\end{restatable}

\begin{restatable}{prop}{variance_saturated_pg} \label{thm: variance_saturated_pg}
  Let $\sigma(a)^2 := \V[R | A=a]$ denote the variance\footnote{We compute the variance of vector random variables elementwise; for details, see \S \ref{app: additional_background}.} of the reward corresponding to action $a$. Assume that the policy is saturated on the action $c$, i.e., $\pi(c) = 1$. Then, the variance of the regular PG estimator (with or without the baseline subtraction) is zero: $\V[\gvec^{\text{REG}}(A)] = \V[\hat{\gvec}^{\text{REG}}(A, R)] = \boldsymbol{0}$. Whereas, the variance of the alternate PG estimator is non-zero: $\V[\gvec^{\text{ALT}}(A, R)] = \sigma(c)^2 \evec_c$.
\end{restatable}

Note that the example and the propositions presented above require the policy to lie at the boundary of the probability simplex which is unsatisfiable for softmax policies, as this would require some of the action preferences to be infinitely large. Despite that, using continuity arguments, we can still reason that for the regular estimator, both the expected gradient and its variance vanish in the proximity of the simplex boundary, whereas the alternate estimator will have non-zero variance\footnote{In practice, PG methods are often used with stochastic gradients averaged over mini-batches of data. An increase in batch-size would reduce the variance of alternate, thereby decreasing the random walk effect. However, the biasedness of the alternate estimator is independent of the batch-size and can still help in escaping policy saturation.}. In addition to this, the stochastic update for the alternate estimator is much higher than that for the regular estimator. This can be seen from Figure \ref{fig: bandit_intuition} (left): look at the length of the update arrows near the bottom-right corner on the simplex.

\begin{figure*}[t]
  \centering
  \includegraphics[scale=0.22]{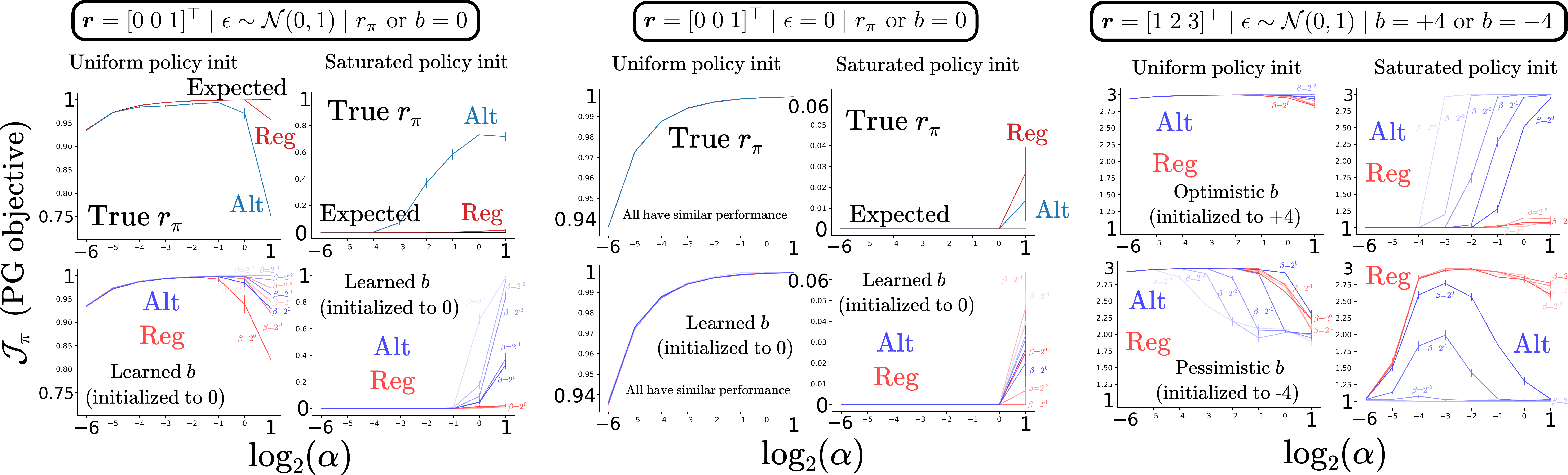}
  \caption{Parameter sensitivity plots for different PG estimators on 3-armed bandit problems. The headers give the specific experimental settings. The algorithms were run for 1000 timesteps and the plots show the mean performance during the last 50 timesteps, averaged over 150 runs.}
  \label{fig: combined_bandit_parameter_sensitivity}
\end{figure*}

\subsection*{Alternate Estimator Utilizes its Biasedness} \label{sec: alternate_utilizes_biasedness}
Apart from utilizing the reward noise, the alternate estimator can also utilize the bias in the critic estimate $b$. To see this, consider a baseline $b \neq r_\pi$. Then the alternate estimator becomes biased:
\begin{equation*}
  \nabla_{\thetavec} \mathcal{J} \neq \E[\hat{\gvec}^\text{ALT}(A, R)] = \E[(R - b) \evec_A] = \pivec \odot(\rvec - b \boldsymbol{1}).
\end{equation*}
However, this biased update $\E[\hat{\gvec}^\text{ALT}(A, R)]$ is not the gradient of any function (refer \S \ref{app: attractor_repellor}). Therefore, in order to understand its behavior, we look at its fixed point and under what conditions it acts as an attractor or a repellor. Figure \ref{fig: bandit_intuition} (right) illustrates this fixed point and its behavior based on how the baseline is initialized for one specific instance of a bandit problem. 

We now state a theorem that formalizes this property. 
Let $r_1 \leq r_2 \leq \cdots \leq r_k$ represent an ordering of the true rewards for the bandit problem. Let $\pi_t = e^{\thetavec^{(t)}} \big/ \boldsymbol{1}^\top e^{\thetavec^{(t)}}$ represent the softmax policy at timestep $t$. Here, $\thetavec^{(t)}$ represents action preference vector at timestep $t$, which is updated using the rule $\thetavec^{(t+1)} = \thetavec^{(t)} + \alpha \E[\hat{\gvec}^\text{ALT}(A, R)]$ for $\alpha>0$.   
\begin{restatable}{lemma}{fixed_pt_biased_pg} \label{thm: fixed_pt_biased_pg}
  \textbf{Fixed points of the biased gradient bandit update.} \textbf{(1)} If there exists an $n \in \{1, 2, \ldots, k-1\}$ such that $r_1 \leq \cdots \leq r_n < b < r_{n+1} \leq \cdots \leq r_k$, then $\E_\pi[\hat{\gvec}^\text{ALT}(A, R)]$ is never equal to zero. \textbf{(2)} If $b = r(a)$ for at least one action, then $\E_\pi[\hat{\gvec}^\text{ALT}(A, R)] = 0$ at any point on the face of the probability simplex given by $\sum_{a \in \mathcal{I}_b} \pi(a) = 1$ with $\mathcal{I}_b := \{a | r(a) = b\}$. \textbf{(3)} If $b < r_1$ or $b > r_k$, then $\E_\pi[\hat{\gvec}^\text{ALT}(A, R)] = 0$ at a point $\pi^*$ within the simplex boundary given by $\pi^*(a) = \frac{1}{r(a) - b} \left( \sum_{c \in \mathcal{A}} \frac{1}{r(c) - b} \right)^{-1}, \; \forall a \in \mathcal{A}$.
\end{restatable}

\begin{restatable}{thm}{fixed_pt_biased_pg_attract_repulse} \label{thm: fixed_pt_biased_pg_attract_repulse}
  \textbf{Nature of the fixed point $\pivec^*$.} Assume that $\pi_t \neq \pi^*$. If the baseline is pessimistic, i.e. $b < r_1$, then for any $\alpha > 0$, the fixed point $\pi^*$ acts as a repellor: $D_{\text{KL}}(\pi^* \| \pi_{t+1}) > D_{\text{KL}}(\pi^* \| \pi_{t})$, where $D_{\text{KL}}$ denotes the KL-divergence. And if the baseline is optimistic, i.e. $b > r_k$, then given a sufficiently small positive stepsize $\alpha$ the fixed point $\pi^*$ acts as an attractor: $D_{\text{KL}}(\pi^* \| \pi_{t+1}) < D_{\text{KL}}(\pi^* \| \pi_{t})$.
\end{restatable}
The above theorem illustrates an important property of the alternate estimator: with an optimistically initialized baseline, the agent is updated towards a more uniform distribution $\pi^*$. And for an optimistic baseline, this distribution has a higher probability of picking the action with the maximum reward as compared to other actions; this can be seen from Figure \ref{fig: bandit_intuition} (right) and the expression for $\pi^*$ in Lemma 2.1. Therefore, if the agent were stuck in a sub-optimal corner of the probability simplex, an optimistic baseline would make its policy more uniform and encourage exploration\footnote{On the flip side, with a pessimistically initialized baseline, the alternate estimator can pre-maturely saturate towards a sub-optimal corner; see Figure \ref{fig: bandit_intuition} (right).}. And even though $\pi^*$ is different from the optimal policy, the agent with an alternate estimator can still reach the optimal policy, because as the agent learns and improves its baseline estimate, the alternate PG estimator becomes asymptotically unbiased. And hopefully by this time, the agent has already escaped the saturated policy region.

\section{BANDIT EXPERIMENTS} \label{sec: bandit_experiments}
In this section, we present experiments with the gradient bandit algorithm on 3-armed bandit tasks, with a fixed reward structure $\rvec$ and a normally distributed reward noise $\epsilon$. The experimental results serve as a demonstration and the verification of the properties of the alternate and the regular estimators discussed above. We trained five different agents on the bandit task. The policy weights were updated using either the expected gradient $\nabla \mathcal{J} = \pivec \odot (\rvec - r_\pi)$, the regular estimator with true $r_\pi$ (Eq.\ \ref{eq: regular_gradient_bandit}) or a learned baseline (Eq.\ \ref{eq: regular_gradient_bandit_with_baseline}), or the alternate estimator with true $r_\pi$ (Eq. \ref{eq: alternate_stochastic_softmax_grad_vector}) or a learned baseline (Eq.\ \ref{eq: alternate_stochastic_softmax_grad_vector_with_baseline}). The baseline was learned using a running average: $b_{t+1} = (1 - \beta) b_t + \beta R_t$. Additional details are given in \S \ref{app: experiments_bandits}.

\textbf{Experiment 1} demonstrates that the alternate estimator performs competitively with the regular estimator for uniform policy initialization and clearly outperforms it in the case of a saturated policy initialization. We set $\rvec = [0\;0\;1]^\top$ with noise $\epsilon \sim \mathcal{N}(0, 1)$. The action preferences were initialized to $\theta_{a} = 0, \; \forall a$ for uniform policy, and to $\theta_{a_0} = 10$ and $\theta_{a_1} = \theta_{a_2} = 0$ for the saturated policy. The baseline $b$ was initialized to zero. 
Figure \ref{fig: bandit_learning_curve} shows the learning curves for the best parameter configuration. Figure \ref{fig: combined_bandit_parameter_sensitivity} (left) shows the sensitivity plots; in particular it shows the final performance of the methods for different parameter settings.

\begin{figure}[!tbp]
  \centering
  \includegraphics[scale=0.30]{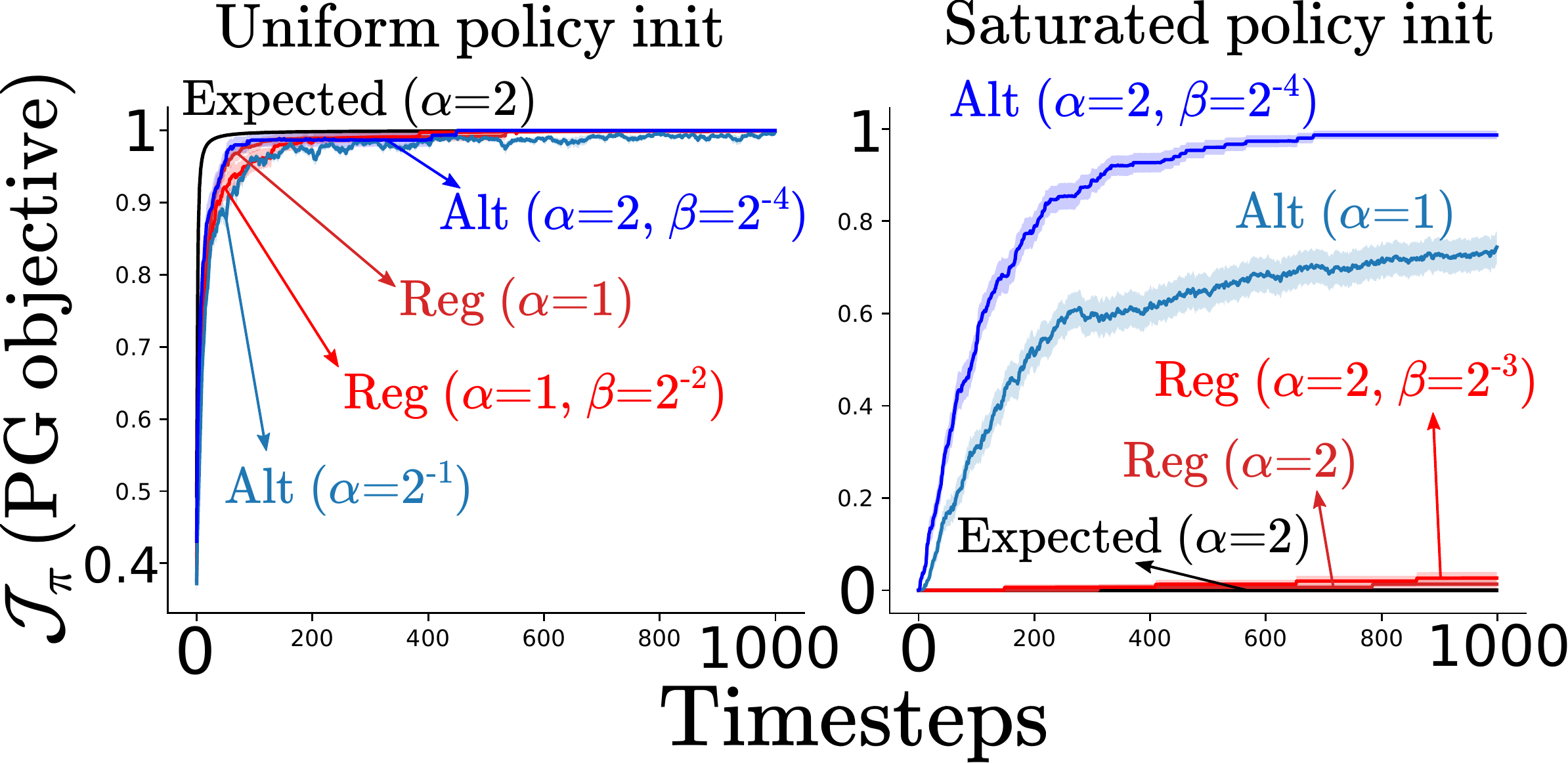}
  \caption{Learning curves for the gradient bandit algorithm. The results were averaged over 150 independent runs; shaded region shows standard error.}
  \label{fig: bandit_learning_curve}
\end{figure}

For saturated policy initialization, we observe that the alternate estimators (both $r_\pi$ and a learned $b$) learned good behaviors, whereas the expected gradient and the regular estimators failed to learn anything. Further, alternate with a learned baseline converged faster than alternate with $r_\pi$. We attribute this result to a better exploration afforded by the baseline.

\textbf{Experiment 2} retained the previous experimental setting, but set the reward noise $\epsilon = 0$. 
Figure \ref{fig: combined_bandit_parameter_sensitivity} (middle) shows the sensitivity plots for both uniform and saturated policy initializations. 
These results show that for uniform initialization, all the methods had a similar performance. For the saturated initializations, none of the methods were able to learn a good policy. This experiment verifies our claim that reward noise helps the alternate estimator to escape the sub-optimal regions in the policy space; without reward noise, the agent is unable to perform a random walk.

\textbf{Experiment 3} studies the effect of initializing the reward baseline optimistically ($b = +4$) or pessimistically ($b=-4$). The reward was set to $\rvec = [1\;2\;3]^\top$ with noise $\epsilon \sim \mathcal{N}(0, 1)$. Figure \ref{fig: combined_bandit_parameter_sensitivity} (right) shows the sensitivity plots for both the estimators. From the plots, we see that optimistic initialization greatly helped the alternate estimator in escaping the saturated policy. Whereas, a pessimistic initialization significantly hampered the performance of the alternate estimator, making it worse than the regular estimator even for the uniform policy case. In contrast, the baseline initialization affected the regular estimator in a milder way. 

\section{ALTERNATE PG ESTIMATOR FOR MDPs} \label{sec: alt_pg_mdp}
The alternate estimator readily extends to MDPs, where it enjoys properties analogous to the bandit case. We start by deriving the regular estimator for MDPs using the vector notation. Given a state $s$, let $\pivec (\cdot | s) \in \mathbb{R}^{|\mathcal{A}|}$ be the policy vector with $[\pivec (\cdot | s)]_a = \pi(a | s)$ and $[\qvec_\pi(s, \cdot)]_a = q_\pi(s, a)$ be the action value vector. The softmax policy for MDPs is $\pi(a | s) = e^{[\thetavec_{\wvec}(s)]_a} \big/ \sum_{b \in \mathcal{A}} e^{[\thetavec_{\wvec}(s)]_b}$, where $\thetavec_{\wvec}: \mathcal{S} \rightarrow \mathbb{R}^{\mathcal{A}}$ denotes the action preference vector, parameterized by $\wvec$, and $[\thetavec_{\wvec}(s)]_a$ is its $a$th element. In practice, $\thetavec_{\wvec}$ could be implemented using a neural network. For brevity, we will drop $\wvec$, i.e. $\thetavec \equiv \thetavec_{\wvec}$. Now, using the chain rule of differentiation (see \S \ref{app: additional_background}) $\nabla_{\wvec} \pivec(\cdot | s) = \big[ \nabla_{\wvec} \thetavec(s) \big] \big[ \nabla_{\thetavec} \pivec(\cdot | s) \big]$, we can re-write the PG theorem (Eq. \ref{eq: pg_main}) to explicitly separate the gradient of the action likelihood $\nabla_{\thetavec} \pivec(\cdot | s)$ and the preference $\nabla_{\wvec} \thetavec(s)$:
\begin{align}
  \nabla \mathcal{J}_{\pi} &= \sum_{s} \nu_\pi(s) \sum_{a} \nabla \pi(a | s) q_\pi(s, a) \nonumber \\
  &= \sum_{s} \nu_\pi(s) \nabla \pivec (\cdot | s) \qvec_\pi(s, \cdot) \nonumber \\
  &= \sum_{s} \nu_\pi(s) \big[ \nabla_{\wvec} \thetavec(s) \big] \nabla_{\thetavec} \pivec(\cdot | s) \qvec_\pi(s, \cdot) \label{eq: explicit_pg_matrix_vec}.
\end{align}
In the analysis that follows, we will analytically work out the gradient of the action likelihood for the softmax policy. We will leave the gradient of the preference which depends on the specific function approximator used and could be computed later using, say, an automatic differentiation package. We can write
\begin{align}
  & \nabla_{\thetavec} \pivec(\cdot | s) \qvec_\pi(s, \cdot) = ( \Ivec - \pivec(\cdot | s) \boldsymbol{1}^\top ) \diag(\pivec(\cdot | s)) \qvec_\pi(s, \cdot) \nonumber \\
  &= \diag(\pivec(\cdot | s)) \qvec_\pi(s, \cdot) - \pivec(\cdot | s) \boldsymbol{1}^\top \diag(\pivec(\cdot | s)) \qvec_\pi(s, \cdot) \nonumber \\
  &= \pivec(\cdot | s) \odot \qvec_\pi(s, \cdot) - \pivec(\cdot | s) \pivec(\cdot | s)^\top \qvec_\pi(s, \cdot) \label{eq: mdp_vector_grad_softmax_mid} \\
  &= \diag(\qvec_\pi(s, \cdot)) \pivec(\cdot | s) - \pivec(\cdot | s) \qvec_\pi(s, \cdot)^\top \pivec(\cdot | s) \nonumber \\
  &= [ \diag(\qvec_\pi(s, \cdot)) - \pivec(\cdot | s) \qvec_\pi(s, \cdot)^\top ] \pivec(\cdot | s) \nonumber \\
  &= \sum_a \pi(a | s) q_\pi(s, a) \big( \evec_A - \pivec(\cdot | s) \big). \label{eq: mdp_primary_vector_grad_softmax}
\end{align}
Putting Eq. \ref{eq: mdp_primary_vector_grad_softmax} into Eq. \ref{eq: explicit_pg_matrix_vec} gives us that $\nabla_{\wvec} \mathcal{J}_{\pi}$
\begin{align}
  &= \sum_{s} \nu_\pi(s) \sum_a \pi(a | s) \nabla_{\wvec} \thetavec(s) \big( \evec_A - \pivec(\cdot | s) \big) q_\pi(s, a) \nonumber \\
  &= \mathop{\E} \Big[ \underbrace{\Big( \nabla_{\wvec} [\thetavec(S)]_A - \sum_{a} \pi(a | S) \nabla_{\wvec} [\thetavec(S)]_a \Big) h_\pi(S, A)}_{=: \gvec^{\text{REG}}(S, A)} \Big], \nonumber
\end{align}
where we introduced the advantage function $h_\pi(S, A) := q_\pi(S, A) - v_\pi(S)$ by subtracting an action independent $v_\pi$ from $q_\pi$. Note that despite the unfamiliar expression, this is the regular estimator: $\gvec^{\text{REG}}(S, A) = \nabla_{\wvec} \log \pi(A | S) h_\pi(S, A)$ (refer \S \ref{app: theoretical_analysis_mdp}).

\begin{figure*}[!tbp]
  \centering
  \includegraphics[scale=0.23]{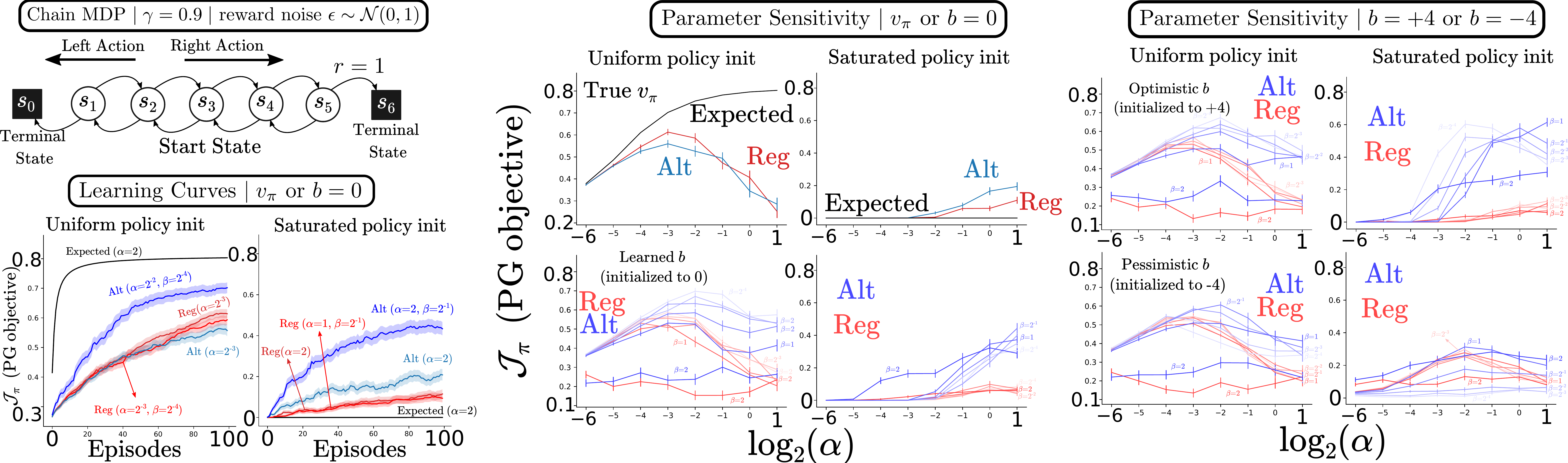}
  \caption{Learning curves and parameter sensitivity plots for REINFORCE on the five state chain MDP. Each agent was run for 100 episodes. The parameter sensitivity plots show the mean performance during the last 10 episodes, averaged over 150 runs. Policy was initialized, for all states, to $\theta_{\texttt{right}} = 0$; and $\theta_{\texttt{left}} = 0$ (for uniform), $\theta_{\texttt{left}} = 3$ (for saturated), and $\theta_{\texttt{left}} = 1$ (for saturated + pessimistic baseline).}
  \label{fig: tabular_combined}
\end{figure*}

\textbf{The alternate estimator for MDPs.} From Eq. \ref{eq: mdp_vector_grad_softmax_mid},
\begin{align}
  & \nabla_{\thetavec} \pivec(\cdot | s) \qvec_\pi(s, \cdot) \nonumber \\
  &= \diag(\qvec_\pi(s, \cdot)) \pivec(\cdot | s) - \pivec(\cdot | s)^\top \qvec_\pi(s, \cdot) \pivec(\cdot | s) \nonumber \\
  &= \big[ \diag(\qvec_\pi(s, \cdot)) - v_\pi(s) \Ivec \big] \pivec(\cdot | s) \nonumber \\
  &= \sum_a \pi(a | s) \Big(q_\pi(s, a) - v_\pi(s) \Big) \evec_A, \label{eq: mdp_alternate_vector_grad_softmax}
\end{align}
where we used that $v_\pi(s) = \sum_{a} \pi(a|s) q_\pi(s, a) = \pivec(\cdot | s)^\top \qvec_\pi(s, \cdot)$. Now put Eq. \ref{eq: mdp_alternate_vector_grad_softmax} in Eq. \ref{eq: explicit_pg_matrix_vec} to obtain
\begin{align}
  \nabla_{\wvec} \mathcal{J}_{\pi} &= \sum_{s} \nu_\pi(s) \sum_a \pi(a | s) \nabla_{\wvec} \thetavec(s) \evec_A h_\pi(S, A) \nonumber \\
  &= \E \Big[ \underbrace{\nabla_{\wvec} [\thetavec(S)]_A h_\pi(S, A)}_{=: \gvec^{\text{ALT}}(S, A)} \Big]. \label{eq: mdp_alternate_stochastic_loss}
\end{align}
The alternate gradient estimator is equivalent to the gradient of the softmax policy, if we drop the normalization constant; i.e. $\nabla_{\wvec} \log \frac{e^{[\thetavec(S)]_A}}{1} = \nabla_{\wvec} [\thetavec(S)]_A$. Also a note that even though the regular and the alternate PG estimators are equal in expectation, in general they are not equal for an arbitrary state-action pair: $\gvec^{\text{REG}}(S, A) \neq \gvec^{\text{ALT}}(S, A)$. Finally, it is straightforward to adapt the alternate estimator given in Eq. \ref{eq: mdp_alternate_stochastic_loss} to work with different PG methods such as REINFORCE, Actor-Critic, TRPO, or PPO (see \S \ref{app: diff_alternate_estimators}).

\section{TABULAR-MDP EXPERIMENTS} \label{sec: tabular_experiments}
In this section, we demonstrate that the alternate estimator for MDPs enjoys similar benefits as the alternate estimator for bandits. We use the chain environment, shown in Figure \ref{fig: tabular_combined} (bottom left), which is an episodic MDP where the expected rewards are zero everywhere expect at the rightmost transition. We train five different agents using REINFORCE (Williams, 1992). All the agents maintain a tabular policy (and in some cases additionally a tabular value function estimate). The policy is learned either using the expected PG update, or using the REINFORCE algorithm using the regular estimator (true $v_\pi$ or learned baseline) or the alternate estimator (again with true $v_\pi$ or learned baseline). The critic is estimated using Monte-Carlo sampling. Further details (including algorithm pseudocode) and additional experiments are given in \S \ref{app: experiments_tabular}.

\textbf{Experiment 4} showcases that the alternate PG estimator is competitive with the regular estimator in case of uniform policy initialization and superior to it in case of sub-optimally saturated initialization. The results are essentially the same as that from the bandit experiments. Figure \ref{fig: tabular_combined} (bottom-left) shows the learning curves for the best performing parameter configuration, and Figure \ref{fig: tabular_combined} (middle) shows the stepsize sensitivity corresponding to the final performance for each parameter setting. Observe that for the saturated initialization case, alternate with baseline is vastly superior to all the methods, and alternate with true $v_\pi$ is a little better than the regular estimators. We attribute the superior performance of alternate estimator with baseline to the bias of the estimator combined with utilizing the noise in the returns, which probably allows it to have better exploration; this point is also reinforced from the stepsize sensitivity plots: smaller $\beta$ values have higher final performance. We show the performance of the methods with no reward noise in \S \ref{app: experiments_tabular}: the results are similar to the bandit setting.

\textbf{Experiment 5} studies the performance of the alternate estimator with optimistically and pessimistically initialized baselines. Figure \ref{fig: tabular_combined} (right) shows that having an optimistic baseline significantly helps the alternate estimator; in particular, compare the performance of the alternate estimator on saturated policies with ($b = +4$) and without ($b = 0$) optimism. Whereas, a pessimistic baseline hurts its performance. Also note that the alternate estimator with the optimistic baseline prefers smaller critic stepsizes (allowing it to enjoy the optimism for longer), and vice-versa for the pessimistic baseline. Even though we found that, in general, an optimistic baseline helps the performance of the alternate estimator, for certain MDPs (where a uniform policy is bad for exploration; see \S \ref{app: uniform_policy_bad}) it can also hurt its performance.

\begin{figure*}[!tbp]
  \centering
  \includegraphics[scale=0.21]{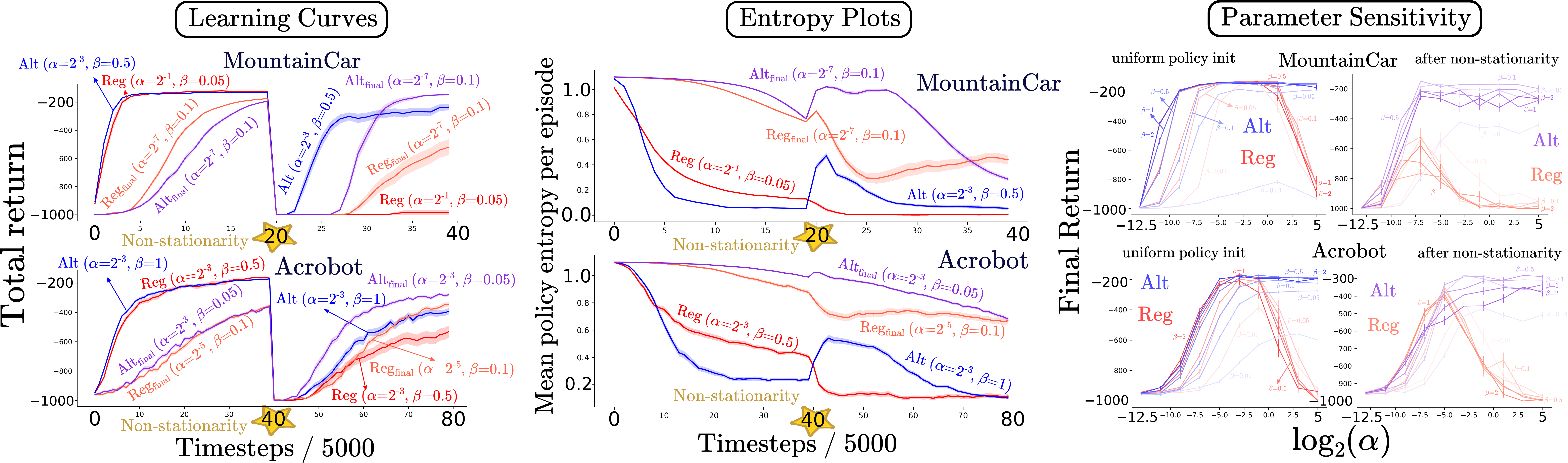}
  \caption{Learning curves, entropy plots, and sensitivity plots for Online Actor-Critic with linear function approximation on MountainCar (200k timesteps) and Acrobot (400k timesteps). Learning curves are for the best performing stepsize configurations at the time non-stationarity was introduced and at the final timestep. Entropy plots show entropy of the policy on the exact states encountered during each episode. The sensitivity plots show the mean performance during the last 5000 timesteps. All the results were averaged over 50 runs.}
  \label{fig: linear_nonstationary_combined}
\end{figure*}

\section{FUNCTION APPROXIMATION EXPERIMENTS}
In this section, we demonstrate that the desirable properties of the alternate estimator continue to hold with function approximation. We run experiments with the online Actor-Critic algorithm (Degris et al., 2012) on classical MDPs using linear function approximation (with tile-coding) and neural networks. Unlike the previous experiments, we did not saturate the policy artificially. Instead, we induced non-stationarity in the environment which naturally lead to sub-optimally saturated policies. Additional details and an extensive empirical analysis (including artificially induced policy saturation) are given in \S \ref{app: experiments_linear} and \S \ref{app: experiments_neural}.

\textbf{Experiment 6} studies the online AC algorithm on MountainCar (Moore, 1990) and Acrobot (Sutton, 1996) control tasks. To induce non-stationarity in either task, we switched the \texttt{left} and \texttt{right} actions after half-time. Figure \ref{fig: linear_nonstationary_combined} (left) shows the learning curves for the best parameter configuration. For each estimator, we selected two sets of parameter configurations: one that had the best performance right before the non-stationarity hit, and another one that had the best performance at the end of the experiment (denoted by a \texttt{final} in the subscript). For MountainCar, the alternate estimator is superior to the regular estimator for both sets of stepsizes. Remarkably, the best performing parameter set for regular (\textcolor{red}{red curve}) at timestep 100k was unable to recover from the non-stationarity; whereas alternate (\textcolor{blue}{blue curve}), despite having similar performance as regular at 100k timestep, was able to recover. On Acrobot, the difference in performance is still there but relatively smaller.

We attribute the superior performance of the alternate estimator to the bias in the critic estimate. To see this, consider MountainCar: at 100k timesteps, the critic would have converged to predict a return of about $-200$. But when the non-stationarity hit, the agent would have started receiving returns much lower than $-200$. But this means that at that time, the critic estimate became optimistic and encouraged exploration by pushing the policy towards a more uniform distribution. Figure \ref{fig: linear_nonstationary_combined} (middle) corroborates this point: the policy entropy for alternate (but not for regular) jumps right around the timestep when the non-stationarity hit. Figure \ref{fig: linear_nonstationary_combined} (right) shows the parameter sensitivity plots for these tasks.

\begin{figure}[!tbp]
  \centering
  \includegraphics[scale=0.22]{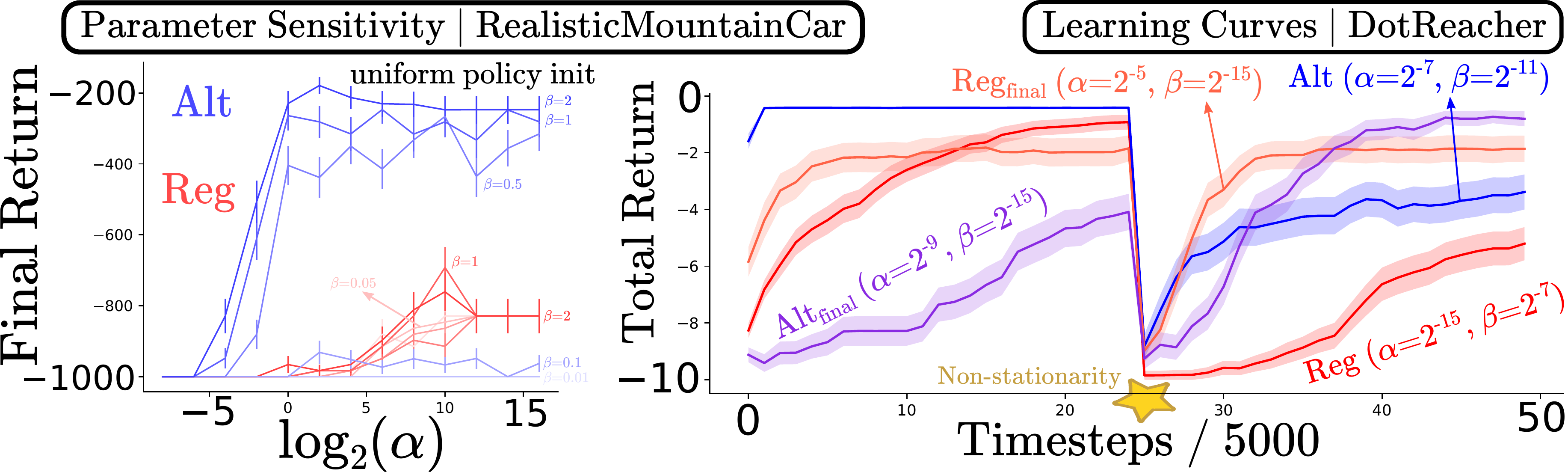}
  \caption{\textbf{(Left)} Parameter sensitivity of online AC with linear function approximation on RealisticMountainCar (100k timesteps, 50 runs); plot shows average performance in the last 5000 steps. \textbf{(Right)} Learning curves for online Actor-Critic with neural networks on DotReacher (100k timesteps, 50 runs).}
  \label{fig: realistic_plus_dotreacher}
\end{figure}

\begin{figure*}[!tbp]
  \centering
  \includegraphics[scale=0.20]{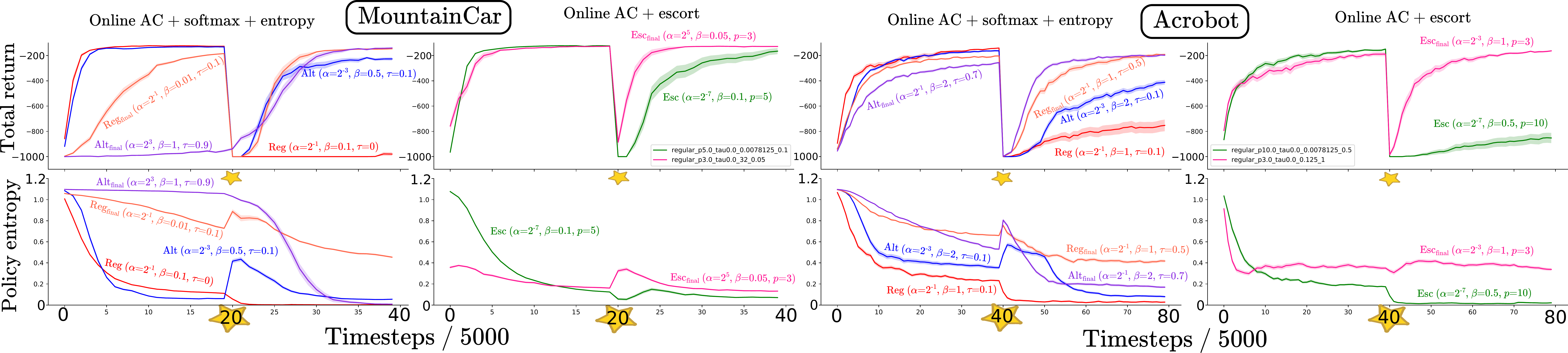}
  \caption{Learning curves and entropy plots for Online Actor-Critic with the softmax policy parameterization (regular and alternate estimators) and the escort policy parameterization. We chose the best performing stepsize configurations at the time non-stationarity was introduced and an additional set of stepsizes that performed best at the final timestep. Entropy plots show mean entropy of the policy on the exact states encountered during each episode. All the results were averaged over 50 runs.}
  \label{fig: lc_escort_entropy}
\end{figure*}

\textbf{Experiment 7} considers a difficult control environment called RealisticMountainCar (see \S \ref{app: experiments_linear}). As Figure \ref{fig: realistic_plus_dotreacher} (left) shows even with a uniformly initialized policy, the alternate estimator performed superior to the regular estimator; clearly, the optimistically initialized baseline ($b = 0$) helped in exploration.

\textbf{Experiment 8} indicates that the alternate estimator works with neural policies as well. Figure \ref{fig: realistic_plus_dotreacher} (right) shows the learning curves of online AC with neural networks on the DotReacher task (see \S \ref{app: experiments_neural}). The results follow the same pattern as Experiment 6: the alternate estimator is still competitive or superior, although the difference is less pronounced.

\textbf{Experiment 9} studies two approaches designed to handle policy saturation: entropy regularization (Ahmed et al., 2019) and escort transform (Mei et al., 2020b). Entropy regularization aims to overcome policy saturation by encouraging exploration. It augments the rewards received from the environment with the action-entropy of the policy scaled by a constant parameter $\tau$. A larger $\tau$ encourages a more uniform policy. Escort transform (ET) is an alternative to softmax policies, given by $\pi(a|s) = |[\thetavec(s)]_a|^p \big/\sum_{a'} |[\thetavec(s)]_{a'}|^p$, where $p \geq 1$ is the escort parameter. ET has inherently better convergence properties for escaping policy saturation as compared to softmax. For further details, see \S \ref{app: additional_details_online_ac} and \S \ref{app: escort_pg}.

We trained online AC with entropy regularization and softmax policy (both regular and alternate estimators) and separately with ET on non-stationary MountainCar and Acrobot tasks. Figure \ref{fig: lc_escort_entropy} shows the learning curves for the best performing parameter settings. These learning curves show that adding an entropy bonus didn't seem to improve the performance of the regular estimator (\textcolor{red}{red curve}). Regular with entropy performed worse even when compared to alternate without an entropy bonus (from Experiment 6). On the other hand, ET seemed to handle the non-stationarity somewhat well. Particularly, on MountainCar, escort (\textcolor{green}{green curve}) slightly outperformed alternate (\textcolor{blue}{blue curve}). However, on Acrobot, alternate remained superior. From the parameter sensitivity plots for these experiments (given in \S \ref{sec: escort_entropy}), we further found that alternate outperforms regular with entropy for almost all the parameters considered, We also found that escort transform's performance is somewhat sensitive to the value of the parameter $p$, which highlights another advantage of the alternate estimator: it has one less parameter than escort transform.

\section{RELATED WORKS}
We now discuss three prior works that bear a strong resemblance to the alternate estimator. Young (2019) studied the unbiased alternate estimator with a focus on saving compute by only calculating the gradient of a single action preference  per timestep (weakly related to \S \ref{app: tree_sampling}). In contrast, we focused on using the biased alternate estimator to overcome policy saturation. The alternate estimator looks very similar to the NeuRD (Hennes et al., 2019) update: $\nabla_{\wvec} \mathcal{J} = \sum_{s} \nu_\pi(s) \sum_a \nabla_{\wvec} [\thetavec(S)]_A h_\pi(S, A)$. However, NeuRD is missing the policy weighting from Eq. \ref{eq: mdp_alternate_stochastic_loss} (which would make it more aggressive than the unbiased alternate estimator). Also, NeuRD doesn't consider learned (and hence biased) baselines and was motivated from a game theoretic perspective. The attractive/repulsive effect of baselines on the alternate estimator (Theorem 1) was previously reported (termed as the \textit{non-commital} property) for the regular estimator by Chung et al. (2021). (Our experiments in Figure \ref{fig: intuition_bandits_reward123_noise0_opti_pessi} also demonstrate this for the regular estimator.) However, because of the unbiasedness of the regular estimator, this effect is much weaker compared to what is experienced by the (biased) alternate estimator.

\section{CONCLUSIONS}
We proposed an alternate policy gradient estimator for softmax policies that, as we demonstrated theoretically and empirically, effectively utilizes the reward noise and the bias in the critic to escape sub-optimally saturated regions in the policy space. Our analysis, conducted on multiple bandit and MDP tasks, suggests that this estimator works well with different PG algorithms and different function approximation schemes. The alternate estimator makes existing PG methods more viable for non-stationary problems, and by extension for many practical real-life control tasks.

\textbf{Limitations and Future Work:} Inherently, the scope of this work is limited to discrete action spaces and the log-likelihood family of PG estimators. In this article, we provided a sufficient proof-of-concept that motivates a further study of the alternate estimator. Our theoretical analysis only considered the tabular bandit setting where the results were asymptotic in nature. It would be useful to extend them to MDPs with finite time analysis. A particularly interesting direction is to see, possibly using a two-timescale analysis, how the bias in the alternate estimator affects the policy optimization dynamics as the baseline estimate becomes more accurate. Empirically, our results did not explore the batch PG algorithms such as PPO (Schulman et al., 2017) or A2C (Mnih et al., 2016). We had good reasons for avoiding them: batch algorithms have a large number of interacting parts which introduces tangential factors making a careful empirical analysis much more difficult. Having said that, an important future work would be to study the alternate estimator with these algorithms on large MDPs.

\section*{Acknowledgments}
The authors gratefully acknowledge funding from the Canada CIFAR AI Chairs program, the Reinforcement Learning and Artificial Intelligence (RLAI) laboratory, the Alberta Machine Intelligence Institute (Amii), and the Natural Sciences and Engineering Research Council (NSERC) of Canada. Shivam Garg also gratefully acknowledges support from Csaba Szepesv\'{a}ri during part of the duration of this project, and thanks Compute Canada for providing computational resources.
\section*{References}
\addcontentsline{toc}{section}{References}

\medskip  

\begin{list}{}{%
    \setlength{\topsep}{0pt}%
    \setlength{\leftmargin}{0.2in}%
    \setlength{\listparindent}{-0.2in}%
    \setlength{\itemindent}{-0.2in}%
    \setlength{\parsep}{\parskip}%
  }%

\item[] Agarwal, A., Kakade, S. M., Lee, J. D., Mahajan, G. (2021). On the theory of policy gradient methods: Optimality, approximation, and distribution shift. \textit{Journal of Machine Learning Research, 22}(98), 1-76.

\item[] Agarwal, A., Jiang, N., Kakade, S. M., Sun, W. (2019). Reinforcement Learning: Theory and Algorithms. \textit{CS Dept., UW Seattle, Seattle, WA, USA, Technical Report.}

\item[] Ahmed, Z., Le Roux, N., Norouzi, M., Schuurmans, D. (2019). Understanding the impact of entropy on policy optimization. In \textit{International conference on machine learning} (pp. 151-160).
  
\item[] Auer, P., Cesa-Bianchi, N., Fischer, P. (2002). Finite-time analysis of the multiarmed bandit problem. \emph{Machine learning, 47}(2), 235-256.

\item[] Barnard, E. (1993). Temporal-difference methods and Markov models. \textit{IEEE Transactions on Systems, Man, and Cybernetics, 23}(2), 357-365.

\item[] Berner, C., Brockman, G., Chan, B., Cheung, V., D\c{e}biak, P., Dennison, C., Farhi, D., Fischer, Q., Hashme, S., Hesse, C., J\'{o}zefowicz, R. (2019). Dota 2 with Large Scale Deep Reinforcement Learning. \textit{arXiv preprint arXiv:1912.06680.}

\item[] Chung, W., Thomas, V., Machado, M. C., Le Roux, N. (2021). Beyond variance reduction: Understanding the true impact of baselines on policy optimization. In \textit{International Conference on Machine Learning} (pp. 1999-2009).

\item[] Ciosek, K., Whiteson, S. (2020). Expected Policy Gradients for Reinforcement Learning. \textit{Journal of Machine Learning Research, 21}(52).

\item[] Cormen, T. H., Leiserson, C. E., Rivest, R. L.,  Stein, C. (2009). \textit{Introduction to algorithms.} MIT press.

\item[] Degris, T., Pilarski, P. M., Sutton, R. S. (2012). Model-free reinforcement learning with continuous action in practice. In \textit{American Control Conference (ACC)} (pp. 2177-2182). IEEE.

\item[] Ding, Y., Zhang, J., Lavaei, J. (2021). Beyond Exact Gradients: Convergence of Stochastic Soft-Max Policy Gradient Methods with Entropy Regularization. \textit{arXiv preprint arXiv:2110.10117.}

\item[] Hennes, D., Morrill, D., Omidshafiei, S., Munos, R., Perolat, J., Lanctot, M., Gruslys, A., Lespiau, J.B., Parmas, P., Duenez-Guzman, E., Tuyls, K., (2019). Neural replicator dynamics. \textit{arXiv preprint arXiv:1906.00190.}

\item[] Kakade, S. M. (2001). A natural policy gradient. \textit{Advances in neural information processing systems, 14.}

\item[] Kingma, D. P., Ba, J. (2014). Adam: A method for stochastic optimization. \textit{arXiv preprint arXiv:1412.6980.}

\item[] Lafferty, J., Liu, H., Wasserman, L. (2008). Concentration of measure. \textit{Available online: \url{http://www.stat.cmu.edu/~larry/=sml/Concentration.pdf}.}
  
\item[] Li, G., Wei, Y., Chi, Y., Gu, Y., Chen, Y. (2021). Softmax policy gradient methods can take exponential time to converge. \textit{arXiv preprint arXiv:2102.11270.}

\item[] Mahmood, A. R., Korenkevych, D., Vasan, G., Ma, W., Bergstra, J. (2018). Benchmarking reinforcement learning algorithms on real-world robots. \textit{Conference on robot learning.}

\item[] Mei, J., Xiao, C., Szepesvari, C., Schuurmans, D. (2020a). On the global convergence rates of softmax policy gradient methods. \textit{In International Conference on Machine Learning (pp. 6820-6829).}

\item[] Mei, J., Xiao, C., Dai, B., Li, L., Szepesv\'{a}ri, C., Schuurmans, D. (2020b). Escaping the Gravitational Pull of Softmax. \textit{Advances in Neural Information Processing Systems, 33.}

\item[] Mei, J., Dai, B., Xiao, C., Szepesv\'{a}ri, C., Schuurmans, D. (2021). Understanding the Effect of Stochasticity in Policy Optimization. arXiv preprint arXiv:2110.15572.

\item[] Mnih, V., Badia, A. P., Mirza, M., Graves, A., Lillicrap, T., Harley, T., Silver, D., Kavukcuoglu, K. (2016). Asynchronous methods for deep reinforcement learning. In \textit{International conference on machine learning} (pp. 1928-1937).

\item[] Moore, A. W. (1990). Efficient memory-based learning for robot control.

\item[] Pardo, F., Tavakoli, A., Levdik, V., Kormushev, P. (2018). Time limits in reinforcement learning. In \textit{International Conference on Machine Learning} (pp. 4045-4054).

\item[]  Paszke, A., Gross, S., Massa, F., Lerer, A., Bradbury, J., Chanan, G., Killeen, T., Lin, Z., Gimelshein, N., Antiga, L., Desmaison, A., Kopf, A., Yang, E., DeVito, Z., Raison, M., Tejani, A., Chilamkurthy, S., Steiner, B., Fang, L., Bai, J., Chintala, S. (2019). Pytorch: An imperative style, high-performance deep learning library. \textit{Advances in neural information processing systems, 32}, 8026-8037.
  
\item[] Peters, J., Schaal, S. (2008). Natural actor-critic. \textit{Neurocomputing}, 71(7-9), 1180-1190.

\item[] Peters, J., Mulling, K., Altun, Y. (2010). Relative entropy policy search. In \textit{Proceedings of the Twenty-Fourth AAAI Conference on Artificial Intelligence.}

     \item[] Schulman, J., Levine, S., Abbeel, P., Jordan, M., Moritz, P. (2015). Trust region policy optimization. In \textit{ International conference on machine learning} (pp. 1889-1897).
       
\item[] Schulman, J., Wolski, F., Dhariwal, P., Radford, A., Klimov, O. (2017). Proximal policy optimization algorithms. \textit{arXiv preprint arXiv:1707.06347.}

\item[] Singh, S. P., Sutton, R. S. (1996). Reinforcement learning with replacing eligibility traces. \textit{Machine learning, 22}(1), 123-158.
  
\item[] Sutton, R. S. (1996). Generalization in reinforcement learning: Successful examples using sparse coarse coding. \textit{Advances in neural information processing systems}, 1038-1044.

\item[] Sutton, R. S., McAllester, D. A., Singh, S. P., Mansour, Y. (2000). Policy gradient methods for reinforcement learning with function approximation. In \textit{Advances in neural information processing systems} (pp. 1057-1063).

\item[] Sutton, R. S., Barto, A. G. (2018). \textit{Reinforcement Learning: An Introduction,} Second Edition. MIT Press.
  
\item[] Williams, R. J. (1992). Simple statistical gradient-following algorithms for connectionist reinforcement learning. \textit{Machine learning, 8}(3-4), 229-256.

\item[] Young, K. (2019). Learning What to Remember: Strategies for Selective External Memory in Online Reinforcement Learning Agents. \textit{University of Alberta, M.Sc. Thesis.}

\end{list}
\onecolumn \makesupplementtitle
\appendix
\tableofcontents

\newpage \section{Minor Background} \label{app: additional_background}
\subsection{Differentiation of Vector Variables} \label{sec: differentiation_vector_variables}
In this section, we first explain the vector notation we use and then state some useful vector differentiation rules. We assume that $\thetavec \in \mathbb{R}^m$ is a vector and $\theta_i$ denotes the $i$th element of $\thetavec$, $\cvec \in \mathbb{R}^m$ is a vector independent of $\thetavec$, $\diag(\thetavec)$ denotes the $m \times m$ diagonal matrix with $\theta_j$ as the $j$th element on the main diagonal, the function $\alpha: \mathbb{R}^m \rightarrow \mathbb{R}$ defined as $\alpha(\thetavec) := \alpha(\theta_1, \theta_2, \ldots, \theta_m) \in \mathbb{R}$ is a scalar function dependent on the vector $\thetavec$, the function $\xvec: \mathbb{R}^m \rightarrow \mathbb{R}^n$ defined as $\xvec(\thetavec) := [ x_1(\thetavec) \; x_2(\thetavec) \; \ldots \; x_n(\thetavec) ]^\top \in \mathbb{R}^n$ is a vector function of $\thetavec$, and $\yvec(\xvec(\thetavec)) \in \mathbb{R}^p$ is a vector function which depends on $\xvec$ and consequently on $\thetavec$ as well. Finally, for a function $f:\mathbb{R} \rightarrow \mathbb{R}$, we define $f(\thetavec) := [ f(\theta_1) \; f(\theta_2) \; \cdots \; f(\theta_m) ]^\top$ as the vector obtained by applying the function $f$ to the vector $\thetavec$ elementwise. For example, $e^{\thetavec} = [ e^{\theta_1} \; e^{\theta_2} \; \cdots \; e^{\theta_m} ]^\top$.

The gradient of a scalar $\alpha(\thetavec)$ with respect to $\thetavec$ is the vector defined as
\begin{equation*}
  \nabla_{\thetavec} \alpha(\thetavec) := \begin{bmatrix} \frac{\partial \alpha(\thetavec)}{\partial \theta_1} & \frac{\partial \alpha(\thetavec)}{\partial \theta_2} & \ldots & \frac{\partial \alpha(\thetavec)}{\partial \theta_m} \end{bmatrix}_{m\times1}^\top.
\end{equation*}
We overload the gradient operator to take gradient of a vector $\xvec(\thetavec)$ with respect to $\thetavec$ as follows:
\begin{IEEEeqnarray*}{lCl}
  \nabla_{\thetavec} \xvec(\thetavec) &:=& \begin{bmatrix} \nabla_{\thetavec} x_1(\thetavec) & \nabla_{\thetavec} x_2(\thetavec) & \ldots & \nabla_{\thetavec} x_n(\thetavec) \end{bmatrix}_{m \times n} \\
  &\equiv& \begin{bmatrix} \frac{\partial x_1(\thetavec)}{\partial \theta_1} & \frac{\partial x_2(\thetavec)}{\partial \theta_1} & \cdots & \frac{\partial x_n(\thetavec)}{\partial \theta_1} \\ \frac{\partial x_1(\thetavec)}{\partial \theta_2} & \frac{\partial x_2(\thetavec)}{\partial \theta_2} & \cdots & \frac{\partial x_n(\thetavec)}{\partial \theta_2} \\ \vdots & \vdots & \ddots & \vdots \\ \frac{\partial x_1(\thetavec)}{\partial \theta_m} & \frac{\partial x_2(\thetavec)}{\partial \theta_m} & \cdots & \frac{\partial x_n(\thetavec)}{\partial \theta_m} \end{bmatrix}_{m \times n}.
\end{IEEEeqnarray*}
According to the above definition, $\nabla_{\thetavec} \xvec(\thetavec)$ is the transpose of the usual Jacobian matrix, and readily reduces to the gradient vector if $\xvec(\thetavec)$ is a scalar (i.e. $n=1$). Further, using this convention, if $\thetavec$ is a scalar (i.e. $m=1$), then
\begin{equation*}
  \nabla_\theta \xvec(\theta) = \begin{bmatrix} \frac{\partial x_1(\theta)}{\partial \theta} & \frac{\partial x_2(\theta)}{\partial \theta} & \cdots & \frac{\partial x_n(\theta)}{\partial \theta} \end{bmatrix}_{1 \times n}.
\end{equation*}
Finally, if both $\xvec(\thetavec)$ and $\thetavec$ are scalars (i.e. $m = n = 1$), then
\begin{equation*}
  \nabla_\theta x(\theta) = \frac{\partial x(\theta)}{\partial \theta}.
\end{equation*}

We are now ready to state some useful propositions, along with their proofs, regarding the differentiation of vector variables.

\begin{restatable}{prop}{dot_product_grad} \label{prop: dot_product_grad}
  $\nabla_{\thetavec} \big( \thetavec^\top \cvec \big) = \nabla_{\thetavec} \big( \cvec^\top \thetavec \big) = \cvec.$
\end{restatable}

\begin{proof}
  \begin{equation*}
    \frac{\partial}{\partial \theta_k} \thetavec^\top \cvec = \frac{\partial}{\partial \theta_k} \sum_{i=1}^m \theta_i c_i = c_k \qquad \Rightarrow \qquad \nabla_{\thetavec} \thetavec^\top \cvec = \cvec.
  \end{equation*}
\end{proof}

\begin{restatable}{prop}{scalar_vector_grad} \label{prop: scalar_vector_grad}
  $\nabla_{\thetavec} \Big[\alpha(\thetavec) \xvec(\thetavec)\Big] = \xvec(\thetavec) \nabla_{\thetavec} \alpha(\thetavec)^\top + \alpha(\thetavec) \nabla_{\thetavec} \xvec(\thetavec).$
\end{restatable}

\begin{proof}
  \begin{IEEEeqnarray*}{lrCl}
    & \Big(\nabla_{\thetavec} \Big[\alpha(\thetavec) \xvec(\thetavec)\Big]\Big)_{ij} &=& \frac{\partial}{\partial \theta_i} \Big[ \alpha(\thetavec)x_j(\thetavec)\Big] = \frac{\partial \alpha(\thetavec)}{\partial \theta_i} x_j(\thetavec) + \alpha(\thetavec) \frac{\partial x(\thetavec)_j}{\partial \theta_i} \\
    \Rightarrow \quad & \nabla_{\thetavec} \Big[\alpha(\thetavec) \xvec(\thetavec)\Big]  &=& \nabla_{\thetavec} \alpha(\thetavec) \xvec(\thetavec)^\top + \alpha(\thetavec) \nabla_{\thetavec} \xvec(\thetavec).
  \end{IEEEeqnarray*}
\end{proof}

\begin{restatable}{prop}{chain_rule_grad} \label{prop: chain_rule_grad}
  $\nabla_{\thetavec} \yvec(\xvec(\thetavec)) = \Big[ \nabla_{\thetavec} \xvec(\thetavec) \Big] \Big[ \nabla_{\xvec(\thetavec)} \yvec(\xvec(\thetavec)) \Big].$
\end{restatable}

\begin{proof}
  We show this result taking the total derivative\footnote{See \url{https://mathworld.wolfram.com/TotalDerivative.html}.} of $\yvec$ with respect to $\thetavec$.
  \begin{IEEEeqnarray*}{lrCl}
    & \Big(\nabla_{\thetavec} \yvec(\xvec(\thetavec))\Big)_{ij} &=& \frac{\partial y_j(\xvec(\thetavec))}{\partial \theta_i} = \sum_{k=1}^m \frac{\partial y_j(\xvec(\thetavec))}{\partial x_k(\thetavec)} \frac{\partial x_k(\thetavec)}{\partial \theta_i} \\
    \Rightarrow \quad & \Big(\nabla_{\thetavec} \yvec(\xvec(\thetavec))\Big)_{\cdot j} &=& \sum_{k=1}^m \frac{\partial y_j(\xvec(\thetavec))}{\partial x_k(\thetavec)} \nabla_{\thetavec} x_k(\thetavec) = \nabla_{\thetavec} \xvec(\thetavec) \nabla_{\xvec(\thetavec)} y_j(\xvec(\thetavec)) \\
    \Rightarrow \quad & \nabla_{\thetavec} \yvec(\xvec(\thetavec)) &=& \Big[ \nabla_{\thetavec} \xvec(\thetavec) \Big] \Big[ \nabla_{\xvec(\thetavec)} \yvec(\xvec(\thetavec)) \Big],
  \end{IEEEeqnarray*}
  where $\Big(\nabla_{\thetavec} \yvec(\xvec(\thetavec))\Big)_{\cdot j}$ denotes the $j$th column of the matrix $\nabla_{\thetavec} \yvec(\xvec(\thetavec))$.
\end{proof}

\begin{restatable}{prop}{elementwise_grad} \label{prop: elementwise_grad}
  Let $f$ be a function, applied elementwise to a vector $\thetavec$, i.e.
  \begin{equation*}
    f(\thetavec) := [f(\theta_1) \; f(\theta_2) \; \cdots \; f(\theta_m) ]^\top.
  \end{equation*}
  Then
  \begin{equation*}
    \nabla_{\thetavec} f(\thetavec) = \diag\left( \begin{bmatrix} \frac{\partial f(\theta_1)}{\partial \theta_1} & \frac{\partial f(\theta_2)}{\partial \theta_2} & \cdots & \frac{\partial f(\theta_m)}{\partial \theta_m} \end{bmatrix}^\top \right).
  \end{equation*}
\end{restatable}

\subsection{Expectation and Variance of Vector Variables}
We define the expectation of a vector random variable (or equivalently that of a vector function of random variables, as discussed in the previous section) in the standard way: for a $d$-dimensional random variable $\Xvec$, $\E[\Xvec] := \sum_{\Xvec} p(\Xvec) \Xvec = \sum_{x_1, x_2, \ldots, x_d} p(X_1=x_1, X_2=x_2, \ldots, X_d=x_d) \Xvec$. Note that this definition is equivalent to taking elementwise expectation, i.e $\E[\Xvec] \equiv \begin{bmatrix} \E[X_1] & \E[X_2] & \cdots & \E[X_d] \end{bmatrix}^\top$. Using this definition, the expectation of a vector function, also see \S \ref{sec: differentiation_vector_variables}, $\fvec(\cdot)$ operating on a (scalar or vector) random variable $\Yvec$, is given by $\E[\fvec(\Yvec)] = \sum_{\Yvec} p(\Yvec) \fvec(\Yvec)$.

In a similar way, the variance and covariance of vector random are computed elementwise as well: $\V[\Xvec] = \E[\Xvec \odot \Xvec] - \E[\Xvec] \odot \E[\Xvec]$ and $\Cov(\Xvec, \Yvec) = \E[\Xvec \odot \Yvec] - \E[\Xvec] \odot \E[\Yvec]$, where $\Xvec$ and $\Yvec$ are vector random variables (they are not matrices!). Also, to make the notation less cumbersome, we will often use $\xvec^2 := \xvec \odot \xvec$ for any vector $\xvec$; therefore, we can equivalently write $\V[\Xvec] = \E[\Xvec^2] - \E[\Xvec]^2$.

\section{Analysis of the Alternate Estimator in the Bandit Setting} \label{app: theoretical_analysis_bandits}
In this section, we begin by discussing the properties of the alternate bandit estimator in the bandit setting. In particular, we show that the alternate gradient bandit estimator has a per-timestep computational complexity of $\log(|\mathcal{A}|)$. We also give proofs of the statements regarding how the alternate estimator utilizes the reward noise and the bias in the critic estimate. After that, we first present another example (similar to Example \ref{eg: saturated_pg}) that shows how an optimistic baseline makes the policy more uniform. Finally, we give the complete policy update and variance plots (analogous to those in Figure \ref{fig: bandit_intuition}) for three armed bandit problems with different reward structures. 

\subsection{Per-timestep Computational Complexity} \label{app: tree_sampling}
We now show that the per-timestep cost of a discrete gradient bandit algorithm using the alternate estimator has a weaker dependence on the number of actions than the regular estimator. The gradient bandit algorithms involves multiple operations during each timestep. At the start of learning, the agent has to initialize the action preferences. Then at each timestep, it needs to compute the gradient estimate, update the action preferences and the policy, and sample an action from the policy. The sampling process itself requires the agent to maintain a sampling data structure which also needs to be updated each time the policy changes. The total per-timestep computational cost of the algorithm depends on all of these operations, which we summarize in Table \ref{table: sampling_for_gradient_bandits}.

\begin{table}[!hbt]
  \centering
  \begin{tabular}{c|c|c}
    & \textbf{Regular Estimator} & \textbf{Alternate Estimator} \\
    \textbf{Computational Process} & \textbf{+} & \textbf{+} \\
    & \textbf{Vose's Alias Sampling} & \textbf{Tree Sampling} \\
    \hline \hline
    Initialize $\thetavec$ & $\mathcal{O}(|\mathcal{A}|)$ & $\mathcal{O}(|\mathcal{A}|)$ \\
    \hline
    Compute $\hat{\gvec}$ and update $\thetavec$ & $\mathcal{O}(|\mathcal{A}|)$ & $\mathcal{O}(1)$ \\
    Compute $\pivec$ & $\mathcal{O}(|\mathcal{A}|)$ & -- \\
    Update the sampling data structure & $\mathcal{O}(|\mathcal{A}|)$ & $\mathcal{O}(\log |\mathcal{A}|)$ \\
    Sample $A \sim \pi$ & $\mathcal{O}(1)$ & $\mathcal{O}(\log |\mathcal{A}|)$ \\
    \hline \hline
    Total cost & $\mathcal{O}(|\mathcal{A}| + n_{\text{iter}} |\mathcal{A}|)$ & $\mathcal{O}(|\mathcal{A}| + n_{\text{iter}} \log |\mathcal{A}|)$
  \end{tabular}
  
  \caption{The key steps involved, along with their respective computational cost, while running the gradient bandit algorithm with the regular or the alternate estimator. We denote the total number of iterations by $n_{\text{iter}}$. Also, it is reasonable to assume that $n_{\text{iter}} \gg |\mathcal{A}|$.} 
  \label{table: sampling_for_gradient_bandits}
\end{table}

The regular estimator updates each component of the action preference vector $\thetavec$ at each timestep, which has a complexity that is linear in the number of actions. Let us assume, for illustration purposes, that the agent with the regular estimator uses Vose's alias sampling method\footnote{For a reference on different sampling methods, see \url{https://www.keithschwarz.com/darts-dice-coins/}.} for obtaining actions from the policy. As a result, the agent needs to build an \textit{alias table} which takes a linear time. Once this table is built, the agent can then sample an action in constant time from this table. However, updating the alias table each time has a linear cost as well. Therefore, the total per-timestep cost for this agent becomes $\mathcal{O}(|\mathcal{A}|)$. In contrast, the alternate estimator updates only a single component of the action preference vector at each timestep. Because of this, a gradient bandit instance using the alternate estimator can be sped up to have a total per-time-step complexity of $\mathcal{O}(\log |\mathcal{A}|)$. To achieve this, we employ a sampling procedure based on a balanced tree data-structure, which we call the \textit{sampling tree}. Although, the sampling tree method needs logarithmic time for sampling a single action, it can be updated in logarithmic time as well if only a single action preference changes at each timestep (as opposed to the alias method which would require a linear time even if only a single action preference were to change).

\subsubsection*{Using a Balanced Binary Tree for Sampling Actions}
The sampling tree is a balanced binary tree (for a reference on our notation, see Chapter 12, Cormen et al., 2009). It has one node corresponding to each action that stores a reference and the exponentiated action preference corresponding to that action. It also stores the sum of all the exponentiated action preferences corresponding to the actions lying in the left subtree of this node. Figure \ref{fig: tree} shows a schematic diagram of such a node, and one possible sampling tree\footnote{The sampling tree is not necessarily a search tree. However, one can apply procedures similar to those discussed here without changing the asymptotic computational complexity, to balanced search trees such as the Red-black trees (Chapter 13, Cormen et al., 2009) to obtain a sampling tree that is also a search tree. This might be required if one needs to maintain a list over the actions ordered by their probability; say, to pick the greedy action at each timestep.}. The cost for building the sampling tree for the first time is linear in the number of actions. The agent would need to initialize the action preferences, create a balanced tree (using Algorithm \ref{alg: build_sampling_tree}), and then calculate the aggregated exponentiated preferences of actions in the left-subtree (using Algorithm \ref{alg: left_cumulant}). And all of these operations are clearly linear in the size of the action space. 

\begin{figure}[!tbp]
  \centering
  \includegraphics[scale=0.4]{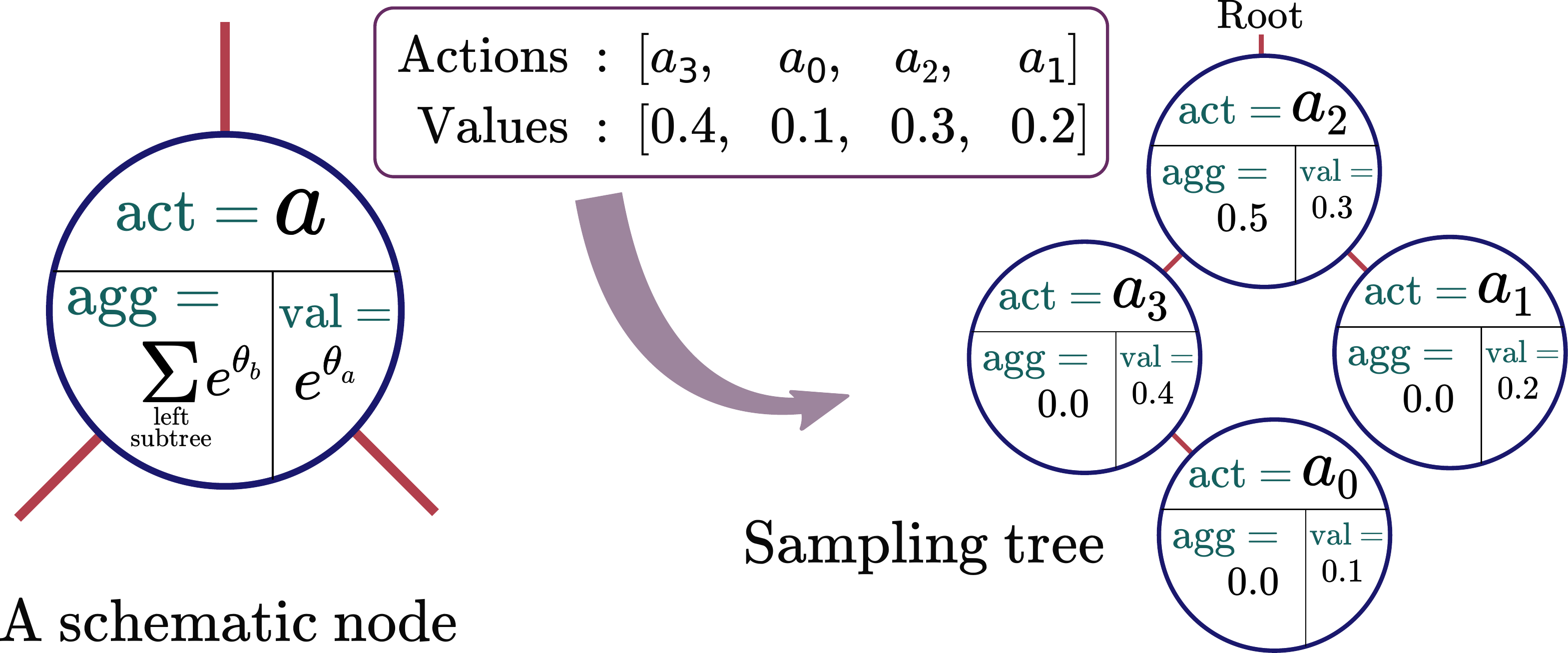}
  \caption{A typical sampling tree node and one possible tree for a softmax policy with four actions. The values refer to exponentiated action preference $e^{\theta_a}$ for that particular action.} 
  \label{fig: tree}
\end{figure}

We summarize the complete operation of building the sampling tree in Algorithm \ref{alg: building_procedure}, and further explain it with an example. Assume that we are given a randomly sorted list of actions $a_{\text{list}} = \{a_3, a_0, a_2, a_1\}$. Let us arbitrarily initialize our action preference list to $\theta_{\text{list}} = \{0.4, 0.1, 0.3, 0.2\}$. Now according to Algorithm \ref{alg: building_procedure}, we will randomly assign either $a_0$ or $a_2$ as the head node; we choose $a_2$ as the head node. Next, we create a left sub-tree from $\{a_3, a_0\}$, this time randomly choosing $a_3$ to be the head node and consequently $a_0$ as its right child. And finally, we create the right sub-tree from $a_1$. This gives us the exact tree shown in Figure \ref{fig: tree}. Note that the nodes in the resulting sampling tree are not ordered in any fashion, and nor is the tree a search tree based on the action preferences. The only required property is that it has to be balanced. 
\begin{algorithm}[!hbt]
  \caption{Full Procedure for Initializing the Sampling Tree}
  \label{alg: building_procedure}
  
  \begin{algorithmic}
    \State Input $a_{\text{list}} = \{a_{i_0}, a_{i_1}, \ldots, a_{i_{n-1}}\}$, where $i_0, i_1, \ldots, i_{n-1}$ is some permutation of $0, 1, \ldots, n-1$
    \State Initialize $\theta_{\text{list}} = \{\theta_{i_0}, \theta_{i_1}, \ldots, \theta_{i_{n-1}}\}$    
    \State head node of the tree $\leftarrow$ \textproc{Build-Sampling-Tree}($a_{\text{list}}$, $\theta_{\text{list}}$, $nil$)
    \State Call \textproc{Compute-Aggregate}(head node of the tree)
  \end{algorithmic}
\end{algorithm}

\begin{algorithm}[!hbt]
  \caption{Building the Sampling Tree}
  \label{alg: build_sampling_tree}
  
  \begin{algorithmic}
    \Function{Build-Sampling-Tree}{$a_{\text{list}}$, $\theta_{\text{list}}$, parent}
    \State $n \leftarrow \text{size}(a_{\text{list}})$
    \If{$n = 0$}
    \State \Return $nil$
    \EndIf
    \If{$n$ is odd} \texttt{\textcolor{purple}{$\;$\# odd number of actions}}
    \State $m \leftarrow (n - 1)/2$
    \Else \texttt{\textcolor{purple}{$\;$\# even number of actions}}
    \State Randomly set $m$ equal to one of $(n - 1)/2$ or $(n + 1)/2$
    \EndIf
    \State Create a new node $z$
    \State $z.act \leftarrow a_{\text{list}}[m]$
    \State $z.val \leftarrow \exp(\theta_{\text{list}}[m])$
    \State $z.p \leftarrow \text{parent}$
    \State \texttt{\textcolor{purple}{$\;$\# assume $\text{list}[i \rightarrow j] := \{\text{list}[i], \text{list}[i+1], \ldots, \text{list}[j-1]\}$}}
    \State $z.left \leftarrow$ \textproc{Build-Sampling-Tree}($a_{\text{list}}[0 \rightarrow m], \theta_{\text{list}}[0 \rightarrow m], z$) 
    \State $z.right \leftarrow$ \textproc{Build-Sampling-Tree}($a_{\text{list}}[m+1 \rightarrow n], \theta_{\text{list}}[m+1 \rightarrow n], z$)
    \State \Return{$z$}
    \EndFunction
  \end{algorithmic}
\end{algorithm}

\begin{algorithm}[!hbt]
  \caption{Calculate Left Sub-tree Aggregate}
  \label{alg: left_cumulant}
  
  \begin{algorithmic}
    \Function{Compute-Aggregate}{$z$}
    \If{$z = nil$} \texttt{\textcolor{purple}{$\;$\# if $z$ is a leaf node}}
    \State \Return{0}
    \EndIf
    \State $L \leftarrow$ \textproc{Compute-Aggregate}($z.left$)
    \State $R \leftarrow$ \textproc{Compute-Aggregate}($z.right$)
    \State $z.agg \leftarrow L$
    \State \Return{$L + z.val + R$}
    \EndFunction
  \end{algorithmic}
\end{algorithm}

 We now discuss how the agent uses this data structure for sampling actions. The sampling process itself takes logarithmic time. To sample an action from the policy, the agent first samples a key from a uniform distribution\footnote{It is standard to assume that sampling a single value from a uniform distribution takes constant time.} over the interval $\left[0, \sum_{i=1}^{|\mathcal{A}|} e^{\theta_i} \right)$, and then calls the \textproc{Action-Select} function (see Algorithm \ref{alg: action_select}) on the sampling tree with this key. Note that the \textproc{Action-Select} function runs in $\log(|\mathcal{A}|)$ time for a balanced sampling tree. The next proposition explains the functioning of the \textproc{Action-Select} function. 
  \begin{algorithm}[!hbt]
    \caption{Selecting an Action from the Sampling Tree}
    \label{alg: action_select}
    \begin{algorithmic}
      \State Generate $x \sim \mathcal{U} \left[ 0, \sum_{i=1}^{|\mathcal{A}|} e^{\theta_i} \right)$ \texttt{\textcolor{purple}{$\;$\# $\mathcal{U}$ refers to the uniform distribution}}
        \State Call \textproc{Action-Select}(head node of the tree, $x$)
        \State \texttt{\textcolor{purple}{$\;$\# function definition}}
        \Function{Action-Select}{$z, x$}
        \If{$x < z.agg$}
        \State \Return \textproc{Action-Select}($z.left, x$)
        \ElsIf{$z.agg \leq x < z.agg + z.val$}
        \State \Return{$z.act$}
        \Else \texttt{\textcolor{purple}{$\;$\# this refers to the case when $x \geq z.agg + z.val$}}
        \State $x_{new} \leftarrow x - z.agg - z.val$
        \State \Return \textproc{Action-Select}($z.right, x_{new}$)
        \EndIf
        \EndFunction
    \end{algorithmic}
  \end{algorithm}
  
  \begin{restatable}{prop}{tree_sampling_mid} \label{thm: tree_sampling_mid}
    Let $T(z_1, z_2, \ldots, z_n)$ denote a sampling tree of size $n$ with nodes $z_1, z_2, \ldots, z_n$ arranged in an inorder traversal (see Chapter 12, Cormen et al., 2009). Also, without loss of generality assume that the node $z_i$ corresponds to the action $a_i$ and stores the value $e^{\theta_i}$. Then given the sampling tree $T$ and a search key $x \in \left[ 0, \sum_{i = 1}^n e^{\theta_i} \right)$ as the input, the \textproc{Action-Select} function returns the action $z_k.act = a_k$ such that the index $k$ satisfies the relation $\sum_{i = 1}^{k-1} e^{\theta_i} \leq x < \sum_{i = 1}^{k} e^{\theta_i}$.
  \end{restatable}
  \begin{proof}
    We will use induction to prove this statement. First, consider a tree of size one, $T(z_j)$ with $z_j.val = e^{\theta_j}$ and $z_j.agg = 0$. By the assumption on the search key, $x \in \left[ 0, e^{\theta_j} \right)$, i.e. the second condition in Algorithm \ref{alg: action_select} is true, and therefore the function returns $z_j.act$. Next, assume that the function does return the intended action for any tree of size $l \leq m$ where $m$ is some arbitrary positive number for an appropriate key $x$. We now prove the proposition statement for a key $x \in \left[ 0, \sum_{i = 1}^{l+1} e^{\theta_i} \right)$, and an arbitrary tree of size $l+1$, $T(z_1, z_2, \ldots, z_{r-1}, z_r, z_{r+1}, \ldots, z_{l+1})$ where $z_r$ represents the root node and $1 \leq r \leq l+1$. Then, depending on the value of $x$, we get three cases:
        \begin{itemize}
        \item \textit{Case 1} $(x < z_r.agg = \sum_{i = 1}^{r-1} e^{\theta_i})$: the function gets called again on the sub-tree $T(z_1, \ldots, z_{r-1})$, of size $r-1 \leq l$, with the key $x \in \left[0, \sum_{i = 1}^{r-1} e^{\theta_i} \right)$.
        \item \textit{Case 2} $(\sum_{i = 1}^{r-1} e^{\theta_i} = z_r.agg \leq x < z_r.agg + z_r.val = \sum_{i = 1}^{r} e^{\theta_i})$: the function returns $z_r.act$.
        \item \textit{Case 3} $(x \geq z_r.agg + z_r.val = \sum_{i = 1}^{r} e^{\theta_i})$: the function gets called on $T(z_{r+1}, \ldots, z_{l+1})$, a sub-tree of size $l-r+1 \leq l$, with the modified key $x_{new} \in \left[0, \sum_{i=r+1}^{l+1} e^{\theta_i} \right)$.
        \end{itemize}
        This closes the induction, and thus the proposition is true for sampling trees of all sizes.
  \end{proof}
  The above proposition shows that the action $a_k$ is selected whenever the search key $x$ lies in the range $\left[ \sum_{i = 1}^{k-1} e^{\theta_i}, \sum_{i = 1}^{k} e^{\theta_i} \right)$. However, since $x \sim \mathcal{U} \left[ 0, \sum_{i=1}^{|\mathcal{A}|} e^{\theta_i} \right)$, the probability of action $a_k$ being chosen is equal to $\frac{e^{\theta_k}}{\sum_{i=1}^{|\mathcal{A}|} e^{\theta_i}}$. Therefore, using Algorithm \ref{alg: action_select} is indeed equivalent to sampling from a softmax policy with action preferences $\theta_1, \theta_2, \ldots, \theta_{|\mathcal{A}|}$. 

      When an agent, using the alternate estimator, updates the action preference for the sampled action, it only needs to update the exponentiated preference for that action in the tree and the aggregated values for all nodes lying on the path from this nodes upto the root of the tree. The procedure shown in Algorithm \ref{alg: fix_upstream} accomplishes this in a logarithmic time. Also note that the structure of the tree remains unchanged during all these operations.
      \begin{algorithm}[!hbt]
        \caption{Fix Aggregates Upstream after an Update}
        \label{alg: fix_upstream}
        \begin{algorithmic}
          \State Call \textproc{Fix-Upstream}(sampling tree, node corresponding to the sampled action $a$, $\theta^{\text{new}}_a - \theta^{\text{old}}_a$)
          \State \texttt{\textcolor{purple}{$\;$\# function definition}}
          \Function{Fix-Upstream}{$T, z, \delta$}
          \Repeat \texttt{\textcolor{purple}{$\;$\# repeat until node $z$ is the root node}}
          \If{$z.p.left = z$} \texttt{\textcolor{purple}{$\;$\# update the parent's cumulant if $z$ is the left child}}
          \State $z.p.agg \leftarrow z.p.agg + \delta$
          \EndIf
          \State $z \leftarrow z.p$ \texttt{\textcolor{purple}{\# move up the tree}}
          \Until{$z.p \neq T.nil$} 
          \EndFunction
        \end{algorithmic}
      \end{algorithm}

      \subsection{\textbf{PROOFS:} The Alternate Estimator Utilizes Reward Noise} \label{app: zero_at_corner}
      In this section, we first show that the expected policy gradient at saturated regions, or equivalently the \textit{corners} of the probability simplex boundary (see Figure \ref{fig: policy_space}), is zero. We then show that at these corners, the variance of the regular estimator is zero as well, but the variance of the alternate estimator is equal to the variance of the reward noise which in general is non-zero. The consequence of these two facts is that whenever the policy is saturated, i.e. it places a high probability mass on some actions, the expected gradient becomes close to zero, and therefore the policy weights are not updated. Further, since the variance of the regular estimator at these corners is also zero, there is neither any gradient signal nor any noise, and the agent is unable to escape the sub-optimal region. In contrast, the alternate estimator, despite being zero in expectation, has a non-zero variance and therefore the agent can utilize the reward noise to escape the sub-optimal region.   

      \begin{restatable}{prop}{expectation_saturated_pg} \label{thm: expectation_saturated_pg}
        Define $\mathcal{I}_c := \{a \;|\; r(a) = c\}$ for some $c \in \mathbb{R}$. Consider a constant $c$ such that\footnote{We can always find such a $c$ by choosing it to be equal to the expected reward of some action.} $\mathcal{I}_c \neq \emptyset$ and that the policy is saturated on the actions in the set $\mathcal{I}_c$, i.e. $\sum_{a \in \mathcal{I}_c} \pi(a) = 1$. Then the expected policy gradient for softmax policies is zero:
        \begin{equation*}
          \nabla_{\thetavec} \mathcal{J} = \E [\gvec^{\text{REG}}(A, R)] = \E [\hat{\gvec}^{\text{REG}}(A, R)] = \E [\gvec^{\text{ALT}}(A, R)] = \zerovec.
        \end{equation*}
      \end{restatable}
      \begin{proof}
        We set $\nabla_{\thetavec} \mathcal{J} = \pivec \odot (\rvec - r_\pi \mathbf{1}) = \zerovec$ to obtain a system of $|\mathcal{A}| + 1$ linear equations in $|\mathcal{A}|$ variables:
        \begin{equation}
          \pi(a) \cdot (r(a) - r_\pi) = 0, \; \forall a \in \mathcal{A} \qquad \quad \text{and } \qquad \quad \sum_{a \in \mathcal{A}} \pi(a) = 1. \label{eq: expected_pg_fixed_pt_eqn} 
        \end{equation}
        Consider $\mathcal{I}_c = \{a \;|\; r(a) = c\}$ where $c$ is equal to the expected reward for some action. Then for a policy satisfying $\sum_{a \in \mathcal{I}_c} \pi(a) = 1$, we immediately see that $\pi(a) = 0, \forall a \notin \mathcal{I}_c$ and $r_\pi = c = r(a), \forall a \in \mathcal{I}_c$. Consequently, this policy would also satisfy Eqs. \ref{eq: expected_pg_fixed_pt_eqn} and thereby $\nabla_{\thetavec} \mathcal{J} = \zerovec$.
      \end{proof}

      The above proposition shows that the expected gradient bandit update $\nabla_{\thetavec} \mathcal{J}_\pi$ is zero only at the simplex boundary which can happen either at any one of the corners, or at a point on one of the face (i.e. a sub-space) of the simplex such that all the actions on that face have the same expected reward. The next proposition considers the variance of the estimators at the simplex corners.

      \begin{restatable}{prop}{variance_saturated_pg} \label{thm: variance_saturated_pg}
        Let $\sigma(a)^2 := \V[R | A=a]$ denote the conditional variance of the reward corresponding to action $a$. Assume that the policy is saturated on the action $c$, i.e., $\pi(c) = 1$. Also assume that $\sigma(c)^2 \neq 0$. Then, the variance of the regular PG estimator (with or without a baseline) is zero: $\V[\gvec^{\text{REG}}(A)] = \V[\hat{\gvec}^{\text{REG}}(A, R)] = \boldsymbol{0}$. Whereas, the variance of the alternate PG estimator (with or without a baseline) is non-zero: $\V[\gvec^{\text{ALT}}(A, R)] = \V[\hat{\gvec}^{\text{ALT}}(A, R)] = \sigma(c)^2 \evec_c$.
      \end{restatable}
      \begin{proof}
        As defined in the proof statement, the conditional variance of the rewards given an action $a$ is denoted using $\sigma(a)^2$, i.e. $\sigma(a)^2 := \V[R | A = a]$. Further, define $\sigmavec^2$ to be the conditional variance vector with $[\sigmavec^2]_k = \sigma(a_k)^2$. Also for brevity, let $\xvec^2 := \xvec \odot \xvec$ for any vector $\xvec$. Then, it is straightforward to show the following statements:
        \begin{align}
          \E[R] &= \sum_{a} \pi(a) \sum_x p(x | a) x = \sum_{a} \pi(a) r(a) = r_\pi, \\
          \E[R^2] &= \E_A \Big [\E_R[R^2 | A] \Big] = \E_A\Big[ \V[R | A] + \E[R | A]^2 \Big] = \sum_a \pi(a) \Big( \sigma(a)^2 + r(a)^2 \Big) \nonumber \\
          &= \pivec^\top (\sigmavec^2 + \rvec^2), \\
          \E[\evec_A] &= \sum_{a} \pi(a) \evec_a = \pivec, \\
          \E[\evec_A R] &= \sum_{a} \pi(a) \evec_a \sum_x p(x | a) x = \sum_{a} \pi(a) \evec_a r(a) = \pivec \odot \rvec, \\
          \E[\evec_A R^2] &= \E_A \Big [ \evec_A \E_R[R^2 | A] \Big] = \sum_{a} \pi(a) \evec_a \Big( \sigma(a)^2 + r(a)^2 \Big) = \pivec \odot (\sigmavec^2 + \rvec^2).
        \end{align}
        Next, let $\etavec$ be an action independent vector and $b$ be an action independent scalar. And assume that the variance for vector variables is calculated elementwise, i.e. the variance and covariance of arbitrary vector random variables $\Xvec$ and $\Yvec$ are defined as $\V[\Xvec] := \E[\Xvec \odot \Xvec] - \E[\Xvec] \odot \E[\Xvec]$ and $\Cov(\Xvec, \Yvec) := \E[\Xvec \odot \Yvec] - \E[\Xvec] \odot \E[\Yvec]$ respectively. Then,
        \begin{align}
          & \V[(\evec_A - \etavec) (R - b)] \nonumber \\
          &= \V[\evec_A R - \evec_A b - \etavec R + \etavec b] = \V[\evec_A R - \evec_A b - \etavec R] \nonumber \\
          &= \V[\evec_A R] + \V[\evec_A b] + \V[\etavec R] - 2 \Cov(\evec_A R, \evec_A b) - 2 \Cov(\evec_A b, \etavec R) + 2 \Cov(\evec_A b, \etavec R) \nonumber \\
          &= \V[\evec_A R] + \V[\evec_A] b^2 + \V[R] \etavec^2 - 2 \Cov(\evec_A R, \evec_A) b - 2 \Cov(\evec_A R, R) \odot \etavec + 2 \Cov(\evec_A, R) \odot \etavec b. \label{eq: gradient_bandit_estimators_variance_expression}
        \end{align}
        Simplifying the above expression for a policy saturated at action $c$, i.e. $\pivec = \evec_c$, gives us
        \begin{align*}
          & \V[\evec_A R] + \V[\evec_A] b^2 + \V[R] \etavec^2 - 2 \Cov(\evec_A R, \evec_A) b - 2 \Cov(\evec_A R, R) \odot \etavec + 2 \Cov(\evec_A, R) \odot \etavec b \\
          &= \evec_c \odot \sigmavec^2 + \zerovec + \sigma(c)^2 \etavec^2 - \zerovec - 2 \sigmavec \odot \evec_c \odot \etavec + \zerovec \\
          &= \sigma(c)^2 \evec_c + \sigma(c)^2 \etavec^2 - 2 \sigma(c)^2 \eta(c)^2 \evec_c = \begin{cases} \sigma(c)^2 \evec_c & \text{if } \etavec = \zerovec \text{(i.e. for alternate estimator), or} \\
            \zerovec & \text{if } \etavec = \pivec = \evec_c \text{(i.e. for the regular estimator).} \end{cases}
        \end{align*}
        This concludes the proof.
      \end{proof}

      Note that these two propositions again require the policy to have support only at a subset of all the actions. This condition cannot be fulfilled for softmax policies since the action preferences for these actions will need to be infinitely large. Despite that, using continuity arguments, we can still reason that for the regular estimator, both the expected gradient and its variance vanish in the proximity of the simplex boundary, whereas the alternate estimator will have non-zero variance.

      \subsection{\textbf{PROOFS:} The Alternate Estimator Utilizes its Biasedness} \label{app: attractor_repellor}
      The alternate estimator can also utilize the bias in the critic estimate $b$ to escape the saturated regions of the policy space. If the critic estimate $b$ is not equal to $r_\pi$, then the alternate estimator becomes biased, whereas the regular estimator does not. And for the alternate estimator, this bias helps make the policy more uniform thereby moving the agent out of the saturated region. To see this biasedness, consider a baseline $b \neq r_\pi$ and a vector $\etavec$ which we will set equal to either $\pivec$ for the regular estimator, or $\zerovec$ for the alternate estimator. Then it is straightforward to see that
      \begin{align}
        \E[(\evec_A - \etavec) (R - b)] &= \E[\evec_A R] - \E[\evec_A b] - \E[R] \etavec + \etavec b \nonumber = \pivec \odot \rvec - \pivec b - r_\pi \etavec + \etavec b \nonumber \\
        &= \begin{cases} \pivec \odot (\rvec - r_\pi \mathbf{1}) & \text{if } \etavec = \pivec, \qquad \text{and} \\ \pivec \odot (\rvec - b \mathbf{1}) & \text{if } \etavec = \zerovec. \end{cases} \label{eq: expected_gradient_bandit_updates}
      \end{align}
      Therefore,
      \begin{equation*}
        \nabla_{\thetavec} \mathcal{J} = \E[\hat{\gvec}^\text{REG}(A, R)] = \pivec \odot (\rvec - r_\pi \mathbf{1}) \neq \pivec \odot(\rvec - b \boldsymbol{1}) = \E[\hat{\gvec}^\text{ALT}(A, R)].
      \end{equation*}

      Interestingly, this biased update $\E[\hat{\gvec}^\text{ALT}(A, R)]$ is not the gradient of any function. To show this, we adopt the same process as Barnard (1993) did for the TD update.
      \begin{restatable}{prop}{fixed_pt_biased_no_grad} \label{thm: fixed_pt_biased_no_grad}
        The biased expected gradient bandits update is not a gradient of any objective function, i.e. for no function $\mathcal{L}$, $\nabla_{\thetavec} \mathcal{L} = \pivec \odot (\rvec - b)$.
      \end{restatable}
      \begin{proof}
        We prove this statement by contradiction. Assume that the biased update is the gradient of some function $\mathcal{L}$, i.e. $\frac{\partial \mathcal{L}}{\partial \theta_a} = \pi(a) \cdot (r(a) - b)$. Then,
        \begin{equation*}
          \frac{\partial^2 \mathcal{L}}{\partial \theta_c \partial \theta_a} = \pi(a) [\mathbb{I}(a = c) - \pi(c)] (r(a) - b) \quad \Rightarrow \quad \frac{\partial^2 \mathcal{L}}{\partial \theta_c \partial \theta_a} \neq \frac{\partial^2 \mathcal{L}}{\partial \theta_a \partial \theta_c}, \quad \text{for } a \neq c,
        \end{equation*}
        which gives a contradiction\footnote{In contrast, for the true gradient bandit update $\nabla_{\thetavec} \mathcal{J} = \pivec \odot (\rvec - r_\pi)$,
          \begin{equation*}
            \frac{\partial^2 \mathcal{J}}{\partial \theta_c \partial \theta_a} = - \pi(a) \pi(c) \big[ r(a) + r(c) - 2 r_\pi \big]  + \pi(a) \mathbb{I}(a = c) \big[ r(a) - r_\pi \big] \quad \Rightarrow \quad \frac{\partial^2 \mathcal{J}}{\partial \theta_c \partial \theta_a} = \frac{\partial^2 \mathcal{J}}{\partial \theta_a \partial \theta_c}, \quad \text{for } a \neq c.
          \end{equation*}
        }. As a result, the biased update cannot be the gradient of any objective function.
      \end{proof}

      The above result suggests that we cannot use the usual optimization analysis, such as setting the gradient equal to zero or computing second order derivatives, to study the properties of this biased update. Therefore, we instead try to understand its behavior by looking at its fixed point. It is now essential to introduce a few more terms. Let $r_1 \leq r_2 \leq \cdots \leq r_k$ represent an ordering of the true rewards for the bandit problem. Let $\pi_t = e^{\thetavec^{(t)}} \big/ \boldsymbol{1}^\top e^{\thetavec^{(t)}}$ represent the softmax policy at timestep $t$. Here, $\thetavec^{(t)}$ represents the action preference vector at timestep $t$; which is updated using the rule $\thetavec^{(t+1)} = \thetavec^{(t)} + \alpha \E[\hat{\gvec}^\text{ALT}(A, R)]$ for $\alpha>0$. 

      \begin{restatable}{lemma}{fixed_pt_biased_pg} \label{thm: fixed_pt_biased_pg}
        The fixed points of the biased gradient bandit update can be categorized as follows: \textbf{(1)} If there exists an $n \in \{1, 2, \ldots, k-1\}$ such that $r_1 \leq \cdots \leq r_n < b < r_{n+1} \leq \cdots \leq r_k$, then $\E_\pi[\hat{\gvec}^\text{ALT}(A, R)]$ is never equal to zero and has no fixed point. \textbf{(2)} If $b = r(a)$ for at least one action, then $\E_\pi[\hat{\gvec}^\text{ALT}(A, R)] = 0$ at any point on the face of the probability simplex given by $\sum_{a \in \mathcal{I}_b} \pi(a) = 1$ with $\mathcal{I}_b = \{a | r(a) = b\}$. \textbf{(3)} If $b < r_1$ or $b > r_k$, then $\E_\pi[\hat{\gvec}^\text{ALT}(A, R)] = 0$ at a point $\pi^*$ within the simplex boundary given by $\pi^*(a) = \frac{1}{r(a) - b} \left( \sum_{c \in \mathcal{A}} \frac{1}{r(c) - b} \right)^{-1}, \; \forall a \in \mathcal{A}$.
      \end{restatable}
      \begin{proof}
        Let $\gvec := \E[\hat{\gvec}^\text{ALT}(A, R)]$, and recall that we update the policy parameters using $\thetavec^{\text{new}} = \thetavec^{\text{old}} + \alpha \gvec$, and then find the policy by applying the softmax function. Now, at the fixed point, the policy should stop changing, i.e. $\pivec^{\text{old}} = \pivec^{\text{new}}$. This means that for each action $a \in \mathcal{A}$,
        \begin{equation}
          \frac{e^{\theta^{\text{old}}_a}}{\sum_b e^{\theta^{\text{old}}_b}} = \frac{e^{\theta^{\text{new}}_a}}{\sum_b e^{\theta^{\text{new}}_b}} = \frac{e^{\theta^{\text{old}}_a + \alpha g_a}}{\sum_b e^{\theta^{\text{old}}_b + \alpha g_b}} = \frac{e^{\theta^{\text{old}}_a} \cdot e^{\alpha g_a}}{\sum_b e^{\theta^{\text{old}}_b + \alpha g_b}} \quad \Leftrightarrow \quad g_a = \frac{1}{\alpha} \log \frac{\sum_b e^{\theta^{\text{old}}_b + \alpha g_b}}{\sum_b e^{\theta^{\text{old}}_b}} =: \kappa,
        \end{equation}
        where we introduced $\kappa$, a constant independent of the action $a$. Therefore, at the fixed point of the gradient bandit update, each element of the expected gradient is equal to the constant $\kappa$.

        Let us use this knowledge to solve for the policy at the fixed point of the biased PG update $\gvec = \pivec \odot (\rvec - b \mathbf{1})$. We have the following system of $|\mathcal{A}| + 1$ linear equations in $|\mathcal{A}| + 1$ variables:
        \begin{equation}
          \pi(a) \cdot (r(a) - b) = \kappa, \; \forall a \in \mathcal{A} \qquad \quad \text{and } \qquad \quad \sum_{a \in \mathcal{A}} \pi(a) = 1. \label{eq: bised_pg_fixed_pt_eqn} 
        \end{equation}
        We now consider the three cases given in the lemma statement separately:

        \noindent
        \textbf{Case 1 ($b$ lies between the expected rewards):} We know that there exist two actions $a_n$ and $a_{n+1}$ such that $r(a_n) = r_n < b < r_{n+1} = r(a_{n+1})$, then
        \begin{equation*}
          \pi(a_{n}) \cdot (r(a_{n}) - b) < 0 \qquad \text{and } \qquad \pi(a_{n+1}) \cdot (r(a_{n+1}) - b) > 0.
        \end{equation*}
        Consequently, for no value of $\kappa$, Eqs. \ref{eq: bised_pg_fixed_pt_eqn} could be satisfied, and those equations do not possess any fixed point.

        \noindent
        \textbf{Case 2 ($r(a) = b$ for atleast one action):}
        If the expected reward for any action $a$ is equal to the baseline $b$, then $\kappa = 0$, and the fixed point solution corresponds to the face of the simplex where the expected reward for each of the actions is equal to $b$:
        \begin{equation}
          \sum_{a \in \mathcal{I}_b} \pi(a) = 1, \quad \text{with } \quad \mathcal{I}_b = \{a \;|\; r(a) = b\}. \label{eq: fixed_pt_biased_pg_set1}
        \end{equation}

        \noindent
        \textbf{Case 3 ($b$ is either less than or greater than all of the expected rewards):} We assume that either $b > r(a)$ or $b < r(a)$ for all actions $a \in \mathcal{A}$. Now we solve for $\kappa$: From Eqs. \ref{eq: bised_pg_fixed_pt_eqn}, $\pi(a) = \frac{\kappa}{r(a) - b}, \; \forall a \in \mathcal{A}$, and therefore
        \begin{equation*}
          \sum_a \pi(a) = 1 \quad \Rightarrow \quad \sum_a \frac{\kappa}{r(a) - b} = 1 \quad \Rightarrow \quad \kappa = \left( \sum_a \frac{1}{r(a) - b} \right)^{-1}.
        \end{equation*}
        Therefore, the fixed point of the biased update is given by
        \begin{equation}
          \pi(a) = \frac{1}{r(a) - b} \left( \sum_{c \in \mathcal{A}} \frac{1}{r(c) - b} \right)^{-1} \quad \text{for all actions } a \in \mathcal{A}. \label{eq: imp_fixed_point_solution_biased}
        \end{equation}
        This completes the proof\footnote{Interestingly, using this style of reasoning to analyze the true expected gradient bandit would not give us any new information. Even if we assume that $g_a = \pi(a) \cdot (r(a) - b) = \kappa, \; \forall a \in \mathcal{A}$, we find that $\kappa$ is always zero for the true gradient bandit update. To see this, sum these $|\mathcal{A}|$ equations to get
          \begin{equation*}
            \sum_{a} \pi(a) \cdot (r(a) - r_\pi) = |\mathcal{A}| \kappa \quad \Rightarrow \quad \sum_{a} \pi(a) r(a) - r_\pi \sum_{a} \pi(a) = |\mathcal{A}| \kappa \quad \Rightarrow \quad \kappa = 0. 
          \end{equation*}
        }.
      \end{proof}

      Before moving forward, we state a result (Lemma 7.10, Lafferty et al., 2008) which will be useful in proving the next theorem.
      \begin{restatable}{lemma}{upper_bound_mgf} \label{thm: upper_bound_mgf}
        Let $X$ be a bounded random variable with bounds $u$ and $l$, i.e. $l \leq X \leq u$. Then
        \begin{equation*}
          \E\left[ e^{X} \right] \leq e^{(u - l)^2/8 + \E[X]}.
        \end{equation*}          
      \end{restatable}
      To keep the analysis simple, we assume that for all actions $a$, $b > r(a)$ or $b < r(a)$. Then, depending on how $b$ is set, the fixed point $\pi^*$ can act either as an attractor or a repellor. 

      \begin{restatable}{thm}{fixed_pt_biased_pg_attract_repulse} \label{thm: fixed_pt_biased_pg_attract_repulse}
        Assume that $\pi_t \neq \pi^*$. If the baseline is pessimistic, i.e. $b < r_1$, then for any $\alpha > 0$, the fixed point $\pi^*$ acts as a repellor: $D_{\text{KL}}(\pi^* \| \pi_{t+1}) > D_{\text{KL}}(\pi^* \| \pi_{t})$, where $D_{\text{KL}}$ denotes the KL-divergence. And if the baseline is optimistic, i.e. $b > r_k$, then given a sufficiently small positive stepsize $\alpha$, the fixed point $\pi^*$ acts as an attractor: $D_{\text{KL}}(\pi^* \| \pi_{t+1}) < D_{\text{KL}}(\pi^* \| \pi_{t})$.
      \end{restatable}
      \begin{proof}
        We first express $\pi_{t+1}$ in terms of $\pi_t$. If we define
        \begin{equation}
          \zeta(a) := \alpha [r(a) - b], \label{eq: zeta_definition}
        \end{equation}
        then for each action $a$, the action preference would be updated using $\thetavec^{(t+1)}_a = \thetavec^{(t)}_a + \pi_t(a) \zeta(a)$.  Applying the softmax function on these updated action preferences gives us $\pi_{t+1}$:
        \begin{equation}
          \pi_{t+1}(a) = \frac{e^{\thetavec^{(t)}_a + \pi_t(a) \zeta(a)}}{\sum_{c \in \mathcal{A}} e^{\thetavec^{(t)}_c + \pi_t(c) \zeta(c)}} = \frac{e^{\thetavec^{(t)}_a} \cdot e^{\pi_t(a) \zeta(a)}}{\sum_{c} e^{\thetavec^{(t)}_c} \cdot e^{\pi_t(c) \zeta(c)}} = \frac{\pi_t(a) \cdot e^{\pi_t(a) \zeta(a)}}{\sum_{c} \pi_t(c) \cdot e^{\pi_t(c) \zeta(c)}}. \label{eq: pi_{t+1}}
        \end{equation}
        Let us now compute the KL-divergence between the fixed point $\pi^*$ and the policy at timestep $t+1$:
        \begin{IEEEeqnarray}{rCl}
          D_{\text{KL}}(\pi^* \| \pi_{t+1}) &:=& \sum_{a \in \mathcal{A}} \pi^*(a) \log \frac{\pi^*(a)}{\pi_{t+1}(a)} = \sum_{a} \pi^*(a) \log \left( \frac{\pi^*(a)}{\pi_t(a)} \cdot \frac{\sum_{c} \pi_t(c) \cdot e^{\pi_t(c) \zeta(c)}}{e^{\pi_t(a) \zeta(a)}} \right) \nonumber \\
          &=& \sum_{a} \pi^*(a) \log \frac{\pi^*(a)}{\pi_t(a)} + \sum_{a} \pi^*(a) \log \left( \sum_{c} \pi_t(c) \cdot e^{\pi_t(c) \zeta(c)} \right) - \sum_{a} \pi^*(a) \pi_t(a) \zeta(a) \nonumber \\
          &=& D_{\text{KL}}(\pi^* \| \pi_{t}) + \log \left( \sum_{c} \pi_t(c) \cdot e^{\pi_t(c) \zeta(c)} \right) - \left( \sum_c \frac{1}{\zeta(c)} \right)^{-1}, \label{eq: KL_relation}
        \end{IEEEeqnarray}
        where in the first line we used the expression for $\pi_{t+1}$ from Eq. \ref{eq: pi_{t+1}} and in the last line we used the expression $\pi^*(a) = \frac{1}{r(a) - b} \left( \sum_{c \in \mathcal{A}} \frac{1}{r(c) - b} \right)$ from Eq. \ref{eq: imp_fixed_point_solution_biased}, to simplify the last term. We now consider the two cases (corresponding to an optimistic and a pessimistic baseline) separately.

        \noindent \textbf{Case 1 (Pessimistic baseline):} With the assumption that $\alpha > 0$ and $b < r(a)$, we immediately see that $\zeta(a) > 0$ for all actions $a$. Noting that $\zeta(a) > 0$, we apply Cauchy-Schwartz inequality:
        \begin{IEEEeqnarray}{lrCl}
          & \left( \sum_c \pi_t(c)^2 \zeta(c) \right) \left( \sum_c \frac{1}{\zeta(c)} \right) &>& \left( \sum_c \pi_t(c) \sqrt{\zeta(c)} \cdot \frac{1}{\sqrt{\zeta(c)}} \right)^2 = 1 \nonumber \\
          \Rightarrow \quad & \sum_c \pi_t(c)^2 \zeta(c) &>& \left( \sum_c \frac{1}{\zeta(c)} \right)^{-1},
        \end{IEEEeqnarray}
        where we could use a strict inequality in the first line because for no constant $\rho$ does $\frac{\pi_t(a)}{\zeta(a)} = \rho$ for all actions $a$; if this were true, it would imply that $\pi_t = \pi^*$, which in turn would violate the assumption made in the theorem statement. Next, recall that $\log(\cdot)$ is a concave function and then apply Jensen's inequality followed by the previous result to get:
        \begin{equation}
          \log \left( \sum_{c} \pi_t(c) \cdot e^{\pi_t(c) \zeta(c)} \right) > \sum_c \pi_t(c) \log \left( e^{\pi_t(c) \zeta(c)} \right) = \sum_c \pi_t(c)^2 \zeta(c) > \left( \sum_c \frac{1}{\zeta(c)} \right)^{-1}.
        \end{equation}
        Above inequality combined with Eq. \ref{eq: KL_relation} gives us the desired result:
        \begin{equation*}
          D_{\text{KL}}(\pi^* \| \pi_{t+1}) > D_{\text{KL}}(\pi^* \| \pi_{t}).
        \end{equation*}
        
        \noindent \textbf{Case 2 (Optimistic baseline):} In contrast to the previous case, the assumptions $\alpha > 0$ and $b > r(a)$, imply that $\zeta(a) < 0$ for all actions $a$. We will first use Lemma \ref{thm: upper_bound_mgf} to upper bound the term $\log \left( \sum_{c} \pi_t(c) \cdot e^{\pi_t(c) \zeta(c)} \right)$ and then obtain the condition on $\alpha$ (as given in Eq. \ref{eq: upper_bound_alpha_fixed_pt_attractor}) to show that the fixed point is an attractor. Consider a random variable $X: \mathcal{A} \rightarrow \mathbb{R}$ defined as $X(c) = \pi_t(c) \zeta(c)$ and distributed according to $\pi_t$, i.e. $\mathbb{P}(X = c) = \pi_t(c)$. Then $\E[X] = \sum_c \pi(c)^2 \zeta(c)$ and $X$ is bounded between $l := \alpha \min_{c} \pi_t(c) [r(c) - b]$ and $u := \alpha \max_{c} \pi_t(c) [r(c) - b]$. Using Lemma \ref{thm: upper_bound_mgf} along with the fact that $\log(\cdot)$ is an increasing function gives us:
        \begin{equation}
          \log \left( \sum_c \pi_t(c) e^{\pi_t(c) \zeta(c)} \right) \leq \frac{(u-l)^2}{8} + \sum_c \pi(c)^2 \zeta(c). \label{eq: fixed_point_upper_bound_log}
        \end{equation}
        Now the question is whether there exists a stepsize $\alpha > 0$, such that the following inequality holds:
        \begin{equation}
          \frac{(u-l)^2}{8} + \sum_c \pi(c)^2 \zeta(c) \overset{?}{<} \left( \sum_c \frac{1}{\zeta(c)} \right)^{-1} \label{eq: condition_alpha_biased_fixed_pt}
        \end{equation}
        And the answer is yes! To see this, again apply Cauchy-Schwartz inequality, only this time for $\zeta(c) < 0$, to obtain $\sum_c \pi_t(c)^2 \zeta(c) < \left( \sum_c \frac{1}{\zeta(c)} \right)^{-1}$. Therefore, there does exist an $\alpha > 0$ that satisfies Ineq. \ref{eq: condition_alpha_biased_fixed_pt}. With Eq. \ref{eq: KL_relation}, Ineq. \ref{eq: fixed_point_upper_bound_log}, and an appropriate condition on $\alpha$, the intended result that $\pi^*$ is an attractor readily follows:
        \begin{IEEEeqnarray*}{lC}
          & \log \left( \sum_c \pi_t(c) e^{\pi_t(c) \zeta(c)} \right) \leq \frac{(u-l)^2}{8} + \sum_c \pi(c)^2 \zeta(c) < \left( \sum_c \frac{1}{\zeta(c)} \right)^{-1} \\
          \Rightarrow \qquad & D_{\text{KL}}(\pi^* \| \pi_{t+1}) < D_{\text{KL}}(\pi^* \| \pi_{t}).
        \end{IEEEeqnarray*}
        All that remains is to find the expression for $\alpha$ that satisfies Ineq. \ref{eq: condition_alpha_biased_fixed_pt}:
        \begin{equation}
          \alpha < \frac{8}{\left( \max_c \pi_t(c) [r(c) - b] - \min_d \pi_t(d) [r(d) - b] \right)^2} \left[ \left( \sum_c \frac{1}{r(c) - b} \right)^{-1} - \sum_c \pi_t(c)^2 [r(c) - b] \right]. \label{eq: upper_bound_alpha_fixed_pt_attractor}
        \end{equation}
        Note that the above upper bound on $\alpha$ depends on the policy at timestep $t$, the choice of the baseline, and also the complete bandit reward structure (which is not known to the agent a priori).
      \end{proof}

      We present an example in Figure \ref{fig: fixed_point_nature_proof}, which shows that the upper bound on the stepsize, given in Eq. \ref{eq: upper_bound_alpha_fixed_pt_attractor}, is not tight. This figure also serves as a verification of the above proof.
      \begin{figure}[!tbp]
        \centering
        \includegraphics[scale=0.53]{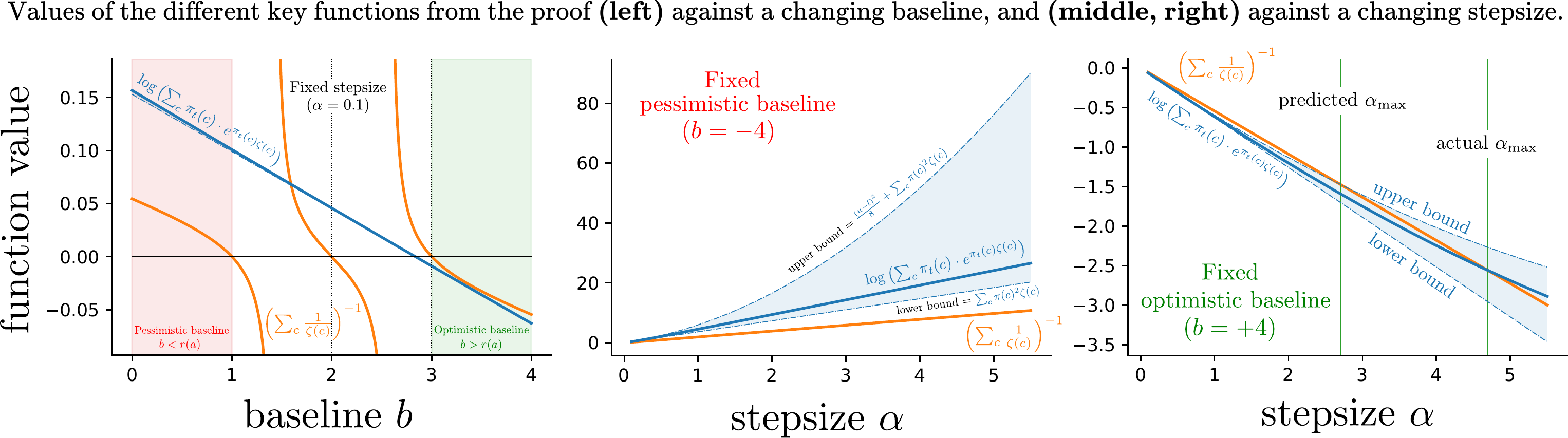}
        \caption{The verification of the inequalities used in Theorem \ref{thm: fixed_pt_biased_pg_attract_repulse}. We plot the key terms used in the proof: $\left( \sum_c \frac{1}{\zeta(c)} \right)^{-1}$, $\log \left( \sum_{c} \pi_t(c) \cdot e^{\pi_t(c) \zeta(c)} \right)$, its lower bound $\sum_c \pi(c)^2 \zeta(c)$, and its upper bound $\frac{(u - l)^2}{8} + \sum_c \pi(c)^2 \zeta(c)$, for a three armed bandit problem with $\rvec = [1 \; 2\; 3]^\top$ and the fixed policy $\pivec = [0.2 \; 0.1 \; 0.7]^\top$. The left figure shows the variation of the key terms with a changing baseline and a fixed stepsize $\alpha$. In the middle and the right plots, we fix the baseline to be pessimistic or optimistic and then vary the stepsize $\alpha$. The rightmost subplot also serves as the concrete example where the maximum stepsize predicted by Eq. \ref{eq: upper_bound_alpha_fixed_pt_attractor}, for the fixed point to remain attractive given an optimistic baseline, is not tight.}
        \label{fig: fixed_point_nature_proof}
      \end{figure}

      The above theorem illustrates an important property of the alternate estimator: with an optimistically initialized baseline, the agent is updated towards a more uniform distribution $\pi^*$. For an optimistic baseline, this distribution has a higher probability of picking the action with the maximum reward as compared to other actions; this can be seen from the expression for $\pi^*$ in Lemma \ref{thm: fixed_pt_biased_pg} or Figure \ref{fig: intuition_bandits_reward123_noise0_opti_pessi} (bottom row, second column). Therefore, if the agent were stuck in a sub-optimal corner of the probability simplex, an optimistic baseline would make the policy more uniform and encourage exploration\footnote{On the flip side, with a pessimistically initialized baseline, the alternate estimator can pre-maturely saturate towards a sub-optimal corner; see Figure \ref{fig: intuition_bandits_reward123_noise0_opti_pessi} (bottom right).}. And even though $\pi^*$ is different from the optimal policy, the agent with an alternate estimator can still reach the optimal policy, because as the agent learns and improves its baseline estimate, the alternate PG estimator becomes asymptotically unbiased. And hopefully by this time, the agent has already escaped the saturated policy region.

      One might wonder: but why the biased update is more helpful than the true gradient; after all the true gradient, unlike the biased update, doesn't just move the policy towards a point close to the optimal corner, it updates it exactly towards the optimal corner. The reason for this is that the true gradient update can be vanishingly small in saturated regions and therefore take an excessive amount of time to escape saturation (Mei et al., 2020b; Li et al., 2021). Whereas, as we hypothesize, the biased update will have a much larger update magnitude and therefore result in a faster escape. Even though, we don't theoretically show this, Example \ref{eg: saturated_pg_opti}, the policy update plots shown in Figure \ref{fig: intuition_bandits_reward123_noise0_opti_pessi} (bottom row, second column), and our experiments on bandits with an optimistic baseline initializations (see \S \ref{sec: opti_pessi_bandit_experiments}) do suggest this.

      \subsection{Example: Optimistic Initialization helps the Alternate Estimator}
      The following example considers an optimistically initialized baseline $b$ and shows how the alternate update can use this biased baseline for saturated policies.
      \begin{restatable}{eg}{saturated_pg} \label{eg: saturated_pg_opti}
        Consider the 3-armed bandit and the policy given in Example \ref{eg: saturated_pg}. Now assume that the agent doesn't have access to $r_\pi$ and instead uses an optimistic baseline $b = \delta \gg 0$. Again, the agent will sample $A = a_0$ at each timestep. Therefore:
        \begin{align*}
          \nabla_{\thetavec} \mathcal{J}_\pi &= \pivec \odot (\rvec - r_\pi \onevec) = [1\;0\;0]^\top \odot [0\;0\;1]^\top = [0\;0\;0]^\top, \\
          \hat{\gvec}^\text{REG}(A, R) &= (R - b) (\evec_A - \pivec) = (\epsilon - \delta) \cdot \Big( [1\;0\;0]^\top - [1\;0\;0]^\top \Big) = [0\;0\;0]^\top, \; \text{and} \\
          \hat{\gvec}^\text{ALT}(A, R) &= (R - b) \evec_A = (\epsilon - \delta) \cdot [1\;0\;0]^\top = [(\epsilon - \delta) \;0\;0]^\top.
        \end{align*}
      \end{restatable}
      As this example shows, with a suitable baseline (in this case a positive baseline that can override the reward noise), the alternate estimator can use the signal from the baseline to escape the saturated policy region, without relying solely on the reward noise. We formalize these two points in the sections ahead.
      
      \subsection{Visualizing the Stochastic PG Estimators for Bandits} \label{app: bandits_intuition}
      In this section, we study the properties of the \textbf{stochastic} regular and the alternate estimators. In the previous sections we primarily looked at the expected policy updates, and now we will focus on the sample based updates. We do so by plotting the stochastic policy gradient updates and the variance corresponding to the two estimators for different 3-armed bandit problems and various policies on the probability simplex (see Figure \ref{fig: policy_space}). These visualizations also illustrate the properties discussed in the previous sections.

      \begin{figure}[!tbp]
        \centering
        \includegraphics[scale=0.7]{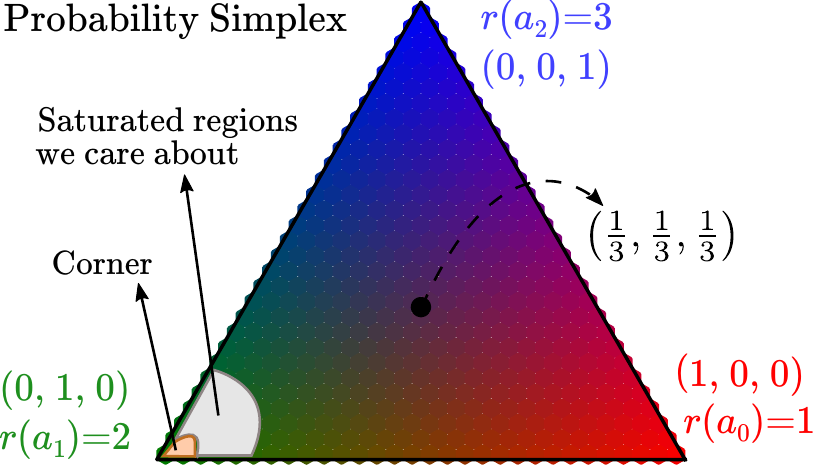}
        \caption{The probability simplex for a three armed bandit problem. Each point on the simplex corresponds to a policy. The color map shown is representative of the policy weight for the corresponding action. For instance, in regions that are more blue, the agent's policy has large probability mass on the the action $a_2$. The reward structure for this problem is $\rvec = [1\;2\;3]^\top$ with the individual components corresponding to the expected reward on the right, left, and the top corners respectively. We also highlight one corner of the simplex, and the saturated policy region near that corner. In our experiments, we focus on these regions neighbouring the corners, i.e. we focus on policies that are quite saturated but still have non-zero components for all the actions.} 
        \label{fig: policy_space}
      \end{figure}

      \textbf{The policy-update plots} were created by drawing a vector from the initial policy to the updated policy after one gradient ascent step using the corresponding PG estimator with a fixed stepsize\footnote{We chose stepsizes that resulted in a visually appropriate arrow length. To make the arrow lengths comparable across the estimators, the stepsize was kept same for the regular and the alternate estimators in the same setting (such as for the same choice of the baseline or the bandit reward structure).}. For each estimator, we show four different plots: three for the stochastic estimator, each corresponding to one of the sampled actions ($a_0, a_1$, or $a_2$); and the fourth one corresponding to the updates made by expected gradient (which itself can be the true policy gradient or the biased one). While calculating the estimator, we also added a noise of constant (unit) magnitude to the expected reward. The blue and grey arrows correspond to the policy change direction for positive and negative reward noises respectively. For clarity, let us illustrate this process with an example. Consider the bandit problem with $\mathcal{A} = \{a_0, a_1, a_2\}$, the reward structure $\rvec = [0\;0\;1]^\top$, and a positive reward noise. Let the old policy be $\pivec^{\text{old}} = [\nicefrac{1}{3}\;\nicefrac{1}{3}\;\nicefrac{1}{3}]^\top$ which will be updated using the stochastic regular estimator without baseline with the sampled action being $a_2$ and a stepsize of $\alpha = 0.4$. The action preferences corresponding to this policy can be computed by taking $\log(\cdot)$ of the action probabilities. This gives us $\thetavec^{\text{old}} = [\log(\nicefrac{1}{3})\;\log(\nicefrac{1}{3})\;\log(\nicefrac{1}{3})]^\top$. Next, we calculate the gradient update using the regular estimator without the baseline: $\hat{\gvec}(R, A) = (R - 0) \cdot (\evec_A - \pivec) = (r(a_2) + 1 - 0) \cdot ([0\;0\;1]^\top - [\nicefrac{1}{3}\;\nicefrac{1}{3}\;\nicefrac{1}{3}]^\top) = 2 \cdot [\nicefrac{-1}{3}\;\nicefrac{-1}{3}\;\nicefrac{2}{3}]^\top$. We then update the action preferences: $\thetavec^{\text{new}} = \thetavec^{\text{old}} + \alpha \hat{\gvec}(R, A)$ and in the end take the softmax of these action preferences to obtain the updated policy $\pivec^{\text{new}}$. Finally, we plot a blue colored vector (since a positive noise was added to the reward) from $\pivec^{\text{old}}$ to $\pivec^{\text{new}}$ on the probability simplex. We repeat this process for different initial policies to obtain the final plots (for instance, see Figure \ref{fig: intuition_bandits_reward001_noise1}).

      \textbf{The variance plots} were created by plotting the elementwise variance for the different stochastic PG estimators on the probability simplex. Each hexagonal point on the simplex corresponds to a particular policy value and the color specifies the variance of a particular component of the PG estimator. The variance was calculated using Eq. \ref{eq: gradient_bandit_estimators_variance_expression}. The top three rows show the variance for the three components and the last row shows the sum of variance along the three components.

      \subsubsection*{Bandit Problem with $\rvec = [0\;0\;1]^\top$ and Non-zero Reward Noise}
      Figure \ref{fig: intuition_bandits_reward001_noise1} shows the policy update plots and the variance heatmaps for the two estimators on a 3-armed bandit problem with the reward structure $\rvec = [0\;0\;1]^\top$ and reward noise $\epsilon \sim \mathcal{N}(0, 1)$. Each column corresponds to the updates and the variance for either the regular estimator with true $r_\pi$, the regular estimator without baseline $(b = 0)$, the alternate with true $r_\pi$, or the alternate estimator without baseline $(b = 0)$. We include the zero baseline case because in our experiments, when we learn the baseline, we often initialize it to zero.

      \begin{figure}[!hbp]
        \centering
        \includegraphics[scale=0.55]{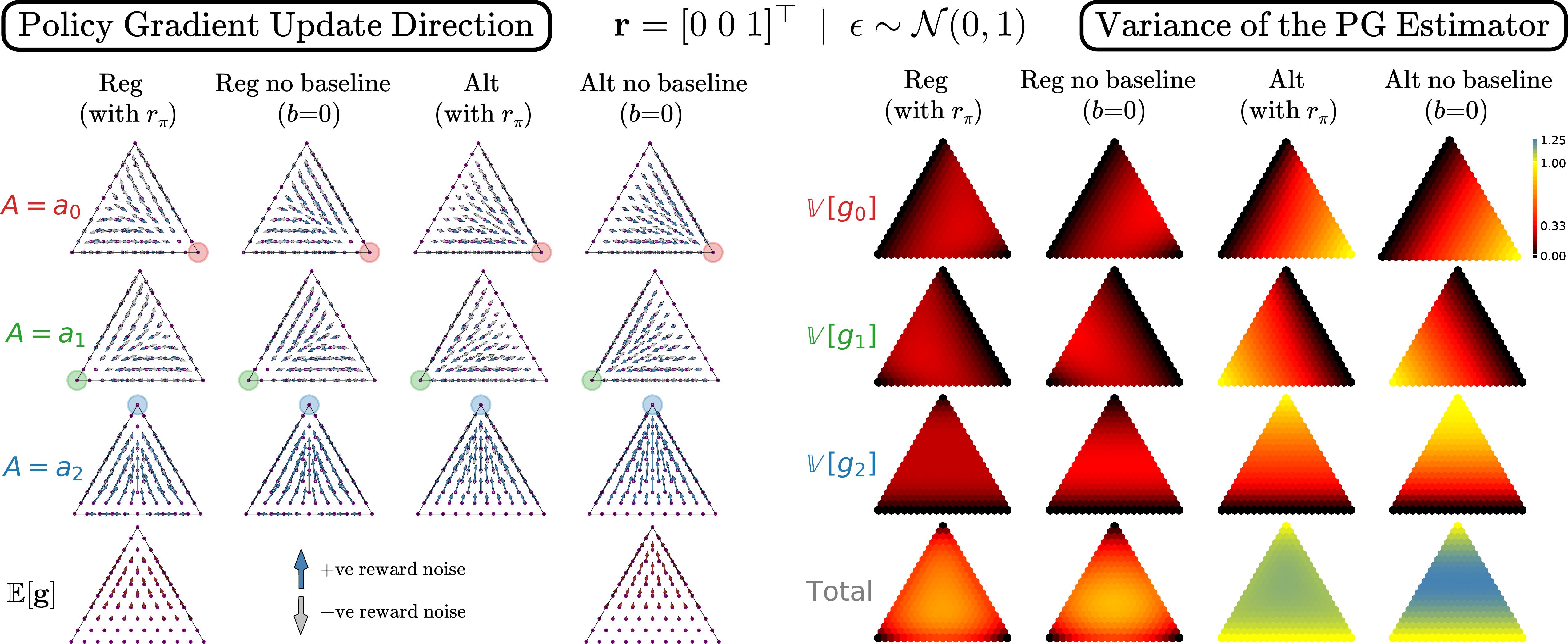}
        \caption{Policy updates and the variance plots for the PG estimators on a 3-armed bandit problem with $\rvec = [0\;0\;1]^\top$ and a non-zero reward noise. \textbf{(Update Plots)} The top three rows show the policy updates corresponding to the different stochastic estimators upon taking a particular action. Last row shows the expected update. Since, the alternate estimator with true $r_\pi$ and both the regular estimators result in the same expected update, we show only a single plot for these three settings. \textbf{(Variance Plots)} The first three rows depict the respective component of the variance vector (of the different gradient estimators) and the last row shows the total variance.} 
        \label{fig: intuition_bandits_reward001_noise1}
      \end{figure}

      From the update plots, we see that the stochastic alternate estimator has much higher update magnitudes than the regular version. For instance, consider a policy saturated say at action $a_0$, i.e. it is near the vicinity of the \textcolor{red}{right} corner. With the alternate gradient estimator, there is a chance that the policy will move towards the center of the simplex, i.e. it will be updated towards a more uniform distribution. Whereas, with the regular estimator, the agent will likely remain stuck. This observation is also reconciled by the expressions for the estimators: the term $\evec_A - \pivec$ in the regular estimator becomes very small as $\pivec$ saturates on action $A$; whereas the alternate estimator only has the term $\evec_A$ and its magnitude remains the same irrespective of the saturation. As discussed earlier, the alternate estimator with no baseline is biased. However, its ascent direction points in the same direction as the true gradient update. From the variance plots, we notice that the variance at the corner for both the regular estimators is zero, and for the alternate estimators is non-zero.

      So near a corner of the simplex, the alternate estimator utilizes the reward noise to essentially perform a random walk. At sub-optimal corners, this random movement is good since it might make the policy more uniform, from where it is easier to move towards the optimal corner. And at the optimal corners, this random movement is not particularly bad either because optimal corners represent stable equilibrium points: in expectation the gradient signal towards them.

      \begin{figure}[!tbp]
        \centering
        \includegraphics[scale=0.55]{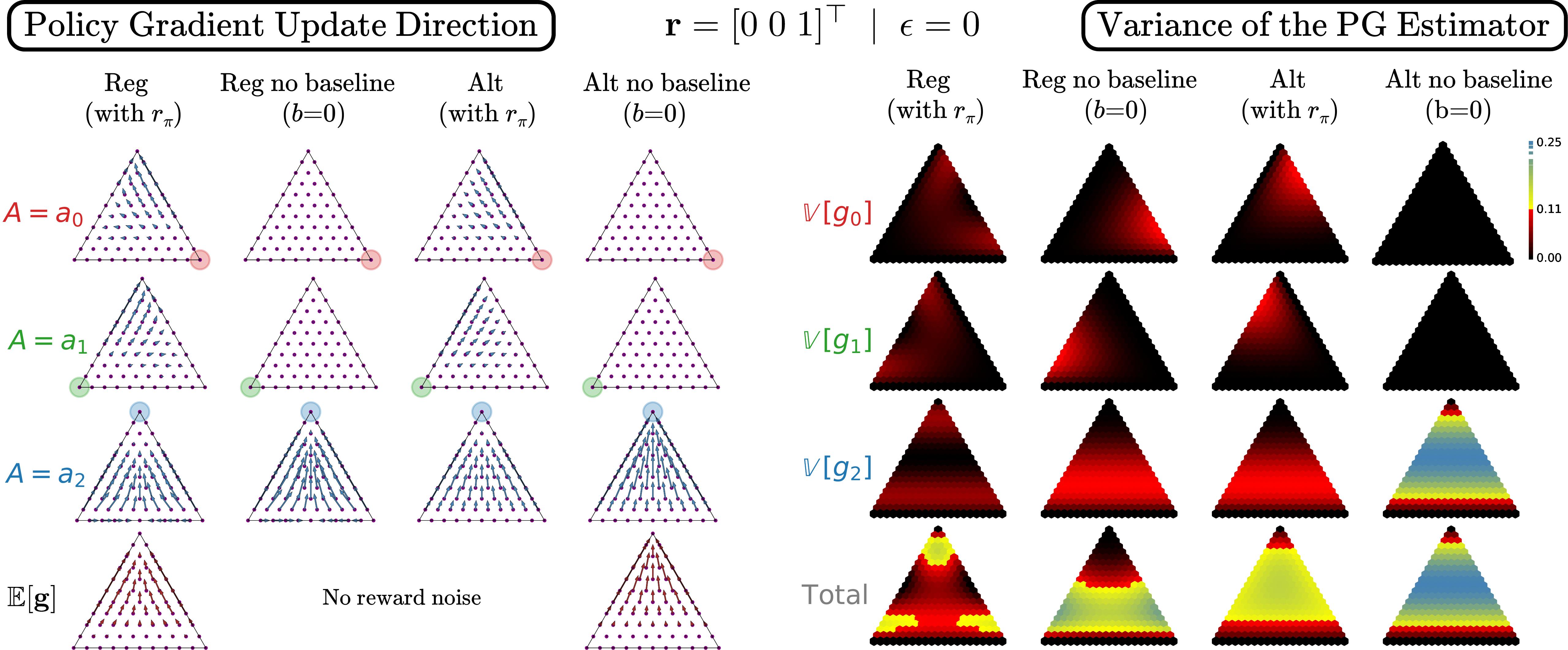}
        \caption{Policy updates and the variance for PG estimators on a 3-armed bandit problem with $\rvec = [0\;0\;1]^\top$ and no reward noise.} 
        \label{fig: intuition_bandits_reward001_noise0}
      \end{figure}

      \subsubsection*{Bandit Problem with $\rvec = [0\;0\;1]^\top$ and Zero Reward Noise}
      Now we study how these estimators behave when we take away the noise from the rewards in the bandit problem. Figure \ref{fig: intuition_bandits_reward001_noise0} shows how the policy is updated and the variance of the gradient estimators for a bandit problem with no reward noise. From these plots, we see that the benefit of the alternate estimator (large update magnitudes at the corners) disappears. With neither any reward signal nor any noise, the alternate estimator cannot escape the corners (see the experiments in \S \ref{sec: bandit_noise_plus_no_noise}). However, the alternate estimator still has a higher update magnitude as compared to the regular estimator near the corners. From the variance plots, we see that, as shown by Proposition \ref{thm: variance_saturated_pg}, the variance for all the estimators is zero at the corners; no reward noise means no variance.

      \subsubsection*{Bandit Problem with $\rvec = [1\;2\;3]^\top$ and Non-zero Reward Noise}
      In the previous settings, there was no reward signal at the sub-optimal corners. Now we consider a case in which there is a non-zero reward signal at the sub-optimal corners and the noise is comparable to it. From Figure \ref{fig: intuition_bandits_reward123_noise1} (left), we see that the alternate estimator (with true $r_\pi$) works better as compared to the regular estimator: it has high noise at sub-optimal corners which might make the policy more uniform. However, with $b=0$, this is no longer true. In fact, with the alternate estimator, we are driven into the sub-optimal corner faster than the regular without the baseline. Since $b = 0$ is pessimistic for the reward structure $\rvec = [1\;2\;3]^\top$, the policy will move away from the center of the simplex (Theorem \ref{thm: fixed_pt_biased_pg_attract_repulse}). So initially, the alternate estimator will drive the policy into the corner. Then as its baseline estimate improves and becomes closer to $r_\pi$, there is some chance that the agent will escape saturation (although our experiments in Figure \ref{fig: bad_init_pessimistic_param_study} show that the alternate estimator with a pessimistic baseline is often inferior to the regular estimator).

      In Figure \ref{fig: intuition_bandits_reward123_noise1} (right), the top row shows the variance of the policy against its entropy, and the bottom row shows the variance of the particular estimator against the policy's distance (measured using KL divergence) from that corner, i.e. for each action $a$, we plot $\V[g_a]$ against $D_{\text{KL}}(\evec_a \| \pi)$. The results are similar as the previous variance plots. These entropy figures, unlike the variance heatmaps, can also be plotted for bandit problems with more than three actions.

      \begin{figure}[!tbp]
        \centering
        \includegraphics[scale=0.6]{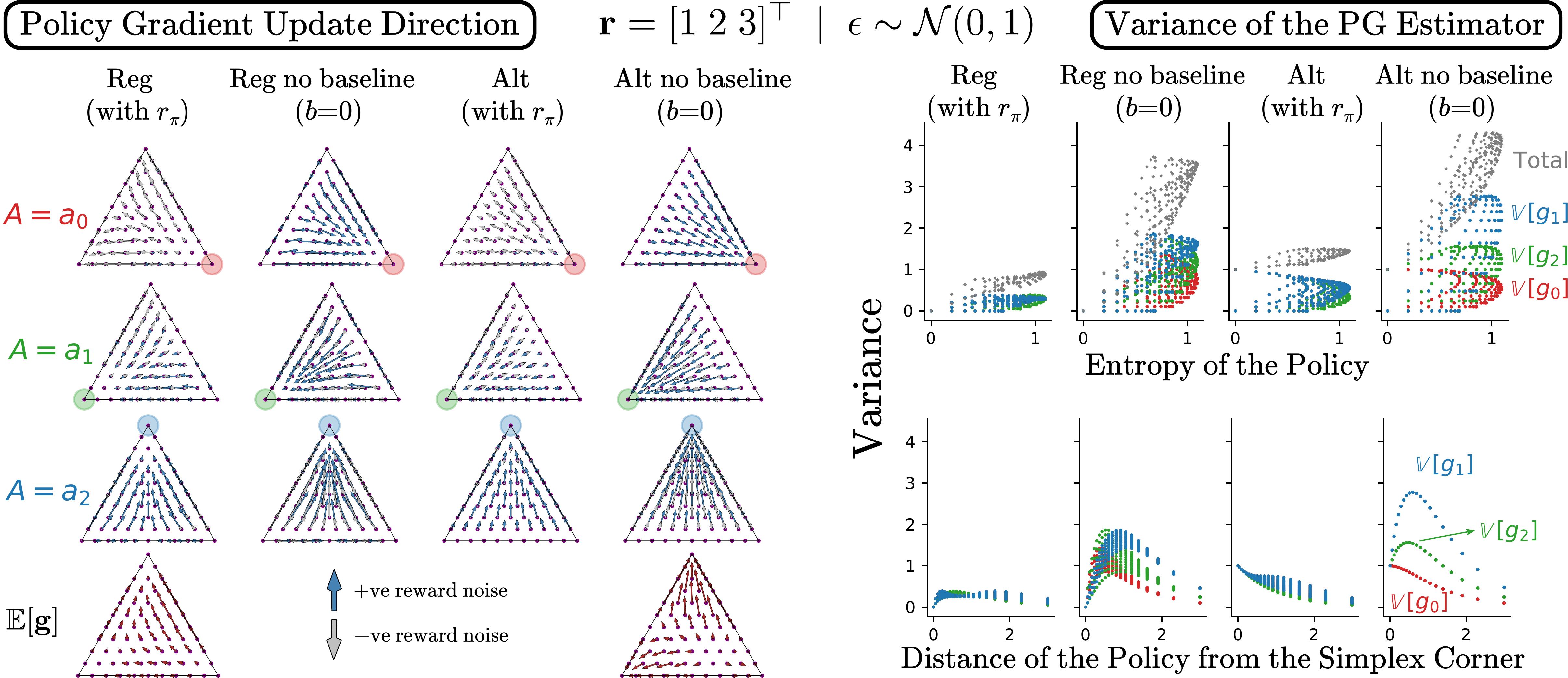}
        \caption{Policy updates and the entropy plots for PG estimators on a 3-armed bandit problem with $\rvec = [1\;2\;3]^\top$ and a non-zero reward noise.} 
        \label{fig: intuition_bandits_reward123_noise1}
      \end{figure}

      \subsubsection*{Bandit Problem with $\rvec = [1\;2\;3]^\top$ and an Optimistic or a Pessimistic Baseline}

      In Figure \ref{fig: intuition_bandits_reward123_noise0_opti_pessi}, we show the policy update plots for optimistic or pessimistic baselines. We fix the reward structure to $\rvec = [1\;2\;3]^\top$ and do not add any reward noise (the results with non-zero reward noise were effectively similar to these ones, and therefore we skipped them). For an optimistic baseline, we see that both the regular and the alternate stochastic estimators make the policy more uniform. The alternate estimator again has higher update magnitudes compared to the regular estimator. The last row illustrates the point made by Theorem \ref{thm: fixed_pt_biased_pg_attract_repulse}, that the alternate estimator with an optimistic baseline attracts the policy to a point on the simplex. Further, this attractive signal, as seen from the figure, is much stronger than the expected gradient signal offered by the regular estimator. The expected regular estimator moves the policy towards the optimal corner albeit with very small update magnitudes, especially at the corner. The trends are similar for the pessimistic baseline, except that the policy is updated aggressively towards a corner. The variance plots for this settings looked very similar to those from the previous settings, and therefore we did not include them.

      \begin{figure}[!tbp]
        \centering
        \includegraphics[scale=0.8]{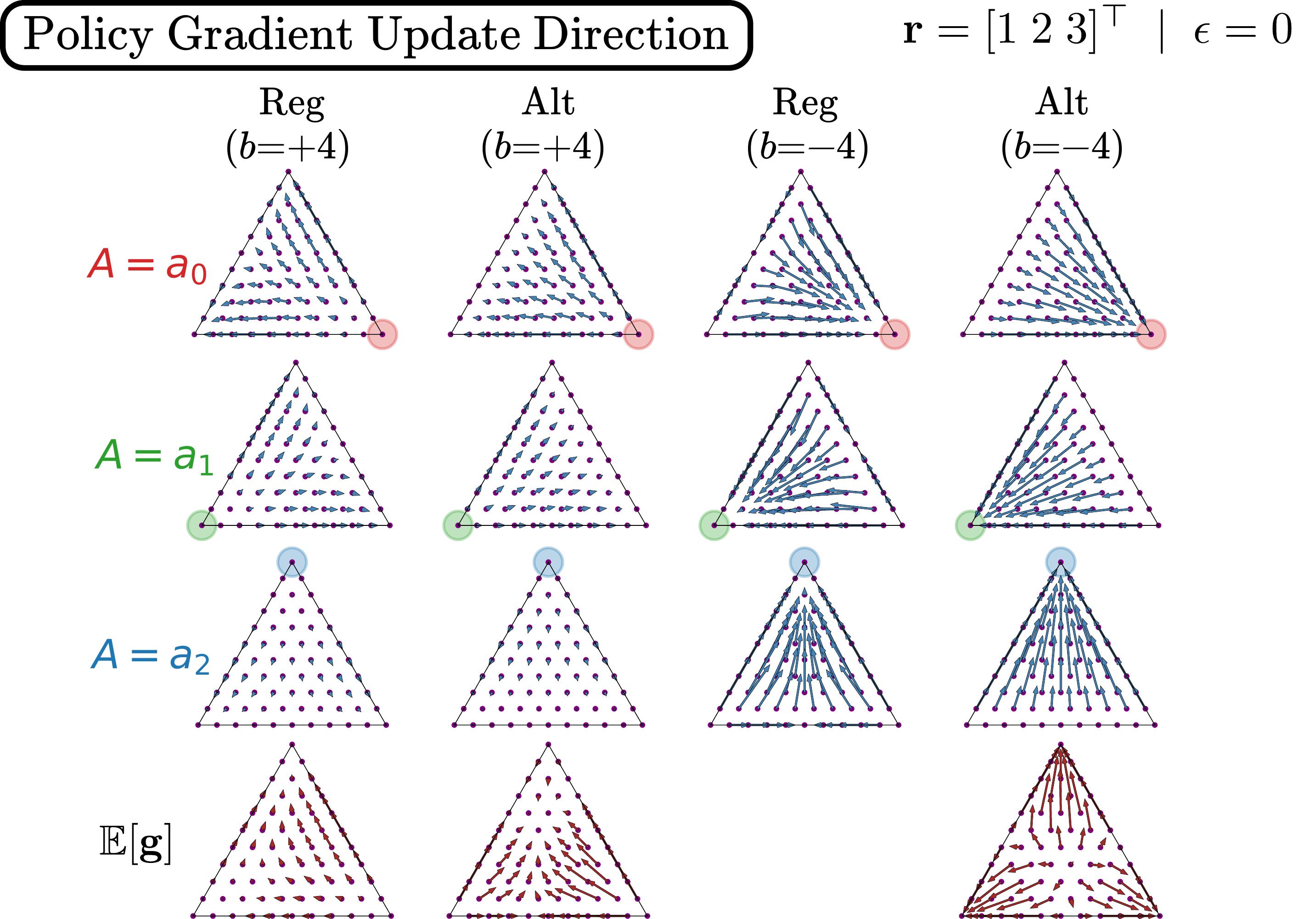}
        \caption{Policy updates for the bandit problem $\rvec = [1\;2\;3]^\top$ and no reward noise. The regular and the alternate estimators use an optimistic $(b = +4)$ or pessimistic $(b = -4)$ baseline. As discussed in \S \ref{sec: alternate_utilizes_biasedness}, both the regular estimators are unbiased, whereas both the alternate estimators are biased.} 
        \label{fig: intuition_bandits_reward123_noise0_opti_pessi}
      \end{figure}
      
      \section{Softmax PG Estimators for MDPs} \label{app: theoretical_analysis_mdp}

      In this section, we begin by quickly giving a remark about the regular estimator in the MDP setting. Then we state how to adapt the alternate PG estimator to various PG algorithms (such as PPO), and end by giving implementation details for linear online Actor-Critic algorithm.

      \subsection{Remark about the Regular Softmax PG Estimator for MDPs}
      We first show that the regular gradient estimator $\gvec^{\text{REG}}(S, A)$ (derived in \S \ref{sec: alt_pg_mdp}) despite the complicated expression is exactly equal to the familiar PG estimator, i.e.
      \begin{equation}
        \gvec^{\text{REG}}(S, A) = \nabla_{\wvec} \log \pi(A | S) \big( q_\pi(S, A) - v_\pi(S) \big).
      \end{equation}
      To see this equivalence, note that for a softmax policy $\pi$ and the state-action pair $(S, A)$,
      \begin{align*}
        \nabla_{\wvec} \log \pi(A | S) &= \nabla_{\wvec} [\thetavec(S)]_A - \nabla_{\wvec} \log \left( \sum_{a} e^{[\thetavec(S)]_a} \right) = \nabla_{\wvec} [\thetavec(S)]_A - \frac{\sum_{a} e^{[\thetavec(S)]_a} \nabla_{\wvec} [\thetavec(S)]_a}{\sum_{b} e^{[\thetavec(S)]_b}} \\
        &= \nabla_{\wvec} [\thetavec(S)]_A - \sum_{a} \pi(a | S) \nabla_{\wvec} [\thetavec(S)]_a.
      \end{align*}
      As the reader might observe, therefore the calculation in \S \ref{sec: alt_pg_mdp} merely resulted in a different expression for the already familiar PG estimator; it did not yield a new one and therefore was little more than an exercise in manipulating the vector variables $\pivec(\cdot | S)$ and $\thetavec(S)$.
      
      \subsection{Adapting the Alternate Estimator to various PG Algorithms} \label{app: diff_alternate_estimators}
      The form for the alternate estimator given by Eq. \ref{eq: mdp_alternate_stochastic_loss} needs to be adapted to a concrete policy gradient algorithm before it can be used by an RL agent. In this section, we demonstrate this adaption process for the Proximal Policy Optimization algorithm (abbreviated as PPO, Schulman et al., 2017). After that, we directly give the expressions for the alternate estimator adapted to various PG algorithms such as REINFORCE (Williams, 1991), Online Actor-Critic (Sutton and Barto, 2018), and Trust Region Policy Optimization (abbreviated as TRPO, Schulman et al., 2015).

      \subsubsection*{The Alternate PPO Estimator}
      PPO proposes the following clipped objective that the agent maximizes at each step:
      \begin{equation}
        \mathcal{J}_{\text{PPO}} := \sum_s \nu_{\pi_{\text{old}}}(s) \sum_a \pi_{\text{old}}(a | s) \cdot \min \left( \begin{matrix} \frac{\pi_{\wvec}(a | s)}{\pi_{\text{old}}(a | s)} h_{\pi_{\text{old}}}(s, a), \\ \text{clip} \left[\frac{\pi_{\wvec}(a | s)}{\pi_{\text{old}}(a | s)}, 1 - \epsilon, 1 + \epsilon \right] h_{\pi_{\text{old}}}(s, a) \end{matrix} \right),
      \end{equation}
      where $h_{\pi_{\text{old}}}(s, a) := q_{\pi_{\text{old}}}(s, a) - v_{\pi_{\text{old}}}(s)$ is the advantage function under the old policy $\pi_{\text{old}}$, and $\epsilon$ is a fixed parameter. The PPO objective introduces an importance sampling ratio to reuse the data collected by the old policy to make multiple update steps to the new policy\footnote{Despite this reuse of the data collected by the old policy, PPO is still very much an on-policy PG algorithm since it does not correct for the state distribution mismatch between $\pi_{\text{old}}$ and the current policy: $\nu_{\pi_{\text{old}}}(s) \neq \nu_{\pi_{\wvec}}(s)$.}. Further, the agent only makes updates for those samples which satisfy a particular condition that we define shortly. Using this view, the gradient of the above objective can be equivalently written as
      \begin{align}
        \nabla_{\wvec} \mathcal{J}_{\text{PPO}} &= \sum_s \nu_{\pi_{\text{old}}}(s) \sum_a \pi_{\text{old}}(a | s) \cdot \mathbb{I} \Big( \text{cond}^{\pi_{\wvec}}_{\pi_{\text{old}}}(s, a) \Big) \frac{\nabla_{\wvec} \pi_{\wvec}(a | s)}{\pi_{\text{old}}(a | s)} h_{\pi_{\text{old}}}(s, a) \nonumber \\
        &= \sum_s \nu_{\pi_{\text{old}}}(s) \sum_a \nabla_{\wvec} \pi_{\wvec}(a | s) h_{\pi_{\text{old}}}(s, a) \mathbb{I} \Big( \text{cond}^{\pi_{\wvec}}_{\pi_{\text{old}}}(s, a) \Big), \label{eqn: ppo_grad}
      \end{align}
      where $\text{cond}()$ is the Boolean condition defined as
      \begin{align*}
        \text{cond}^{\pi_{\wvec}}_{\pi_{\text{old}}}(s, a) \nonumber := \left( h_{\pi_{\text{old}}}(s, a) > 0 \;\bigwedge\; \frac{\pi_{\wvec}(a | s)}{\pi_{\text{old}}(a | s)} < 1 + \epsilon \right) \bigvee \left( h_{\pi_{\text{old}}}(s, a) < 0 \;\bigwedge\; \frac{\pi_{\wvec}(a | s)}{\pi_{\text{old}}(a | s)} > 1 - \epsilon \right).
      \end{align*}
      To reduce clutter, let us define $\tilde{h}_{\pi_{\text{old}}}^{\pi_{\wvec}}(s, a) := h_{\pi_{\text{old}}}(s, a) \mathbb{I} \Big( \text{cond}^{\pi_{\wvec}}_{\pi_{\text{old}}}(s, a) \Big)$. Then the regular PPO estimator can be written as
      \begin{equation*}
        \nabla_{\wvec} \mathcal{J}_{\text{PPO}} = \E_{S \sim \nu_{\pi_{\text{old}}}; A \sim \pi_{\text{old}}(\cdot | S)} \underbrace{\left[ \nabla_{\wvec} \log \pi_{\wvec}(A | S) \frac{\pi_{\wvec}(A | S)}{\pi_{\text{old}}(A | S)} \tilde{h}_{\pi_{\text{old}}}^{\pi_{\wvec}}(S, A) \right]}_{=: \gvec^{\text{REG-PPO}}(S, A)}.
      \end{equation*}
      Now we follow the same procedure that we used to obtain the alternate estimator in Eq. \ref{eq: mdp_alternate_vector_grad_softmax}, but for the expression that appeared in the PPO gradient (Eq. \ref{eqn: ppo_grad}):
      \begin{equation*}
        \sum_a \nabla_{\wvec} \pi_{\wvec}(a | s) \tilde{h}_{\pi_{\text{old}}}^{\pi_{\wvec}}(s, a) = \nabla_{\wvec} \thetavec(s) \sum_a \pi_{\wvec}(a | s) \left[ \tilde{h}_{\pi_{\text{old}}}^{\pi_{\wvec}}(s, a) - \sum_c \pi_{\wvec}(c | s) \tilde{h}_{\pi_{\text{old}}}^{\pi_{\wvec}}(s, c) \right] \evec_a.
      \end{equation*}
      Putting the above expression into Eq. \ref{eqn: ppo_grad}, we get the alternate estimator for PPO:
      \begin{align}
        & \nabla_{\wvec} \mathcal{J}_{\text{PPO}} \nonumber \\
        &= \sum_s \nu_{\pi_{\text{old}}}(s) \sum_a \pi_{\text{old}}(a | s) \nabla_{\wvec} [\thetavec(s)]_a \frac{\pi_{\wvec}(a | s)}{\pi_{\text{old}}(a | s)}  \left[ \tilde{h}_{\pi_{\text{old}}}^{\pi_{\wvec}}(s, a) - \sum_c \pi_{\wvec}(c | s) \tilde{h}_{\pi_{\text{old}}}^{\pi_{\wvec}}(s, c) \right] \nonumber \\
        &= \E_{S \sim \nu_{\pi_{\text{old}}}; A \sim \pi_{\text{old}}(\cdot | S)} \underbrace{\left[ \nabla_{\wvec}[\thetavec(S)]_A \frac{\pi_{\wvec}(A | S)}{\pi_{\text{old}}(A | S)}  \left( \tilde{h}_{\pi_{\text{old}}}^{\pi_{\wvec}}(S, A) - \sum_c \pi_{\wvec}(c | S) \tilde{h}_{\pi_{\text{old}}}^{\pi_{\wvec}}(S, c) \right) \right]}_{=: \gvec^{\text{ALT-PPO}}(S, A)}.
      \end{align}

      Using a process similar to that demonstrated here, we can extend the alternate estimator to different PG algorithms. We now summarize these estimators.

      \subsubsection*{Alternate REINFORCE Estimator}
      The regular REINFORCE estimator is
      \begin{equation}
        \nabla \mathcal{J_\pi} = \E_{H_\infty \sim \mathbb{P}_\pi} \underbrace{\left[\sum_{t=0}^\infty \gamma^t \nabla \log \pi(A_t | S_t) \Big(G_t - v_\pi(S_t) \Big) \right]}_{=: \gvec^{REG-REINF}(H_\infty)}.
      \end{equation}
      And the alternate estimator is given by
      \begin{equation}
        \nabla_{\wvec} \mathcal{J}_{\pi} = \E_{H_\infty \sim \mathbb{P}_\pi} \underbrace{\left[\sum_{t=0}^\infty \gamma^t \nabla_{\wvec} [\thetavec(S_t)]_{A_t} \Big( G_t - v_\pi(S_t) \Big) \right]}_{=: \gvec^{\text{ALT-REINF}}(H_\infty)}.
      \end{equation}

      \subsubsection*{Alternate Online AC Estimator}
      The regular estimator is given by
      \begin{equation}
        \nabla \mathcal{J_\pi} \propto \E_{S, K \sim d_\pi; A \sim \pi(\cdot | S); S', R \sim p(\cdot, \cdot | S, A)} \underbrace{\left[ \gamma^K \nabla \log \pi(A | S) \Big( R + \gamma v_\pi(S') - v_\pi(S) \Big) \right]}_{=: \gvec^{\text{REG-OAC}}(S, A, R, S', K)}, \label{eq: regular_online_ac_estimator_app}
      \end{equation}
      and the alternate estimator is
      \begin{equation}
        \nabla_{\wvec} \mathcal{J}_{\pi} \propto \E_{S, K \sim d_\pi; A \sim \pi(\cdot | S); S', R \sim p(\cdot, \cdot | S, A)} \underbrace{\left[ \gamma^K \nabla_{\wvec} [\thetavec(S)]_{A} \Big( R + \gamma v_\pi(S') - v_\pi(S) \Big) \right]}_{=: \gvec^{\text{ALT-OAC}}(S, A, R, S', K)}. \label{eq: alternate_online_ac_estimator}
      \end{equation}

      \subsubsection*{TRPO Estimators}
      TRPO (Eq. 14, Schulman et al., 2015) proposes that the agent solve the following constrained optimization problem at each timestep
      \begin{equation*}
        \max_{\wvec} \; \underbrace{\sum_s  \nu_{\pi_{\text{old}}}(s) \sum_a \pi_{\wvec}(a | s) q_{\pi_{\text{old}}}(s, a)}_{=: \mathcal{J}_{\text{TRPO}}} \qquad \text{subject to } \sum_s \nu_{\pi_{\text{old}}}(s) \cdot D_{\text{KL}} \big( \pi_{\text{old}}(\cdot | s) \big\| \pi_{\wvec}(\cdot | s) \big) \leq \delta,
      \end{equation*}
      where $\delta$ is a fixed parameter.

      We can re-write the gradient of the objective $\mathcal{J}_{\text{TRPO}}$ as follows
      \begin{equation*}
        \nabla_{\wvec} \mathcal{J}_{\text{TRPO}} = \E_{S \sim \nu_{\pi_{\text{old}}}; A \sim \pi_{\text{old}}(\cdot | S)} \underbrace{\left[ \nabla_{\wvec} \log \pi_{\wvec}(A | S) \frac{\pi_{\wvec}(A | S)}{\pi_{\text{old}}(A | S)} q_{\pi_{\text{old}}}(S, A) \right]}_{=:\gvec^{\text{REG-TRPO}}(S, A)}.
      \end{equation*}
      Then, it is straightforward to see that the alternate estimator would become
      \begin{equation}
        \nabla_{\wvec} \mathcal{J}_{\text{TRPO}} = \E_{S \sim \nu_{\pi_{\text{old}}}; A \sim \pi_{\text{old}}(\cdot | S)} \underbrace{\left[ \nabla_{\wvec} [\thetavec(S)]_{A} \frac{\pi_{\wvec}(A | S)}{\pi_{\text{old}}(A | S)} \Big( q_{\pi_{\text{old}}}(S, A) - \sum_c \pi_{\wvec}(c | S) q_{\pi_{\text{old}}}(S, c) \Big) \right]}_{=:\gvec^{\text{ALT-TRPO}}(S, A)}.
      \end{equation}

      The alternate estimators proposed here for TRPO and PPO require the computation of the expectation of $q_{\pi_{\text{old}}}$ and $\tilde{h}_{\pi_{\text{old}}}^{\pi_{\wvec}}$ respectively, under the current policy $\pi_{\wvec}$ and across the action space. However, if the number of actions is large, this operation becomes prohibitively expensive\footnote{Moreover, if this expectation is calculated exactly, the alternate (TRPO or PPO) estimator will remain unbiased. Depending on the context this might be a good or a bad property to have. However, in this work, we do not pursue a more involved analysis of the alternate estimators for TRPO and PPO. Also note that the objectives maximized by PPO and TRPO are different from the original PG objective $\mathcal{J}_\pi := \sum_s  \nu_{\pi_{\wvec}}(s) \sum_a \pi_{\wvec}(a | s) q_{\pi_{\wvec}}(s, a)$ which would probably make such an analysis more interesting.}. In such cases, techniques (also see Ciosek and Whiteson, 2020) such as maintaining another function approximator, analogous to a value function, or sampling a subset of actions to approximate this expectation can be used.

      \subsection{Additional Details for Implementing Online Linear AC} \label{app: additional_details_online_ac}
      In this section, we provide details that can be used when implementing a PG agent (we used these for the experiments presented in \S \ref{app: experiments_linear}) with the estimators presented above. In particular, we calculate the term $\nabla_{\wvec} \thetavec(S)$ (see the equations in the previous section) to yield analytical policy gradient update expressions. To keep things concrete, we will focus on the expected policy gradient expression, and the regular and the alternate online AC estimators for softmax policies with linear function approximation. These details can be extended in a straightforward way to work with other PG algorithms such as REINFORCE and different policy parameterizations such as the tabular policies.
      
      Recall that the softmax policy is given by $\pi(a | s) = \frac{e^{[\thetavec(s)]_a}}{\sum_{c \in \mathcal{A}} e^{[\thetavec(s)]_c}}$. For linear function approximation and $d$-dimensional state features, $[\thetavec(s)]_a := \wvec_a^\top \xvec(s)$, with $\wvec_a \in \mathbb{R}^d$ being the weight vector corresponding to action $a$, and $\xvec(s) \in \mathbb{R}^d$  being the feature vector for state $s$ (obtained, say, via tile-coding). Also let $\Wvec := [\wvec_1\; \wvec_2\; \cdots\; \wvec_{|\mathcal{A}|}]$ be a $d \times |\mathcal{A}|$ weight matrix. Further, overload the gradient notation, so that for a scalar function $\lambda(\Wvec)$, $\nabla_{\Wvec} \lambda$ is a $d \times |\mathcal{A}|$ matrix defined as $[\nabla_{\Wvec} \lambda]_{ij} = \nicefrac{\partial \lambda}{\partial w_{ij}}$. Using this definition, $\nabla_{\Wvec} \lambda = [\nabla_{\wvec_1} \lambda \; \nabla_{\wvec_2} \lambda \; \cdots \; \nabla_{\wvec_{|\mathcal{A}|}} \lambda]$, which in turn means that we can first calculate $\nabla_{\wvec_c} \mathcal{J}_\pi$ for each action $c \in \mathcal{A}$ and then stack these vectors to obtain $\nabla_{\Wvec} \mathcal{J}_\pi$. The gradient update, for the weights $\Wvec$, can be written as $\Wvec^{\text{new}} = \Wvec^{\text{old}} + \alpha \nabla_{\Wvec} \mathcal{J}_\pi$ for some $\alpha > 0$. Using this formulation, let us now compute the expressions for the expected policy gradient, the regular estimator, and the alternate estimator.

      First note that, using Proposition \ref{prop: dot_product_grad}, for an arbitrary state $s$, and actions $a$ and $c$ we get that
      \begin{equation}
        \nabla_{\wvec_c} [\thetavec(s)]_a = \nabla_{\wvec_c} \big( \wvec_a^\top \xvec(s) \big) = \mathbb{I}(c = a) \xvec(s). \label{eq: grad_w_theta_s_a}
      \end{equation}
      Now for computing the expected policy gradient, put Eq. \ref{eq: mdp_vector_grad_softmax_mid} in Eq. \ref{eq: explicit_pg_matrix_vec} to obtain
      \begin{align*}
        \nabla_{\wvec_c} \mathcal{J}_\pi &= \sum_{s} \nu_\pi(s) \big[ \nabla_{\wvec_c} \thetavec(s) \big] \pivec(\cdot | s) \odot (\qvec_\pi(s, \cdot) - v_\pi(s) \mathbf{1}) \\
        &= \sum_{s} \nu_\pi(s) \sum_a \nabla_{\wvec_c} [\thetavec(s)]_a \pi(a | s) (q_\pi(s, a) - v_\pi(s)) \\
        &= \sum_{s} \nu_\pi(s) \sum_a \mathbb{I}(c = a) \xvec(s) \pi(a | s) (q_\pi(s, a) - v_\pi(s)) \\
        &= \sum_{s} \nu_\pi(s) \xvec(s) \pi(c | s) (q_\pi(s, c) - v_\pi(s)).
      \end{align*}
      Now we stack together the vectors $\nabla_{\wvec_c} \mathcal{J}_\pi$ for all the actions $c \in \mathcal{A}$ to obtain
      \begin{equation}
        \nabla_{\Wvec} \mathcal{J}_\pi = \sum_{s} \nu_\pi(s) \xvec(s) \Big( \pivec_{\wvec}(\cdot | s) \odot \big(\qvec_\pi(s, \cdot) - v_\pi(s) \mathbf{1} \big) \Big)^\top.
      \end{equation}
      Using chain rule of differentiation for vector variables (Proposition \ref{prop: chain_rule_grad}), we obtain
      \begin{IEEEeqnarray*}{lrCl}
        & \nabla_{\wvec_c} \pi(a|s) &=& \big[ \nabla_{\wvec_c} \thetavec(s) \big] \big[ \nabla_{\thetavec(s)} \pi(a|s) \big] = \sum_{d} \nabla_{\wvec_c} [\thetavec(s)]_d \cdot \frac{\partial \pi(a|s)}{\partial [\thetavec(s)]_d} \\
        &&=& \sum_d \mathbb{I}(c = d) \xvec(s) \cdot \pi(a|s) \big[ \mathbb{I}(d = a) - \pi(d | s) \big] = \pi(a|s) \xvec(s) \big[ \mathbb{I}(c = a) - \pi(c | s) \big] \\
        \Rightarrow \quad & \nabla_{\wvec_c} \log \pi(a|s) &=& \xvec(s) \big[ \mathbb{I}(c = a) - \pi(c | s) \big].
      \end{IEEEeqnarray*}
      Again stacking the vectors $\nabla_{\wvec_c} \log \pi(a|s)$ for all actions $c \in \mathcal{A}$ gives us
      \begin{equation}
        \nabla_{\Wvec} \log \pi(a|s) = \xvec(s) \big(\evec_a - \pivec(\cdot|s) \big)^\top. \label{eq: grad_W_log_pi}
      \end{equation}
      Finally, we can compute the regular estimator for online actor-critic. Using the expression for the regular online AC estimator (given in Eq. \ref{eq: regular_online_ac_estimator_app}) for arbitrary $a \in \mathcal{A}$, $r \in \mathcal{R}$, $k \in \{0, 1, \ldots\}$, and $s, s' \in \mathcal{S}$ we obtain
      \begin{align}
        \gvec^{\text{REG-OAC}}(s, a, r, s', k) &:= \gamma^k \nabla_{\Wvec} \log \pi(a | s) \big( r + \gamma v_\pi(s') - v_\pi(s) \big) \nonumber \\
        &= \gamma^k \big( r + \gamma v_\pi(s') - v_\pi(s) \big) \xvec(s) \big(\evec_a - \pivec(\cdot|s) \big)^\top.
      \end{align}
      If we stack the vectors $\nabla_{\wvec_c} [\thetavec(s)]_a$ from Eq. \ref{eq: grad_w_theta_s_a} for different actions $c$, we would get
      \begin{equation*}
        \nabla_{\Wvec} [\thetavec(s)]_a = \xvec(s) \evec_a^\top.
      \end{equation*}
      Then the alternate estimator for online actor-critic (given in Eq. \ref{eq: alternate_online_ac_estimator}) can be computed as 
      \begin{align}
        \gvec^{\text{ALT-OAC}}(s, a, r, s', k) &:= \gamma^k \nabla_{\Wvec} [\thetavec(s)]_{a} \big( r + \gamma v_\pi(s') - v_\pi(s) \big) \nonumber \\
        &= \gamma^k \big( r + \gamma v_\pi(s') - v_\pi(s) \big) \xvec(s) \evec_a^\top.
      \end{align}

      \subsection{Entropy Regularized Policy Gradient} \label{app: additional_details_online_ac}
      In maximum entropy policy gradient (Ahmed et al., 2019; Haarnoja et al., 2018; Peters et al., 2010), the agent maximizes the entropy regularized objective $\mathcal{J}^{\text{ent}}$ defined by:
      \begin{equation}
        \mathcal{J}^{\text{ent}} := \E \left[ \sum_{t=0}^{T-1} \gamma^t \Big( R_{t+1} + \tau \mathcal{H}_\pi(S_t) \Big) \right],
      \end{equation}
      where $\mathcal{H}_\pi(s) = - \sum_{a \in \mathcal{A}} \pi(a|s) \log \pi(a|s)$ is the action-entropy of the policy\footnote{Further, we explicitly define the entropy function to be zero at the terminal state $s_\times$, i.e. $\mathcal{H}_\pi(s_\times) := 0$.} and $\tau$ is the entropy regularization parameter. Note that this objective naturally arises if the agent were to use an entropy augmented reward $r^{\text{ent}}(R_{t+1}, S_t) = R_{t+1} + \tau \mathcal{H}_\pi(S_t)$. Consequently, we can also define special value functions that track this augmented reward:
      \begin{align}
        q_\pi^{\text{ent}}(s, a) &:= \E \bigg[ \sum_{t=k}^{T-1} \gamma^{t-k} \Big( R_{t+1} + \tau \mathcal{H}_\pi(S_t) \Big) \bigg | S_k = s, A_k = a \bigg], \text{and} \\
        v_\pi^{\text{ent}}(s) &:= \E \bigg[ \sum_{t=k}^{T-1} \gamma^{t-k} \Big( R_{t+1} + \tau \mathcal{H}_\pi(S_t) \Big) \bigg | S_k = s \bigg].
      \end{align}

      Using this objective, the policy gradient can be shown to be equal to
      \begin{equation}
        \nabla \mathcal{J}^{\text{ent}} = \sum_{s} \nu_\pi(s) \sum_{a} \pi(a|s) \Big[ \nabla \log \pi(a | s) \big(q_{\pi}^{\text{ent}}(s, a) - v_{\pi}^{\text{ent}}(s) \big) + \tau \nabla \mathcal{H}_\pi(s) \Big].
      \end{equation}
      Compare this with the regular policy gradient expressions given in Eqs. \ref{eq: pg_main} and \ref{eq: pg_expectation}. Let us now compute the analytical expression for the gradient of the entropy term:
      \begin{align}
        \nabla_{\Wvec} \mathcal{H}_\pi(s) &= - \nabla_{\Wvec} \sum_{a \in \mathcal{A}} \pi(a|s) \log \pi(a|s) \nonumber \\
        &= - \sum_{a} \nabla_{\Wvec} \pi(a|s) \log \pi(a|s) - \cancelto{0}{\sum_{a} \pi(a|s) \nabla_{\Wvec} \log \pi(a|s)} = - \sum_{a} \nabla_{\Wvec} \pi(a|s) \log \pi(a|s) \nonumber \\
        &= - \sum_{a} \pi(a|s) \nabla_{\Wvec} \log \pi(a|s) \log \pi(a|s) = - \xvec(s) \sum_a \big(\evec_a - \pivec(\cdot|s) \big)^\top \pi(a|s) \log \pi(a|s) \tag*{(from Eq. \ref{eq: grad_W_log_pi})}  \\
        &= - \xvec(s) \left[ \sum_a \evec_a \pi(a|s) \log \pi(a|s) - \pivec(\cdot|s) \sum_a \pi(a|s) \log \pi(a|s) \right]^\top \nonumber \\
        &= - \xvec(s) \Big[ \pivec(\cdot|s) \odot \log \pivec(\cdot|s) + \mathcal{H}_\pi(s) \pivec(\cdot|s) \Big]^\top.
      \end{align}
      With this expression, we can now easily extend the analysis of the previous section to implement, say, entropy regularized online linear actor-critic.

      \subsection{Escort Transform} \label{app: escort_pg}
      The escort transform (Mei et al., 2020b) is a policy parameterization that is theoretically guaranteed to escape policy saturation at a fast rate. An escort policy is given by
      \begin{equation}
        \pi(a|s) = \frac{|[\thetavec(s)]_a|^p}{\sum_{a'} |[\thetavec(s)]_{a'}|^p},
      \end{equation}
      where $p \geq 1$ is a tunable parameter. For large $p$, the escort transform would behave somewhat similarly to softmax, and as $p \rightarrow \infty$, an escort policy would would approach the greedy policy. Note that the action preferences for escort transform cannot all simultaneously be equal to zero (also see Theorem 4, Mei et al., 2020b). Let us begin by computing the gradient of this transform with respect to the action preferences (for brevity, we represent $[\thetavec(s)]_a$ as $\theta_a$ in the following calculation):
      \begin{align}
        \frac{\partial}{\partial \theta_d} \pi(a|s) &= \frac{\partial}{\partial \theta_d} \frac{|\theta_a|^p}{\sum_{a'} |\theta_{a'}|^p} = \frac{\mathbb{I}(a=d) \sgn(\theta_a) \cdot p |\theta_a|^{p-1} \sum_{a'} |\theta_{a'}|^p - \sgn(\theta_d) \cdot p |\theta_d|^{p-1} |\theta_a|^p}{\left( \sum_{a'} |\theta_{a'}|^p \right)^2} \nonumber \\
        &= \frac{p}{\sum_{a'} |\theta_{a'}|^p} \Big[ \mathbb{I}(a=d) \sgn(\theta_a) \cdot |\theta_a|^{p-1} - \sgn(\theta_d) \cdot |\theta_d|^{p-1} \pi(a|s) \Big],
      \end{align}
      where $\sgn: \mathbb{R} \rightarrow \{-1, 0, 1\}$ is the sign function that returns the sign of its input (or zero if the input is zero). Now, we will compute the escort transform analog of Eq. \ref{eq: grad_W_log_pi}:
      \begin{IEEEeqnarray*}{lrCl}
        & \nabla_{\wvec_c} \pi(a|s) &=& \big[ \nabla_{\wvec_c} \thetavec(s) \big] \big[ \nabla_{\thetavec(s)} \pi(a|s) \big] = \sum_{d} \nabla_{\wvec_c} \theta_d \cdot \frac{\partial \pi(a|s)}{\partial \theta_d} \\
        &&=& \sum_d \mathbb{I}(c = d) \xvec(s) \cdot \frac{p}{\sum_{a'} |\theta_{a'}|^p} \Big[ \mathbb{I}(a=d) \sgn(\theta_a) \cdot |\theta_a|^{p-1} - \sgn(\theta_d) \cdot |\theta_d|^{p-1} \pi(a|s) \Big] \\
        &&=& \frac{p \cdot \xvec(s)}{\sum_{a'} |\theta_{a'}|^p} \Big[ \mathbb{I}(a=c) \sgn(\theta_a) \cdot |\theta_a|^{p-1} - \sgn(\theta_c) \cdot |\theta_c|^{p-1} \pi(a|s) \Big] \\
        \Rightarrow \quad & \nabla_{\wvec_c} \log \pi(a|s) &=& p \cdot \xvec(s) \left[ \mathbb{I}(a=c) \sgn(\theta_a) \cdot \frac{|\theta_a|^{p-1}}{|\theta_a|^p}  - \sgn(\theta_c) \cdot |\theta_c|^{p-1} \frac{\pi(a|s)}{|\theta_a|^p} \right] \\
        &&=& p \cdot \xvec(s) \left[ \frac{\mathbb{I}(a=c)}{\theta_a} - \frac{\sgn(\theta_c) \cdot |\theta_c|^{p-1}}{\sum_{a'} |\theta_{a'}|^p} \right].
      \end{IEEEeqnarray*}
      Finally, we stack the vectors $\nabla_{\wvec_c} \log \pi(a|s)$ for all actions $c \in \mathcal{A}$ to obtain
      \begin{equation}
        \nabla_{\Wvec} \log \pi(a|s) = p \cdot \xvec(s) \left[ \frac{\evec_a}{[\thetavec(s)]_a} - \frac{\sgn(\thetavec(s)) \odot |\thetavec(s)|^{p-1}}{\sum_{a'} |[\thetavec(s)]_{a'}|^p} \right]^\top. \label{eq: grad_W_log_pi_escort}
      \end{equation}
      Using this expression, we can directly implement online linear AC with an escort policy.
      
      \section{Gradient Bandits: Full Experiments} \label{app: experiments_bandits}
      We then give the experimental details and the algorithm pseudocode used in \S \ref{sec: bandit_experiments}. Finally, we give additional experiments and discuss them in more detail.
      
      \subsection{Experimental Details}
      This section presents multiple experiments with gradient bandits on different 3-armed bandit tasks. All the bandit problems had the action set $\mathcal{A} = \{a_0, a_1, a_2\}$, a fixed reward structure $\rvec$, and a normally distributed reward noise $\epsilon$. The experimental results serve as a demonstration and the verification of the properties of the gradient bandits estimators discussed in the main paper.

      By default, in all the experiments, we trained five different gradient bandit agents with a softmax policy: for different agents, the policy weights were updated using either the expected gradient $\nabla \mathcal{J} = \pivec \odot (\rvec - r_\pi)$, the regular estimator with true average reward $\gvec^{\textrm{REG}}(R, A) = (R - r_\pi) (\evec_A - \pivec)$ or a learned baseline $\hat{\gvec}^{\textrm{REG}}(R, A) = (R - b) (\evec_A - \pivec)$, or the alternate estimator with the true average reward $\gvec^{\textrm{ALT}}(R, A) = \evec_A (R - r_\pi)$ or a learned baseline $\hat{\gvec}^{\textrm{ALT}}(R, A) = \evec_A (R - b)$. The baseline in each case was learned using a running average: $b_{t+1} = (1 - \beta) b_t + \beta R_t$, with the stepsize $\beta$. The exact pseudocode for gradient bandits (used in the experiments of \S \ref{sec: bandit_experiments}) is given in Algorithm \ref{alg: grad_bandit_reg}. The action preferences for the softmax policy were initialized to $\theta_{a_1} = \theta_{a_2} = 0$; and $\theta_{a_0} = 0$ for uniform policy initialization or $\theta_{a_0} \in \{5, 10, 50\}$ for saturated policy initialization depending on the degree of saturation.
      
      Each experiment was run for 1000 timesteps. For each estimator, we performed a sweep over the policy stepsize $\alpha \in \{2^{-6}, 2^{-5}, \ldots, 2^1\}$ and the baseline stepsize $\beta \in \{2^{-4}, 2^{-3}, \ldots, 2^0\}$. The sensitivity plots show the mean final performance on the last 50 timesteps averaged over 150 runs for different parameter settings. To reduce clutter, we didn't label all the $\beta$s in the sensitivity plots.

      \begin{algorithm}[!hbt]
        \caption{Gradient Bandit with Regular Estimator}
        \label{alg: grad_bandit_reg}
        \begin{algorithmic}
          \State Input a softmax policy $\pi_{\thetavec}(a) := \frac{e^{\theta_a}}{\sum_{b \in \mathcal{A}} e^{\theta_b}}$ and an average reward function estimate $\hat{r}_\pi$
          \State Input the policy stepsize $\alpha$ and the average reward stepsize $\beta$
          \State Initialize $\thetavec$ and $\hat{r}_\pi$
          \For{total number of timesteps}
          \State Sample an action $A \sim \pi$ and a reward $R \sim p(\cdot | A)$
          \For{each action $a \in \mathcal{A}$}
          \State \texttt{\textcolor{purple}{\# Calculate the gradient estimator}}
          \State $\hat{g}_a^{\textrm{REG}} = \big(R - \hat{r}_\pi \big) \Big( \mathbb{I}(A = a) - \pi(a) \Big)$
          \State \texttt{\textcolor{purple}{\# Update the policy parameter}}
          \State $\theta_a \leftarrow \theta_a + \alpha \cdot \hat{g}_a^{\textrm{REG}}$
          \EndFor
          \State \texttt{\textcolor{purple}{\# Update the average reward estimate}}
          \State $\hat{r}_\pi \leftarrow \beta R + (1 - \beta) \hat{r}_\pi$
          \EndFor
        \end{algorithmic}
      \end{algorithm}
      
      For both the estimators, we swept the policy stepsize $\alpha \in \{2^{-6}, 2^{-5}, \ldots, 2^1\}$ and the reward baseline stepsize $\beta \in \{2^{-4}, 2^{-3}, \ldots, 2^0\}$. Although, we do not always label all the $\beta$s in the sensitivity plots to reduce the clutter. The policy was initialized to $\theta_{a_1} = \theta_{a_2} = 0$; and $\theta_{a_0} = 0$ (for uniform), $\theta_{a_0} \in \{0, 5, 10, 50\}$ for saturated policies depending on the degree of saturation.
      
      \subsection{Alternate Estimator can Escape Saturated Policy Regions} \label{sec: bandits_alternate_estimator_noise_and_no_noise_first_exp}
      Experiments in this section demonstrate that the alternate estimator performs competitively with the regular estimator for uniform policy initialization and clearly outperforms it in the case of a saturated policy initialization when there is reward noise present. They also show that in the absence of reward noise, the alternate estimator fails to escape the saturated policy regions. The bandit reward structure is $\rvec = [0\;0\;1]^\top$ with either normal noise $\epsilon \sim \mathcal{N}(0, 1)$ or zero reward noise.

      \subsubsection*{Learning Curves}
      Figure \ref{fig: bandit_learning_curves} shows the learning curves for the best parameter configuration based on final performance (see sensitivity plots) for the five different agents using different PG estimators for two settings: uniform policy initialization and a sub-optimally saturated policy initialization. The regular and alternate estimators, used either the true reward function $r_\pi$ or learned a reward baseline that was initialized to zero. For the uniform policy case, we observe that all the methods achieved a good performance. More interestingly, for saturated policy initialization, the alternate estimators (both $r_\pi$ and with a baseline $b$) learned good behaviors, whereas the expected PG and the regular estimators failed to learn anything. Further, alternate with a learned baseline converged faster than alternate with $r_\pi$. We attribute this result to a better exploration afforded by the baseline.    
      \begin{figure}[!hbp]
        \centering
        \includegraphics[scale=0.45]{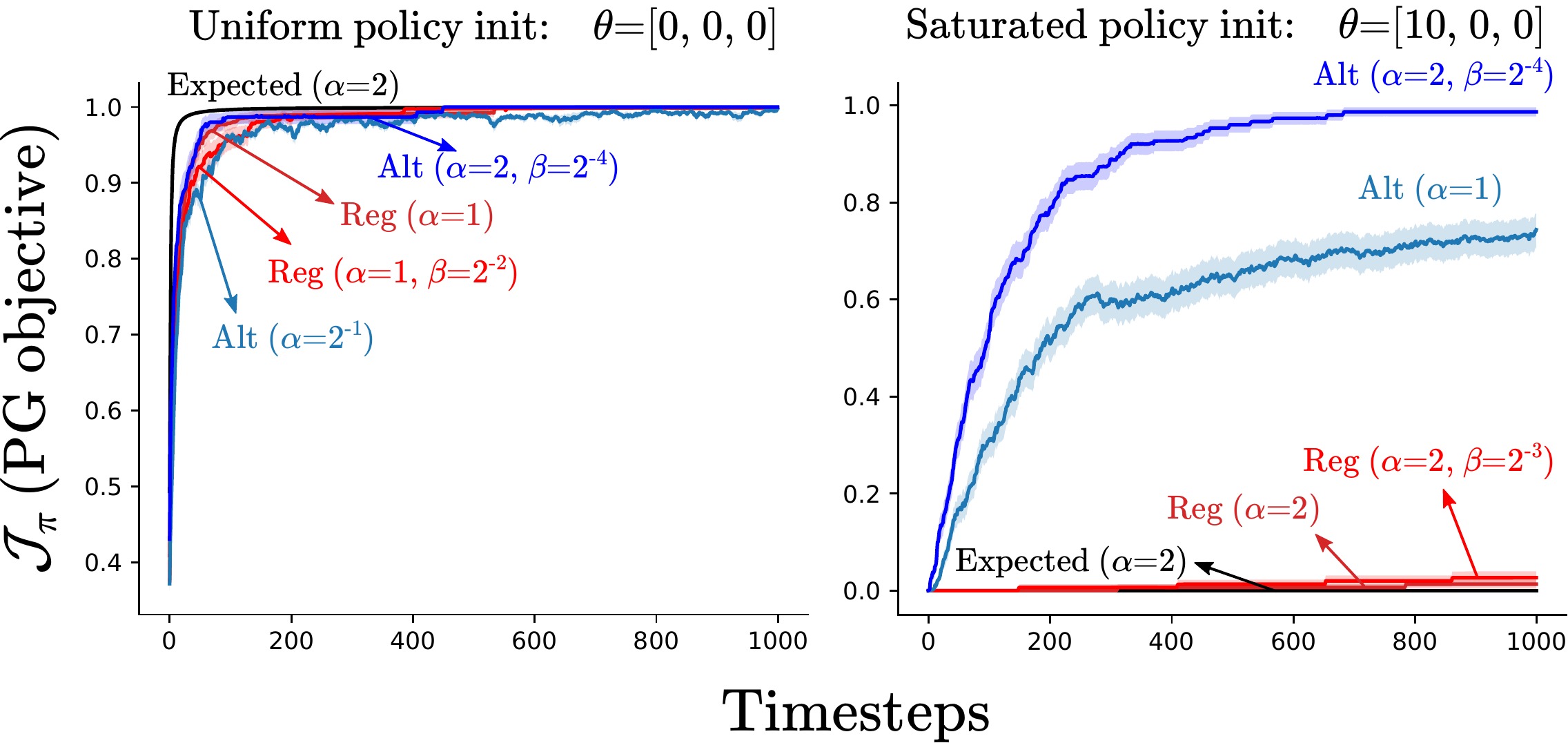}
        \caption{Learning curves for gradient bandits. The reward structure is $\rvec = [0\;0\;1]^\top$ with noise $\epsilon \sim \mathcal{N}(0, 1)$. The estimators without an accompanying value of $\beta$ use the true $r_\pi$. The results are averages over 150 independent runs and the shaded regions represent the standard error.} 
        \label{fig: bandit_learning_curves}
      \end{figure}
      
      \subsubsection*{Parameter Sensitivity Plots and Performance with and without Reward Noise} \label{sec: bandit_noise_plus_no_noise}
      Figure \ref{fig: param_sensitivity_r001} shows the parameter sensitivity of the gradient estimators for different policy saturations and reward noises. The graphs plot the mean final performance\footnote{For small stepsizes, the final performance of all the methods are very similar; it is only at higher stepsizes that we start to see a deviation of the performance of both the stochastic PG estimators from each other and the expected gradient bandits update; for a detailed discussion on stochasticity in PG methods, also see Mei et al. (2021).} during last 50 steps of the 1000 timestep run of each method. The performance was measured by analytically calculating $\mathcal{J}_\pi = \E_\pi[R]$. We observe that for uniform initialization, all the methods (irrespective of the magnitude of reward noise) had a similar performance. In contrast, for saturated policies, the alternate estimators with reward noise were able to learn for a range of stepsizes. However, without reward noise, none of the methods escaped the saturation and their policies did not improve. Note that the baseline was initialized to zero which is not optimistic for $\rvec = [0\;0\;1]^\top$ and therefore the alternate estimator could not utilize the bias in the baseline (as discussed in \S \ref{sec: alternate_utilizes_biasedness}). This experiment verifies our claim that reward noise helps the alternate estimator to escape the sub-optimal regions in the policy space; without reward noise, the agent is unable to perform a random walk.
      \begin{figure}[!tbp]
        \centering
        \includegraphics[scale=0.3]{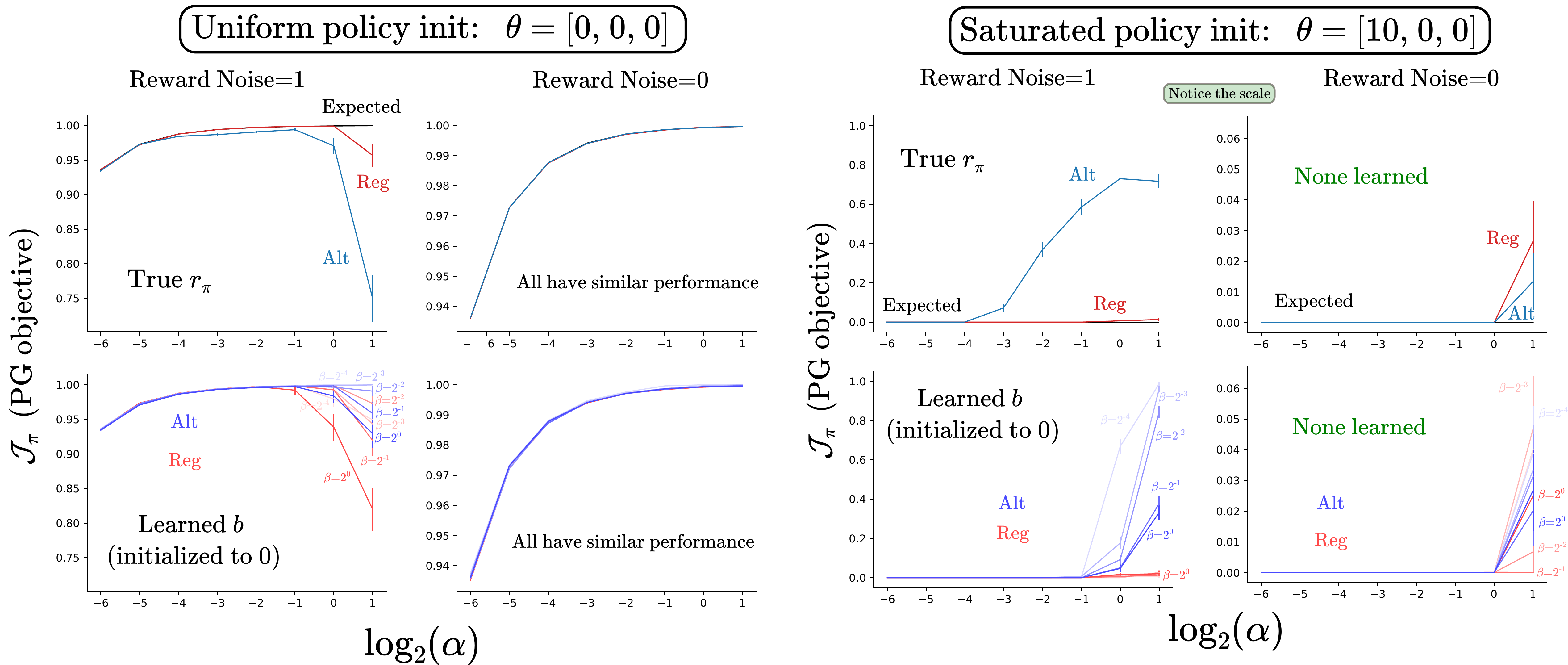}
        \caption{Parameter sensitivity plots for the different estimators on the bandit problem with $\rvec = [0\;0\;1]^\top$. We show results for eight different combinations: the estimators using $r_\pi$ or a learned baseline, uniform or saturated policy initialization, and with or without reward noise. The results were averaged over 150 runs and the bars represent standard error. The $x$ axis shows the policy stepsize $\alpha$ and the curve labels show the reward baseline stepsize $\beta$.} 
        \label{fig: param_sensitivity_r001}
      \end{figure}

      \subsubsection*{Behavior of Action Preferences under Different Estimators}
      To understand how the action preferences are influenced by different estimators, we plotted their values in Figure \ref{fig: uniform_init_learning_curves} (with non-zero reward noise) and Figure \ref{fig: uniform_init_no_noise_learning_curves} (with zero reward noise) using a set of good performing stepsize parameters. We observe that the behavior of the action preferences for the regular estimator with and without reward noise is similar: the preferences increased quite a bit during the first 100 steps and then did not change much, resulting in flat lines. This probably happened because the policy had saturated and the estimator became negligibly small, i.e. with high probability $\evec_A \approx \pivec$ and $R \approx r_\pi$, and thus $\gvec^{\text{REG}} = (R - r_\pi) (\evec_A - \pivec) \approx \zerovec$. In contrast, reward noise affected the behavior of the alternate estimator (with $r_\pi$) quite strongly: the graph for $\theta_2$ with reward noise demonstrates the characteristic random walk pattern. Whereas, in the case of no reward noise, this graph becomes smooth and keeps increasing, since even at saturation, because of missing the $- \pivec$ term, the alternate estimator $\gvec^{\text{ALT}} = (R - r_\pi) \evec_A$ remains somewhat large. The patterns for alternate with a learned $b$ are much less clear, probably because of additional effects from the bias in the baseline. We also observe that the alternate estimator with reward noise is noisier than the regular variant; this is expected because the random walk style behavior gives it a higher variance than the regular version (refer Figure \ref{fig: intuition_bandits_reward001_noise1}). Alternate without reward noise seems to be less noisy as compared to the regular variant which is again reasonable since in absence of reward noise, alternate has a lower variance than the regular estimator (refer Figure \ref{fig: intuition_bandits_reward001_noise0}).
      \begin{figure}[!tbp]
        \centering
        \includegraphics[scale=0.53]{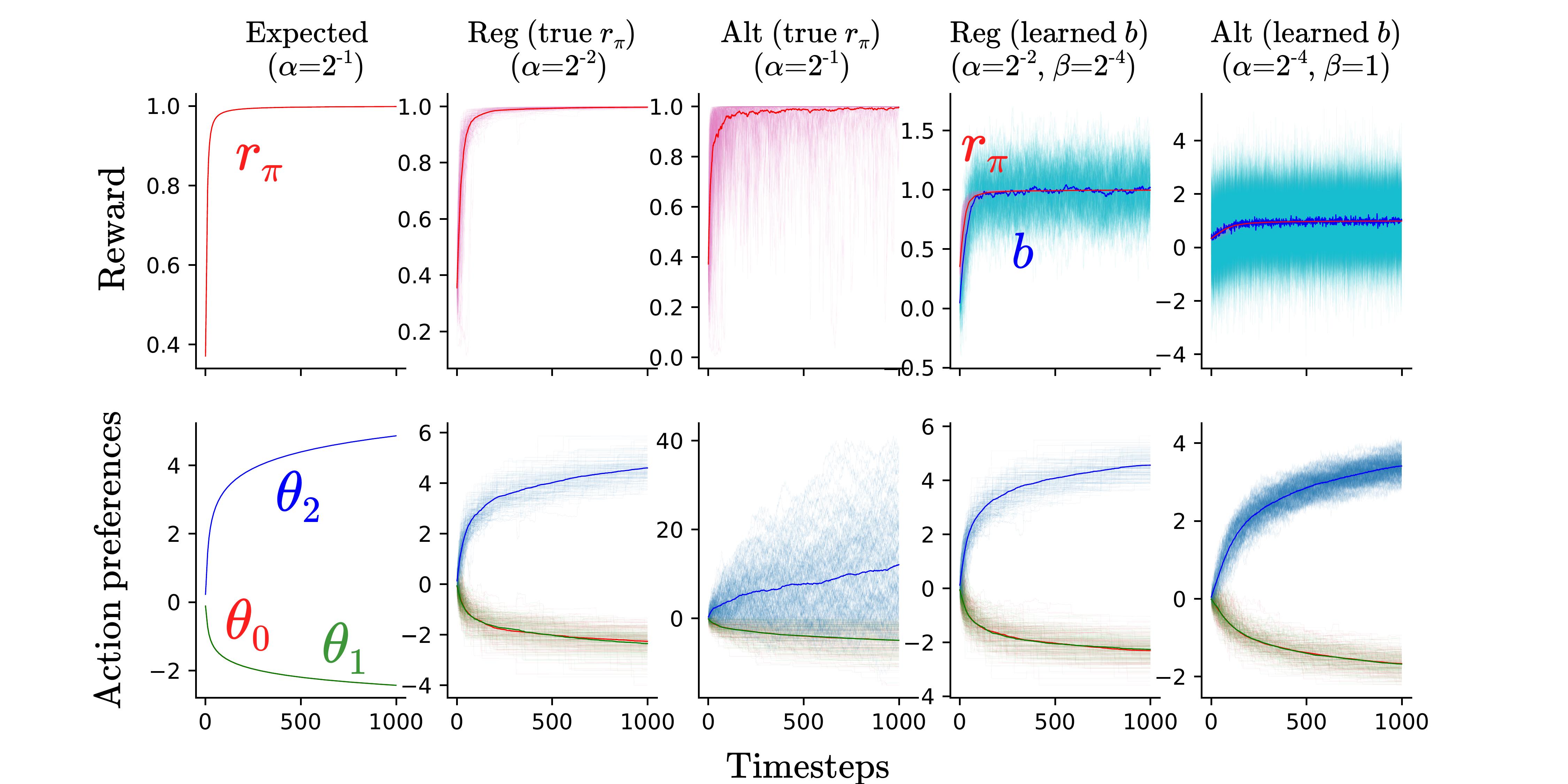}
        \caption{\textbf{(Top)} Learning curves for expected reward $r_\pi$ and value of the baseline $b$ as the policy $\pi$ is updated using different estimators on the bandit problem $\rvec = [0\;0\;1]^\top$ with noise $\epsilon \sim \mathcal{N}(0, 1)$. \textbf{(Bottom)} Learning curves for action preferences corresponding to the different estimators. In all cases, the finer lines depict the 150 individual runs and the thick lines show the averages.} 
        \label{fig: uniform_init_learning_curves}
      \end{figure}
      
      \begin{figure}[!hbp]
        \centering
        \includegraphics[scale=0.53]{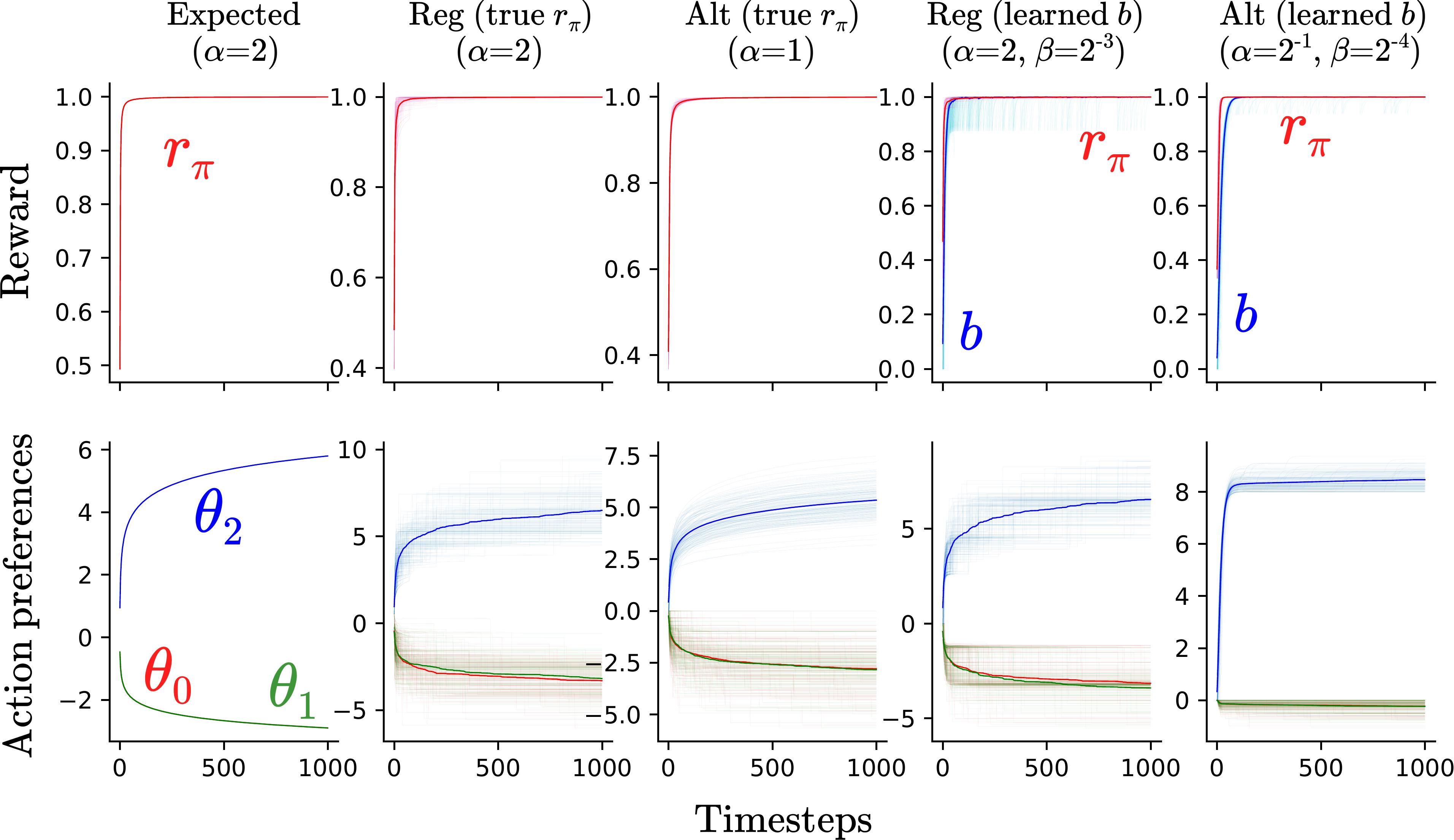}
        \caption{Average reward $r_\pi$, value of the learned baseline $b$, and the action preferences against timesteps for different methods on the bandit problem $\rvec = [0\;0\;1]^\top$ and no reward noise.}
        \label{fig: uniform_init_no_noise_learning_curves}
      \end{figure}
      
      \subsection{Gradient Bandit Estimators with Increasingly Saturated Policy Initializations}
      In this experiment, we trained the estimators on the bandit task $\rvec = [1\;2\;3]^\top$ with noise $\epsilon \sim \mathcal{N}(0, 1)$ and four different policy saturations, ranging from a uniform policy to a policy highly saturated on the worst action. Figure \ref{fig: multiple_bad_init_param_study} shows the parameter sensitivity plots depicting the final performance of the methods during the last 50 steps of their 1000 timestep run. The results show that as the policy saturation increases, the performance of all the methods degrade. Further, the alternate estimator (with true $r_\pi$) is able to handle high saturations (except for $\thetavec^{\text{init}} = [5\;0\;0]^\top$) better than the regular estimator. Whereas, the alternate estimator with a learned baseline initialized to zero (which is pessimistic for this reward structure) did not perform well. This figure, with a baseline initialized to $b = 0$, also serves as a standard (or a control set) against which the results of the next section can be compared, where we explicitly study the effect of an optimistic $(b = +4)$ and a pessimistic $(b = -4)$ reward baseline on the two estimators' performance.
      \begin{figure}[!hbp]
        \centering
        \includegraphics[scale=0.3]{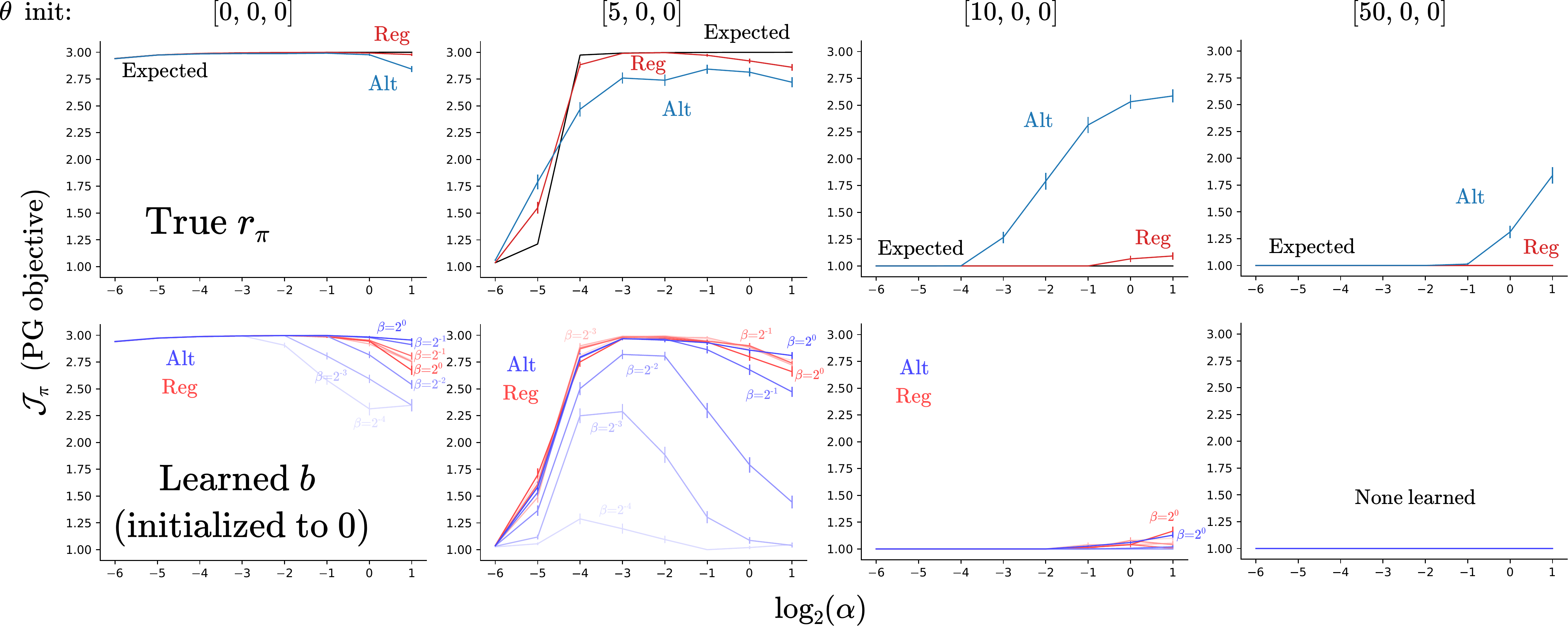}
        \caption{Parameter sensitivity of the PG estimators against increasingly saturated policy initializations on the 3-armed bandit problem $\rvec = [1\;2\;3]^\top$ with $\epsilon \sim \mathcal{N}(0, 1)$. The top row shows the performance of the estimators with the true $r_\pi$. Bottom row shows the performance for a learned baseline, initialized to zero. The subplot headers show how the action preferences were initialized.} 
        \label{fig: multiple_bad_init_param_study}
      \end{figure}

      \subsection{Effect of Bias in the Reward Baselines on the Estimators} \label{sec: opti_pessi_bandit_experiments}
      In this section, we study the effect of initializing the reward baseline optimistically ($b = +4$) or pessimistically ($b=-4$). The reward was set to $\rvec = [1\;2\;3]^\top$ with noise $\epsilon \sim \mathcal{N}(0, 1)$\footnote{The results for the bandit problem with optimistic or pessimistic baselines, and no noise (not shown here) were qualitatively similar to the ones presented here; the signal from the bias in the baseline is quite strong and affects both the estimators much more than the noise in the sampled rewards.}.

      \subsubsection*{Optimistically or Pessimistically Initialized Baseline}
      Figure \ref{fig: bad_init_optimistic_param_study} and Figure \ref{fig: bad_init_pessimistic_param_study} show the sensitivity plots for the regular and alternate estimators for an optimistically and a pessimistically initialized baseline respectively. These results show that the optimistic initialization greatly helped the alternate estimator in escaping saturated policy regions. Whereas, a pessimistic initialization significantly hampered the performance of the alternate estimator, making it worse than the regular estimator even for the uniform policy. In contrast, the baseline initialization affected the regular estimator in a milder way. Also with an optimistic initialization, the alternate estimator preferred smaller stepsizes allowing it to enjoy the optimism effect for longer; and vice-versa for a pessimistic initialization.
      
      These results also verify the fact that the alternate estimator is asymptotically unbiased: the sensitivity plots show that the agent was able to converge to the optimal corner (corresponding to action $a_2$) and thus obtained a reward of about 3. As the agent's estimate of $r_\pi$ improved, i.e. the baseline $b$ became closer to the average reward $r_\pi$, the bias in the alternate estimator disappeared.

      \begin{figure}[!tbp]
        \centering
        \includegraphics[scale=0.33]{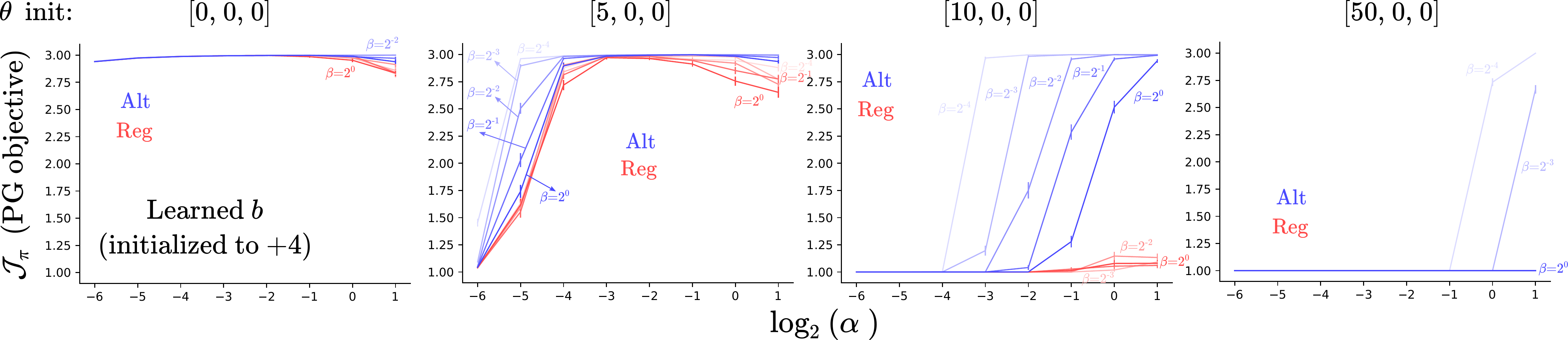}
        \caption{Parameter sensitivity plots for the different estimators for an \textbf{optimistically} initialized reward baseline ($b = +4$) on the three armed bandit problem $\rvec = [1\;2\;3]^\top$ and $\epsilon \sim \mathcal{N}(0, 1)$.} 
        \label{fig: bad_init_optimistic_param_study}
      \end{figure}

      \begin{figure}[!tbp]
        \centering
        \includegraphics[scale=0.33]{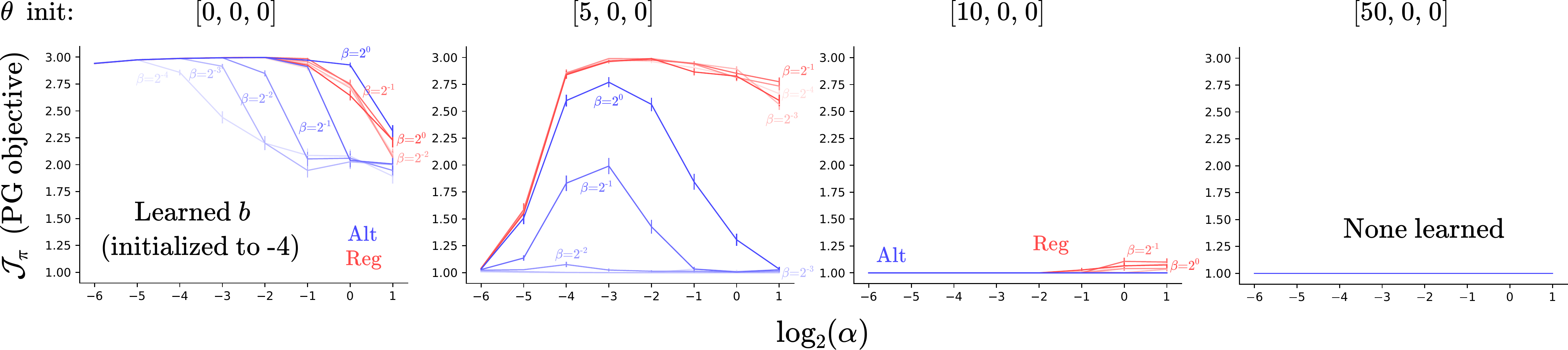}
        \caption{Parameter sensitivity plots for the different estimators for a \textbf{pessimistically} initialized reward baseline ($b = -4$) on the three armed bandit problem $\rvec = [1\;2\;3]^\top$ and $\epsilon \sim \mathcal{N}(0, 1)$.} 
        \label{fig: bad_init_pessimistic_param_study}
      \end{figure}

      \subsubsection*{Fixed Baseline} \label{sec: fixed_baseline_gradient_bandits}
      Now we study what happens when we fix the baseline and not learn it, thereby maintaining the biasedness of the alternate estimator. Figure \ref{fig: bias_alternate_param_study} shows the parameter sensitivity plots for this setting for multiple policy initializations. From the results, we see that the regular estimator was not affected by a fixed baseline (since it remains unbiased with any baseline): it was able to learn a good, but not the optimal, policy for a uniform policy initialization. In contrast, the alternate estimator was impacted heavily: for an optimistic baseline, it converged to a value of about 2.4 and for a pessimistic baseline it converged to a value of zero. The values to which the alternate estimator converged  remained the same irrespective of the policy initialization. This hints at the fact that the update signal offered by the alternate estimator, due to bias in the critic, is quite strong and can override the high saturation in softmax policies.
      
      \begin{figure}[!hbp]
        \centering
        \includegraphics[scale=0.45]{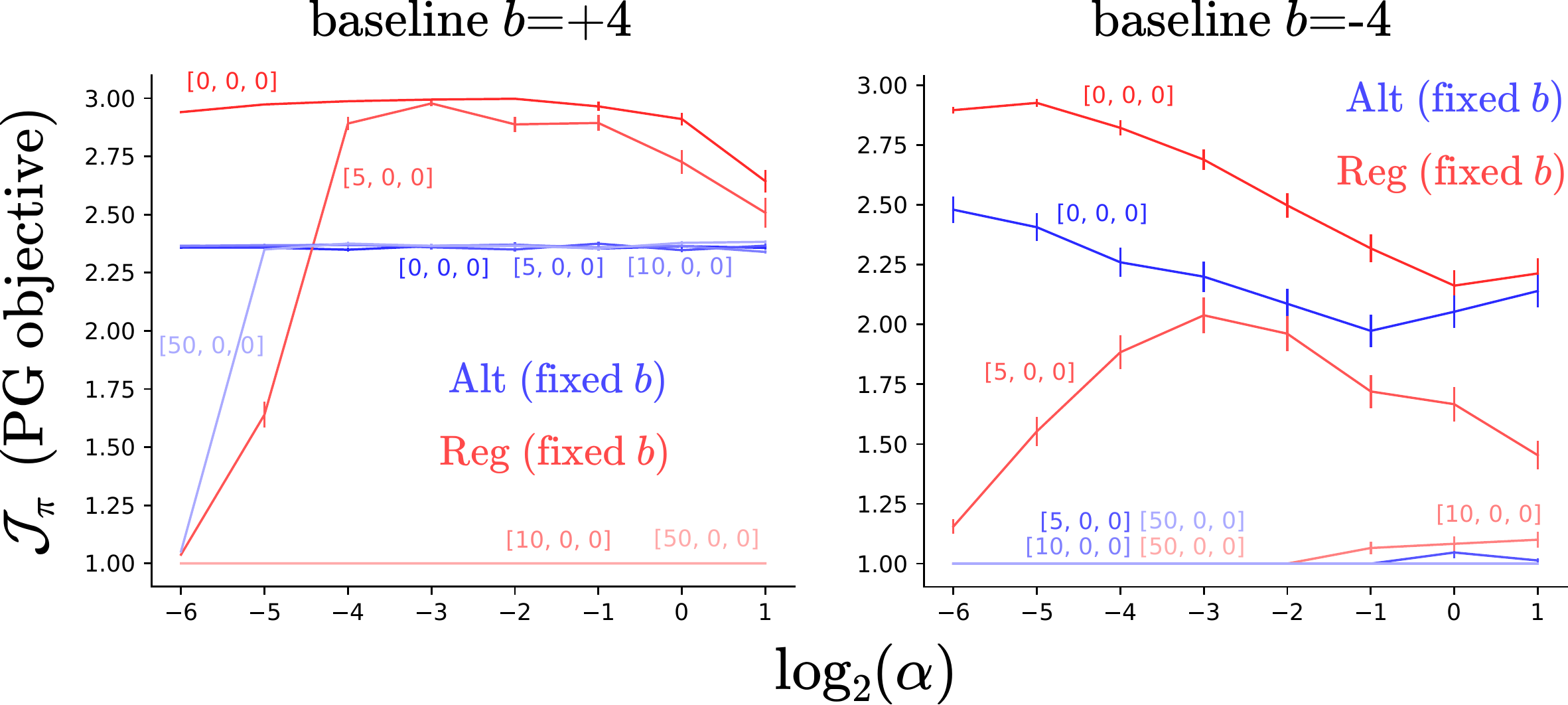}
        \caption{Bias in the alternate PG estimator. Parameter sensitivity of the PG estimators for a fixed baseline on bandit problem $\rvec = [1\;2\;3]^\top$ with $\epsilon \sim \mathcal{N}(0, 1)$. The baseline in these experiments was fixed and not updated, and therefore, the alternate estimator remained biased throughout the learning. The individual labels on the curves show the action preference initializations.} 
        \label{fig: bias_alternate_param_study}
      \end{figure}

      This experiment, additionally, serves as a verification of Lemma \ref{thm: fixed_pt_biased_pg}. Recall from the lemma that for a fixed baseline, the fixed point of the expected alternate update is given by $\pi^*(a) = \frac{1}{r(a) - b} \left( \sum_c \frac{1}{r(c) - b} \right)^{-1}$. Evaluating this expression for this experiment gives $\pivec^* = [0.18\; 0.27\; 0.55]^\top$. Using $\rvec = [1\;2\;3]^\top$, we can calculate $\mathcal{J}_{\pi^*} = {\pivec^*}^\top \rvec = 2.36$. And unsurprisingly, this value matches the value shown in Figure \ref{fig: bias_alternate_param_study} (left) to which the alternate estimator converged under various policy initializations. Further, as we had predicted from the update plots shown in Figure \ref{fig: intuition_bandits_reward123_noise0_opti_pessi} (bottom-right), the alternate estimator with a pessimistic baseline converged to the nearest corner. For policies saturated on action $a_0$, this corner corresponded to action $a_0$ as well and thus the agents obtained a reward of 1 for all the policy saturations considered.
      
      \subsection{Can the Regular Estimator be Modified to Deal with Saturated Policies?}
      In this section we consider whether certain modifications to the regular estimator can make it effective in dealing with saturated softmax policies.

      \subsubsection*{Dropping the Reward Baseline}
      From our previous experiments, we know that the higher variance of the alternate estimator helps it in escaping the corners of the probability simplex. So we investigate whether a similar increase in variance obtained by removing the reward baseline from the regular estimator will help it in escaping the corners as well. Figure \ref{fig: multiple_init_regular_no_baseline.pdf} shows the sensitivity of the final mean performance (during the last 50 timesteps of a 1000 timestep run) to the stepsize $\alpha$ for the regular estimator without any baseline on the bandit task $\rvec = [1\;2\;3]^\top$ with $\epsilon \sim \mathcal{N}(0, 1)$. The results show that the regular estimator without a baseline cannot be a good competitor to the alternate estimator for escaping the saturated policy regions. Upon careful thought, this should not be surprising: the regular estimator, with or without the baseline, has zero variance at simplex corners (see Proposition \ref{thm: variance_saturated_pg} and Figure \ref{fig: intuition_bandits_reward001_noise1}). Consequently, there is neither any gradient signal nor any noise near saturated regions and the policy gets stuck. This is unlike the alternate estimator which not just has a higher variance, but a variance structure that is better suited to escaping the saturated policy regions.    
      \begin{figure}[!hbp]
        \centering
        \includegraphics[scale=0.5]{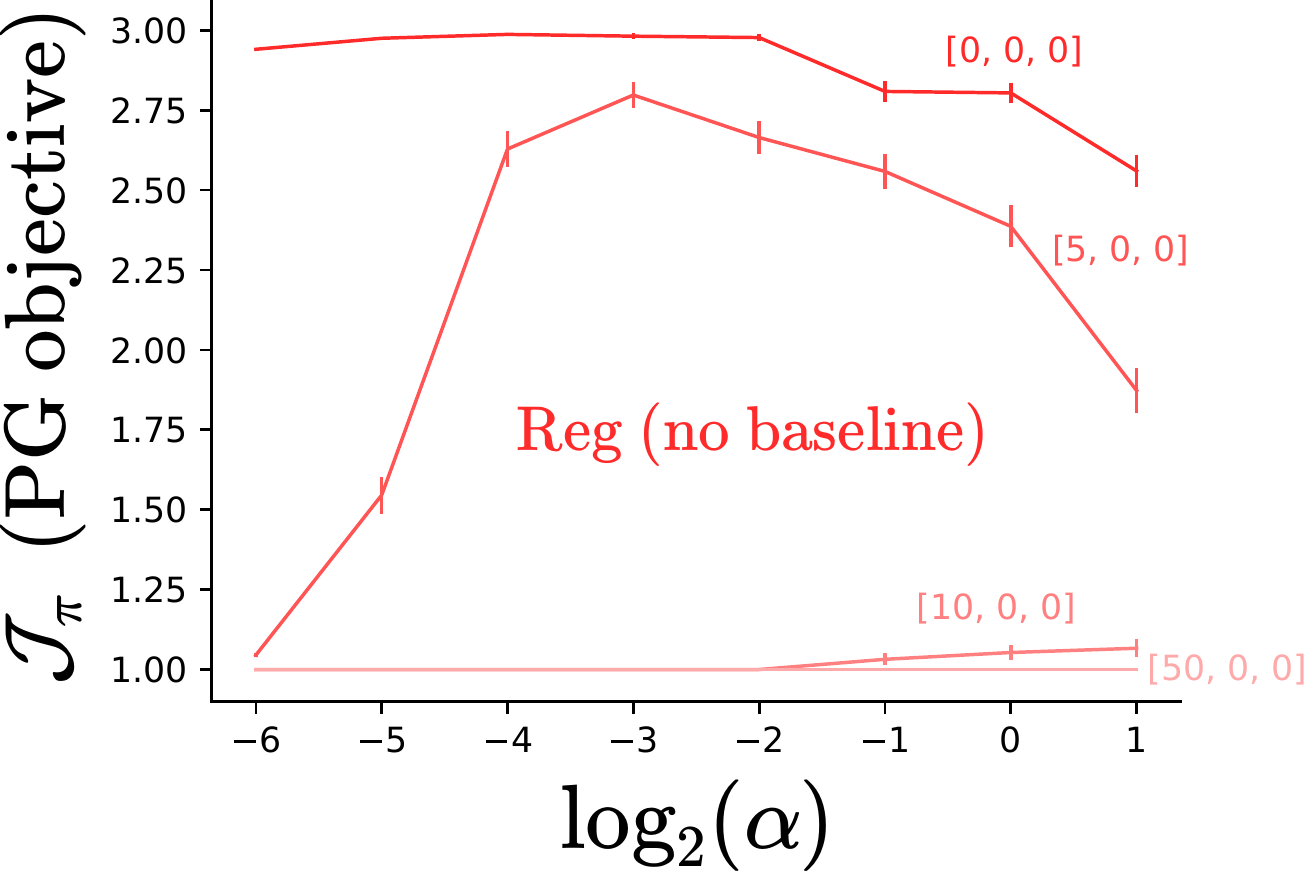}
        \caption{Parameter sensitivity plots for the regular estimator without a reward baseline. The individual curve labels show the action preference initializations.} 
        \label{fig: multiple_init_regular_no_baseline.pdf}
      \end{figure}

      \subsubsection*{Adding Normally Distributed Noise to the Policy Updates} \label{sec: gradient_noise_bandits}
      The reward noise utilized by the alternate estimator ultimately gets translated into noisy updates for the action preferences, whereas this translation doesn't go through very well for the regular estimator, especially at the corners (refer Example \ref{eg: saturated_pg}). Therefore, in this experiment, we study the effect of adding an artificial noise to the gradient updates for both the estimators. In particular, we trained the different estimators on the bandit problem $\rvec = [1\;2\;3]^\top$ with $\epsilon \sim \mathcal{N}(0, 1)$. And to make the updates noisy, we added a normally distributed noise to the gradient updates, i.e. for each action $a$, the action preferences were updated as follows: $\theta_a^{\text{new}} = \theta_a^{\text{old}} + \alpha (g_a + \xi)$, where $\xi \sim \mathcal{N}(0, \texttt{grad-noise})$; we also swept over the the parameter $\texttt{grad-noise} \in \{0, 0.1, 0.5, 1, 2\}$. We experimented with a policy saturated to $\thetavec = [10\;0\;0]^\top$ at the beginning of training; and additionally a uniform policy to gauge whether this additional update noise had a detrimental effect on the estimators' performance, since with a uniform policy there is already a strong gradient signal and the agent does not require any extraneous noise for moving towards the optimal policy. 
      
      Figure \ref{fig: competitor_true_rpi_regular_noisy} shows the sensitivity of the final performance for the expected gradient, and the regular and the alternate estimators using the true $r_\pi$. Figure \ref{fig: competitor_regular_noisy} shows similar results for the regular and alternate estimators with a learned baseline, initialized to zero. The results from this experiment suggest that adding a normal noise to the gradient bandit update works well for escaping the saturated policy regions in this bandit setting for all the five estimators; in particular, the regular estimator can escape the saturation with a high enough noise. At first, this observation might seem to weaken the case for the alternate estimator. However, note that adding gradient noise also worsens the performance of all the estimators for the uniform policy case, which is clearly an undesirable property. Now recall that the alternate estimator benefits much more from an optimistic baseline than it does from the reward noise. In contrast to the noise, an optimistic baseline actually offers an clear signal that moves the agent out of saturation and towards a more uniform policy distribution. Then, the obvious follow-up question to ask is whether adding gradient noise to the regular estimator makes it more effective than the alternate estimator with an optimistically initialized baseline. Figure \ref{fig: optimistic_noisy.pdf}, which shows the sensitivity plots for the regular and the alternate estimator with both an optimistic baseline and added gradient noise, answers this question in a negative. Comparing the results from Figure \ref{fig: optimistic_noisy.pdf} with those from Figure \ref{fig: competitor_regular_noisy} shows that the alternate estimator benefits from an optimistic baseline much more than the regular estimator benefits from added gradient noise. More importantly, even though adding gradient noise helped with this particular bandit instance, as we show in \S \ref{sec: grad_noise_mdp}, adding gradient noise doesn't help the regular estimator overcome saturated policies in MDPs.
      
      \begin{figure}[h]
        \centering
        \includegraphics[scale=0.6]{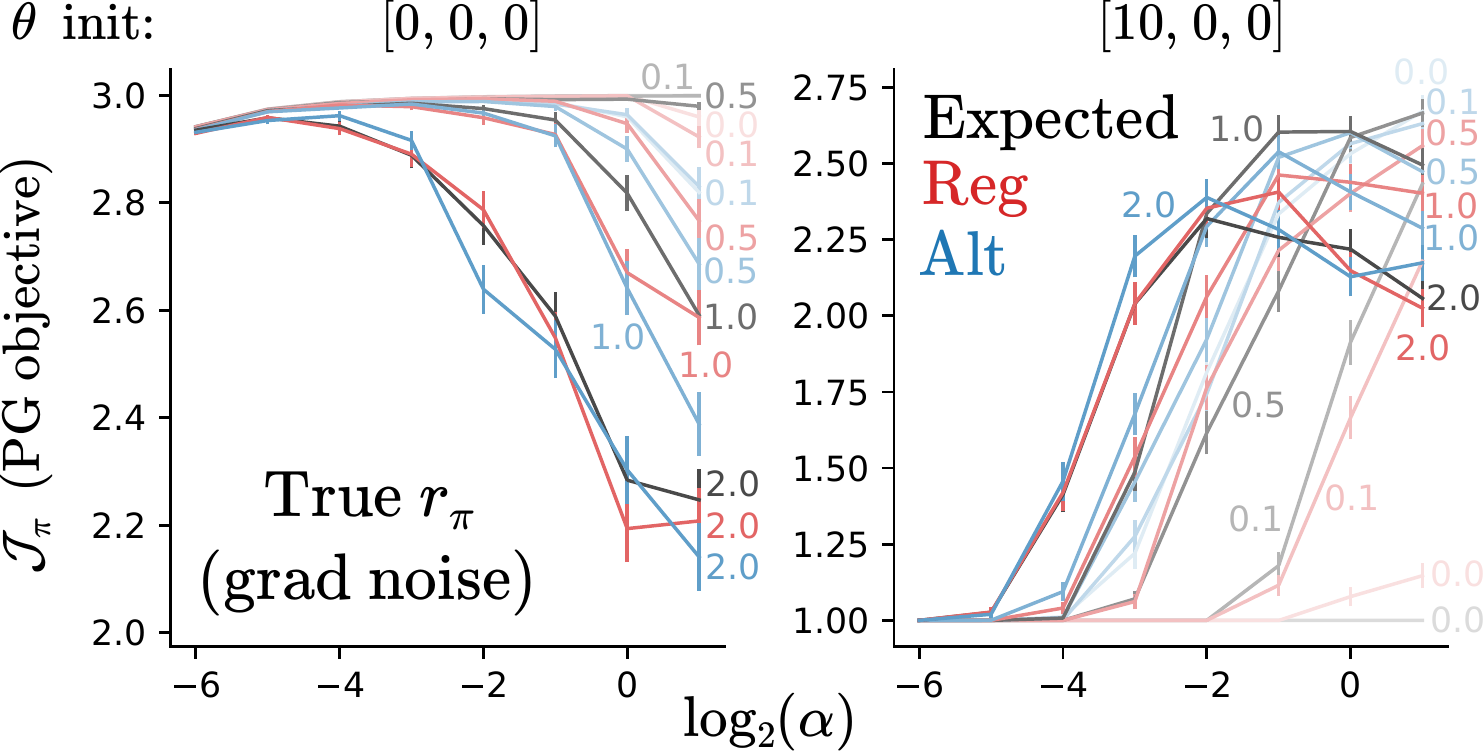}
        \caption{Sensitivity plots for the gradient bandit estimators with true $r_\pi$ and an additional normally distributed gradient noise. The individual graph labels $(0, 0.1, 0.5, 1, 2$) show the standard deviation of the normal noise added to the gradient. The subplot headers show the policy saturation at the beginning of the experiment.}
        \label{fig: competitor_true_rpi_regular_noisy}
      \end{figure}
      
      \begin{figure}[h]
        \centering
        \includegraphics[scale=0.4]{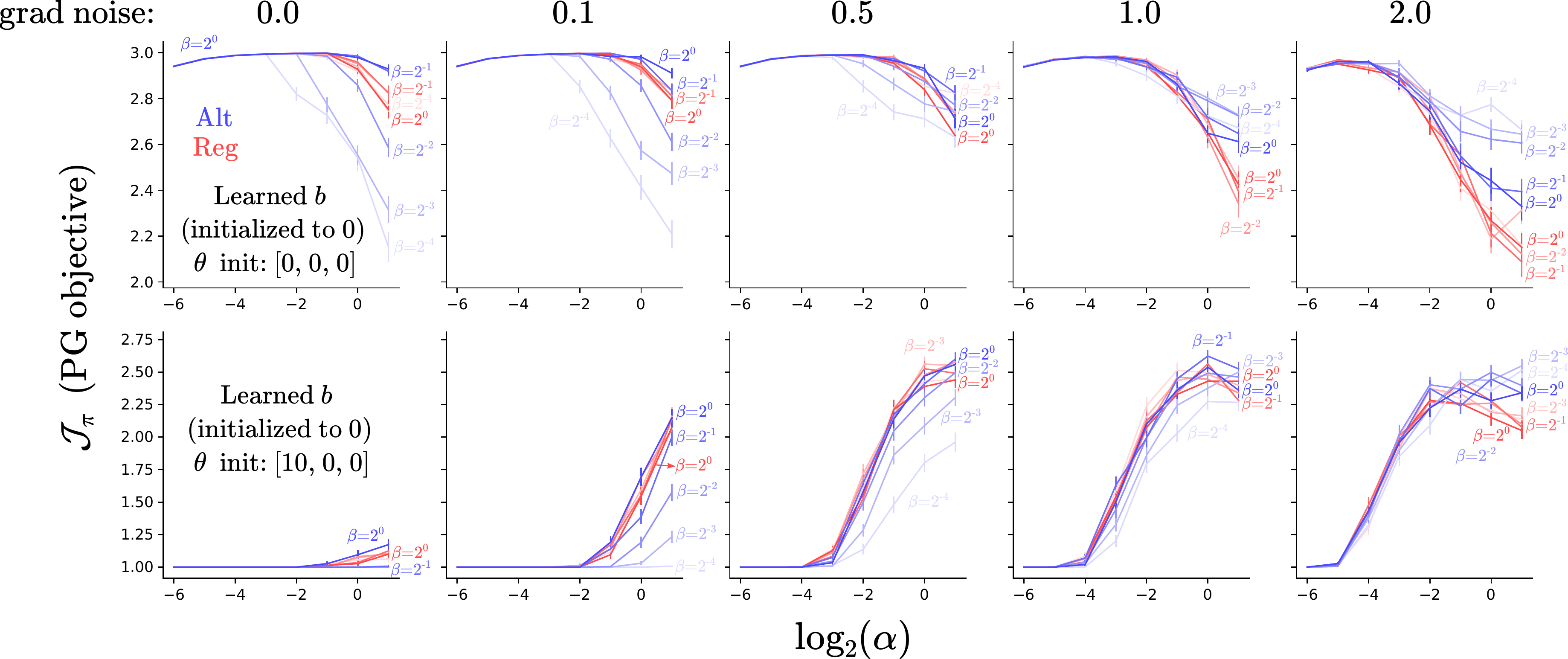}
        \caption{Sensitivity plots for the regular and the alternate estimators with added gradient noise and a learned baseline initialized to zero. The figure header shows the standard deviation of the Gaussian gradient noise. Top and bottom rows correspond to uniform and saturated policy initializations respectively. Interestingly, for the uniform policy case and a high standard deviation of the noise, the alternate estimator is more robust to the gradient noise than the regular estimator.}
        \label{fig: competitor_regular_noisy}
      \end{figure}
      
      \begin{figure}[h]
        \centering
        \includegraphics[scale=0.4]{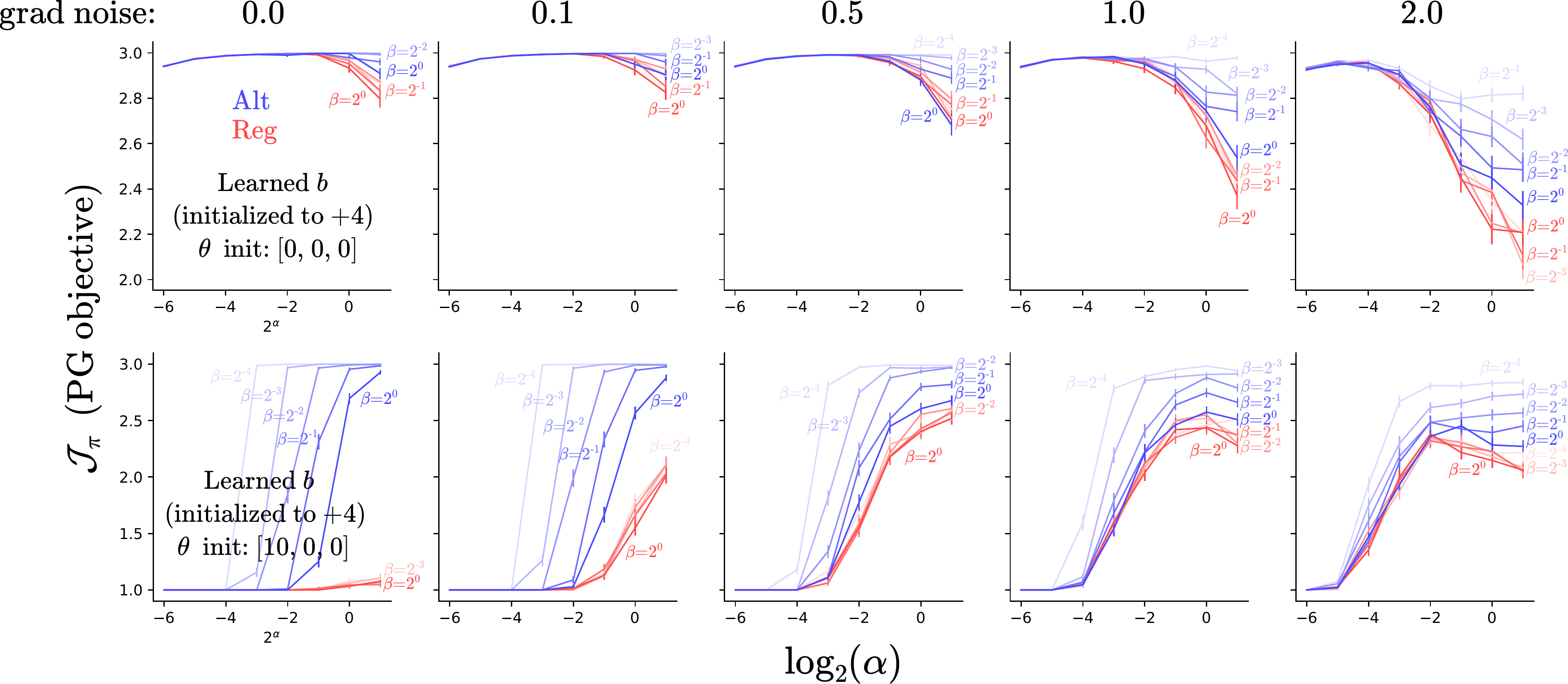}
        \caption{Sensitivity plots for the estimators with gradient noise and an optimistic baseline.} 
        \label{fig: optimistic_noisy.pdf}
      \end{figure}

      \section{REINFORCE with Tabular Representation: Full Experiments} \label{app: experiments_tabular}
      In this section, we demonstrate that the alternate estimator for MDPs enjoys similar benefits as the alternate estimator for bandits. We begin by giving the experimental details, and present additional experiments that complement those presented in \S \ref{sec: tabular_experiments}

      \subsection{Experimental Details}
      We use the chain environment\footnote{Implementation available here: \url{https://github.com/svmgrg/rl_environments/tree/main/LinearChain}.}, shown in Figure \ref{fig: chain}, which is an episodic MDP where the expected rewards are zero everywhere expect at the rightmost transition. We train five different agents using REINFORCE (exact pseudocode given in Algorithm \ref{alg: reinforce}). All the agents maintain a tabular policy and, in some cases additionally, a tabular value function estimate. The policy is learned either using the expected PG, or using the REINFORCE algorithm with either the regular estimator (true $v_\pi$ or a learned baseline) or the alternate estimator (again with true $v_\pi$ or a learned baseline). The critic is estimated using Monte-Carlo sampling.

      \begin{figure}[!hbp]
        \centering
        \includegraphics[scale=0.8]{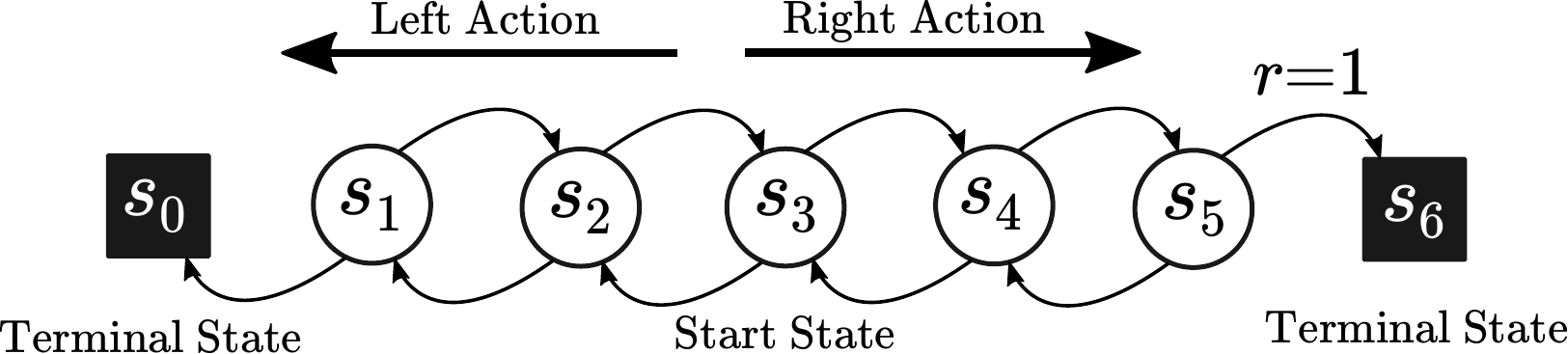}
        \caption{Chain Environment. At the beginning of each episode, the agent is reset to the state $s_3$. At each timestep, it can take the action \texttt{left} or \texttt{right}, which deterministically moves it into the corresponding next state. The episode ends when the agent reaches either of the two terminal states. The expected rewards are zero at all transitions except the one into the right most terminal state. Further, a reward noise $\epsilon \sim \mathcal{N}(0, 1)$ is added to the expected reward at each timestep.}
        \label{fig: chain}
      \end{figure}
      
      \begin{algorithm}[!hbt]
        \caption{REINFORCE with Baseline}
        \label{alg: reinforce}
        \begin{algorithmic}
          \State Input a policy $\pi_{\wvec}(a|s)$ and a value function estimate $\hat{v}_{\omegavec}(s)$ 
          \State Input the policy stepsize $\alpha$ and the value function stepsize $\beta$
          \State Initialize $\wvec$ and ${\omegavec}$
          \For{total number of episodes}
          \State Sample an episode $(S_0, A_0, R_1, \ldots, S_{T-1}, A_{T-1}, R_T, S_T) \sim \mathbb{P}_\pi$
          \State \texttt{\textcolor{purple}{\# Calculate the PG estimate}}
          \State $\hat{\gvec}^\text{{REINF}} = \sum_{t=0}^{T-1} \gamma^t (G_t - \hat{v}_{\omegavec}(S_t)) \nabla_{\wvec} \log \pi_{\wvec}(A_t|S_t)$
          \State \texttt{\textcolor{purple}{\# Update the policy}}
          \State $\wvec \leftarrow \wvec + \alpha \cdot \hat{\gvec}^\text{{REINF}}$
          \For{each step of the episode $t = 0, 1, \ldots, T-1$}
          \State Calculate the return $G_t = R_{t+1} + \gamma R_{t+2} + \cdots + \gamma^{T-t-1} R_T$
          \State \texttt{\textcolor{purple}{\# Update the value function estimate}}
          \State ${\omegavec} \leftarrow {\omegavec} + \beta \cdot \Big( G_t - \hat{v}_{\omegavec}(S_t) \Big) \nabla_{\omegavec} \hat{v}_{\omegavec}(S_t)$
          \EndFor
          \EndFor
        \end{algorithmic}
      \end{algorithm}

      All the methods were run for 100 episodes. The learning curves show the mean performance and standard error, and the sensitivity plots show the mean performance over the last 10 episodes. The results were averaged over 150 independent runs. For getting a uniform policy initialization, we set $\theta_{\texttt{left}} = \theta_{\texttt{right}} = 0$ for all the five states. A saturated policy was obtained by initializing $\theta_{\texttt{right}} = 0$ and $\theta_{\texttt{left}} \in \{1, 2, 3\}$ depending on the degree of saturation, again for all the five states. The agents either used the true value function $v_\pi$ (computed analytically) or learned a value function estimate $b$ which was initialized to either zero, $+4$, or $-4$ for each state of the MDP. We swept over the policy stepsize $\alpha \in \{2^{-6}, 2^{-5}, \ldots, 2^1\}$ and the value function baseline stepsize $\beta \in \{2^{-4}, 2^{-3}, \ldots, 2^1\}$. We set the discount factor $\gamma = 0.9$. To reduce the runtime of the experiments, if the agent was unable to solve the task after 100 timesteps, we terminated the episode\footnote{It might seem unsatisfying to artificially end an episode after 100 timesteps. However, not doing so resulted in some of the experiments running for about a million steps, which dramatically increased the runtime of the experiments. To put this number in perspective, an optimal policy can solve the chain environment in three steps. These million timestep long episodes occurred when the policy learned to go right with a high probability in one state and learned to go left in the very next state, effectively creating an unending loop between these two states.

        We justify the 100 timestep cutoff by noting that the contribution of states after a hundred steps becomes negligibly small. This can be seen by focusing on the term $\gamma^t$ in the expression for the REINFORCE estimator $\hat{g}^{\text{REINF}}$ (Algorithm \ref{alg: reinforce}): for our chosen discount factor, $\gamma^{100} = 0.9^{100} = 2.7 \times 10^{-5}$. We also note that in our experiments, we learned the value function using simple Monte-Carlo estimates and, contrary to the recommendation made by Pardo et al. (2018), did not bootstrap at the state where the episode was artificially timed out. We again justify this by the fact that $\gamma^{100} v_\pi(s_{\texttt{timeout}})$ is negligibly small compared to the returns observed in our experiments (see the scale of the $y$-axis in Figure \ref{fig: multiple_bad_init_learning_curve}) and is thus safe to ignore. Yet another argument in favor of keeping a timeout in our experiments is that, unlike timeouts with continuing MDPs where they are guaranteed to occur, in the chain environment the occurrence of the episodes being timed out decreases as the agent's policy improves.

        Also note that we are using a discount factor $\gamma \neq 1$ for an episodic problem despite the fact that episodic problems do not necessarily require discounting. This choice was primarily made to justify timing out the episode after 100 steps. However, even if we ignore the timeout issue, keeping a discount factor around for episodic MDPs can be useful. It caps the maximum return at the fixed value of $\frac{r_{\text{max}}}{1 - \gamma}$, and helps us avoid situations such as the infinite loop between two states in the supposedly episodic chain MDP. \label{footnote: timeout_rambling}}.
      
      \subsection{Performance of PG Estimators against Increasingly Saturated Policy Initializations} \label{sec: reinforce_chain_different_policy_saturations_exp}       
      
      We begin by showcasing that the alternate PG estimator is competitive with the regular estimator in case of uniform policy initialization and superior to it in case of sub-optimally saturated policy initialization. These results are essentially identical to those from the bandit experiments (\S \ref{sec: bandits_alternate_estimator_noise_and_no_noise_first_exp}). Figure \ref{fig: multiple_bad_init_learning_curve} shows the learning curves for the best performing parameter configuration (chosen from the sensitivity plots), and Figure \ref{fig: multiple_bad_init_sensitivity} shows the stepsize sensitivity corresponding to the final performance for each parameter setting. Observe that irrespective of the policy saturation, the alternate estimator with baseline is vastly superior to all the methods (except expected PG), and that the alternate estimator with true $v_\pi$ is a little better than the regular estimators. We attribute the superior performance of alternate estimator with baseline to the bias of the estimator combined with utilizing the noise in the returns, which probably allows it to have better exploration. From the sensitivity plots, we further see that as the initial policy is saturated more towards the left direction (in the chain environment), performance of all the methods worsens.

      \begin{figure}[t]
        \centering
        \includegraphics[scale=0.32]{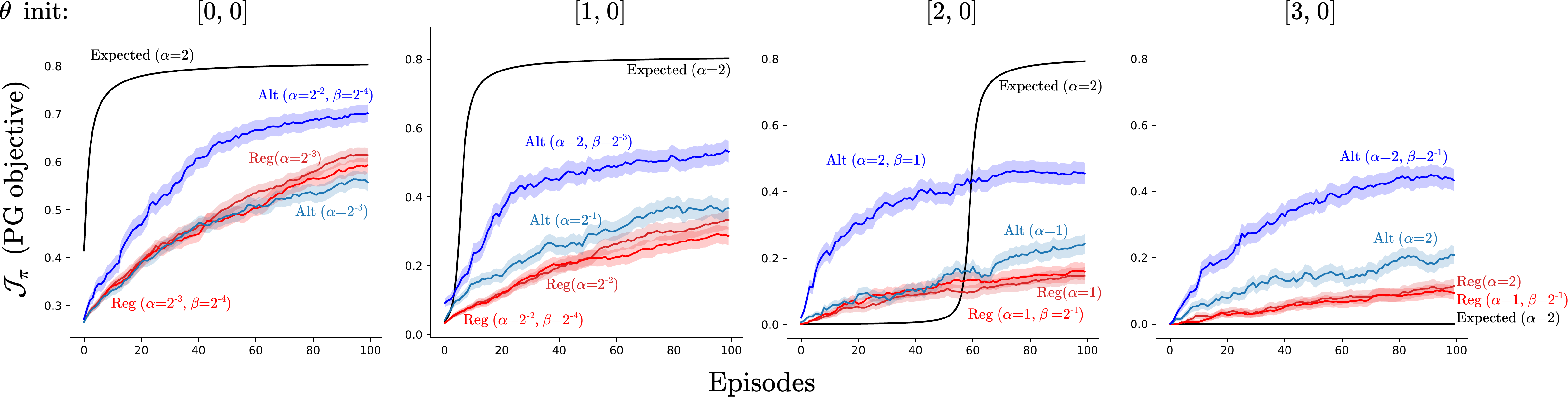}
        \caption{Learning curves for the REINFORCE agent with different estimators against increasingly saturated policy initializations on the five state chain environment. The figure header shows the action preference initialization for $\theta_{\texttt{left}}$ and $\theta_{\texttt{right}}$ respectively. The estimators which don't mention the  baseline stepsize $\beta$ used the true value function $v_\pi$, whereas the others learned a baseline using Monte-Carlo sampling. The baseline was initialized to zero for all the states.} 
        \label{fig: multiple_bad_init_learning_curve}
      \end{figure}

      \begin{figure}[t]
        \centering
        \includegraphics[scale=0.32]{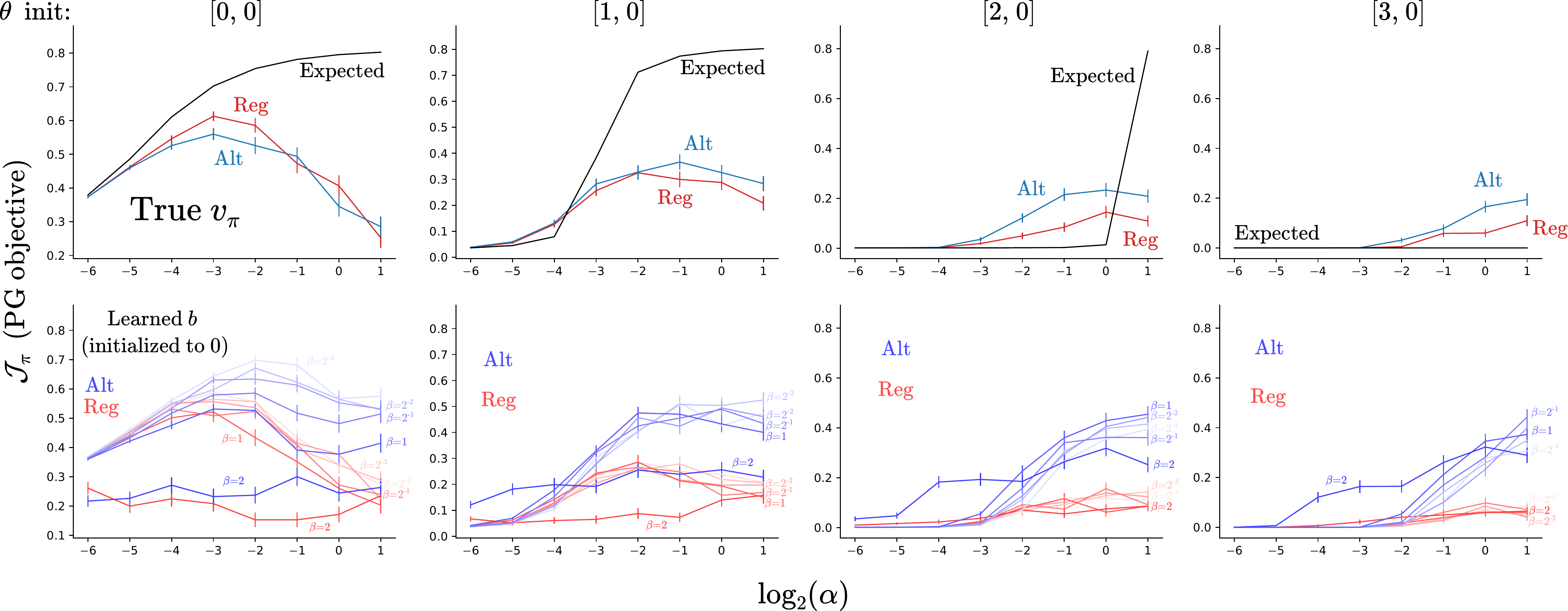}
        \caption{Sensitivity plots for different PG Estimators, showing the mean final performance during the last 10 episodes, against increasingly saturated policy initializations on the chain environment. The different columns correspond to the different degrees of policy saturations and the figure headers show the action preference initializations for the \texttt{left} and \texttt{right} actions respectively.} 
        \label{fig: multiple_bad_init_sensitivity}
      \end{figure}
      
      \subsubsection*{PG Estimators without Reward Noise}
      We now study the performance of different estimators when we remove the reward noise from the chain MDP. Figure \ref{fig: multiple_bad_init_no_noise_learning_curve} shows the learning curves and Figure \ref{fig: multiple_bad_init_no_noise_sensitivity} shows the parameter sensitivity plots for different estimators using REINFORCE on the chain MDP with reward noise $\epsilon = 0$. The results are somewhat similar to the bandit setting: in absence of reward noise, the performance of the alternate estimator drops with saturated policy initialization. However, there is one important difference from the bandit setting. For MDPs, the alternate estimator finally cares about the noise in the return estimate which depends on both the reward noise and the sampling noise. Therefore, even without reward noise, in MDPs the alternate estimator can utilize the noise in estimating the returns from Monte-Carlo sampling. However, for the chain MDP, the transitions were deterministic and therefore the sampling noise was also quite less.
      \begin{figure}[t]
        \centering
        \includegraphics[scale=0.32]{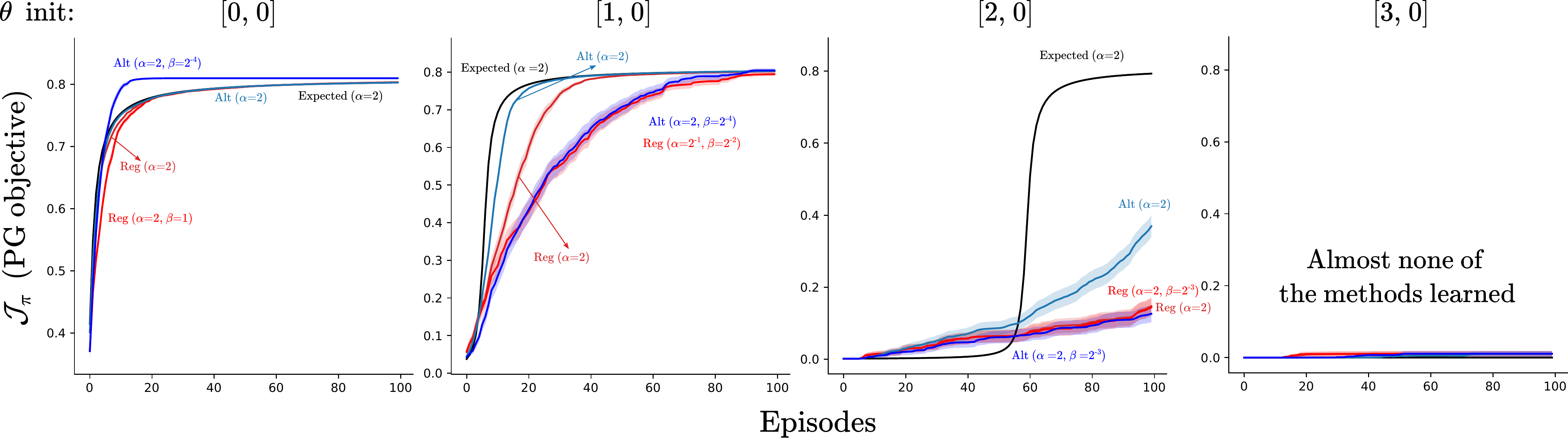}
        \caption{Learning curves for PG estimators without reward noise. The setting is same as Figure \ref{fig: multiple_bad_init_learning_curve}, except that the reward noise is set to zero. We see that, for saturated policies, the alternate estimator performs worse as compared to its performance in Figure \ref{fig: multiple_bad_init_learning_curve} where it had reward noise, and \textbf{curiously} alternate with $v_\pi$ performs better than alternate with a learned baseline.} 
        \label{fig: multiple_bad_init_no_noise_learning_curve}
      \end{figure}

      \begin{figure}[!hbp]
        \centering
        \includegraphics[scale=0.32]{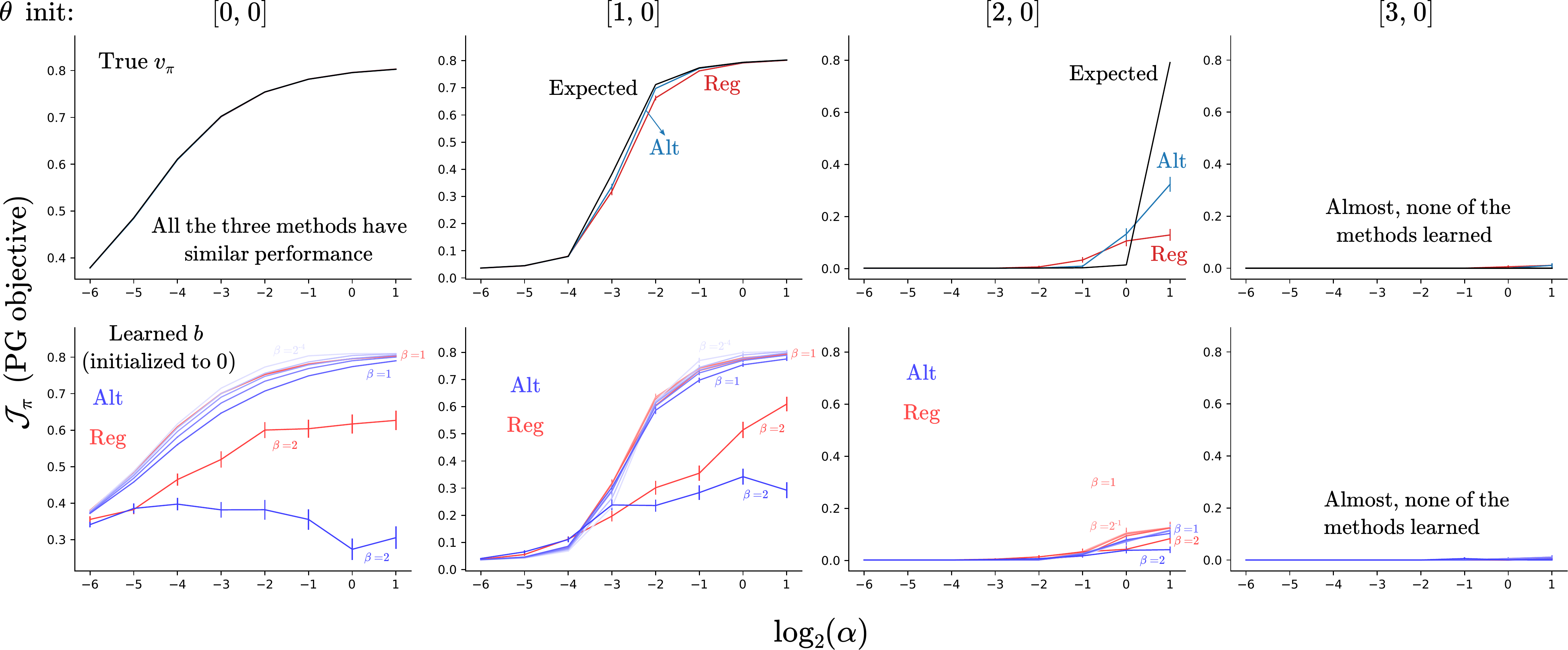}
        \caption{Sensitivity plots for PG estimators without reward noise. Setting is same as in Figure \ref{fig: multiple_bad_init_sensitivity} except that there is no reward noise here.} 
        \label{fig: multiple_bad_init_no_noise_sensitivity}
      \end{figure}

      \subsection{Optimistic and Pessimistic Value Function Estimates} \label{sec: opti_pessi_pg_experiments}  
      In this experiment, we study the performance of the alternate estimator when the value function baseline is initialized optimistically or pessimistically and then learned using Monte-Carlo sampling. Figures \ref{fig: optimistic_linear_chain_sensitivity} and \ref{fig: pessimistic_linear_chain_sensitivity} show the parameter sensitivity plots for agents with optimistically ($b = +4$) or pessimistically ($b = -4$) initialized baselines respectively, against different degrees of policy saturations. The results, which are yet again similar to the bandit setting, show that having an optimistic baseline significantly helped the alternate estimator; in particular, compare the performance of the alternate estimator on saturated policies with ($b = +4$, as shown in Figure \ref{fig: optimistic_linear_chain_sensitivity}) and without ($b = 0$, as shown in Figure \ref{fig: multiple_bad_init_sensitivity}) optimism. Whereas, a pessimistic baseline hurt its performance. Moreover, the alternate estimator with the optimistic baseline prefers smaller critic stepsizes (allowing it to enjoy the optimism for longer), and vice-versa for the pessimistic baseline. 

      \begin{figure}[!tbp]
        \centering
        \includegraphics[scale=0.32]{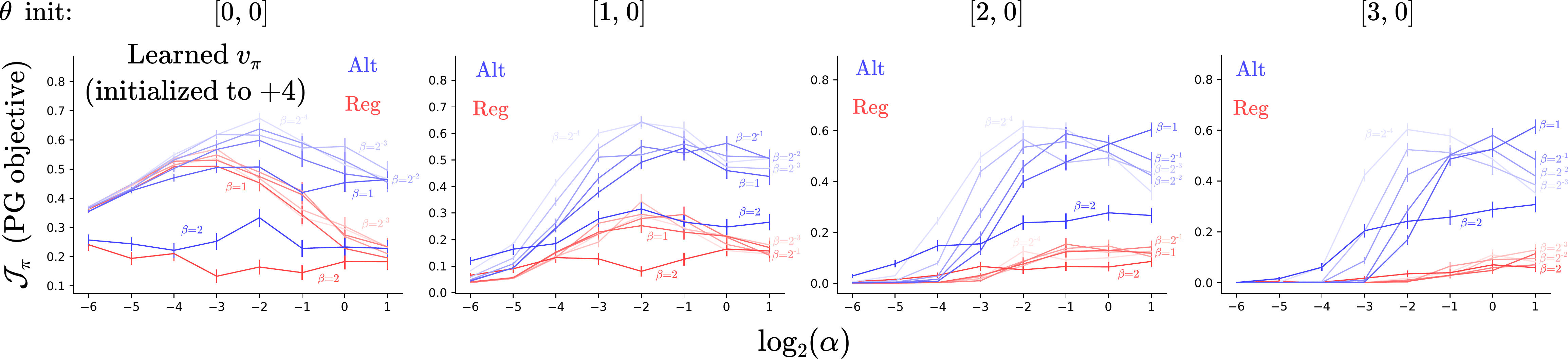}
        \caption{Sensitivity plots for an optimistically initialized value function estimate. The setting is similar to that in Figure \ref{fig: multiple_bad_init_sensitivity}, except that the baseline is initialized to $+4$.} 
        \label{fig: optimistic_linear_chain_sensitivity}
      \end{figure}

      \begin{figure}[!tbp]
        \centering
        \includegraphics[scale=0.32]{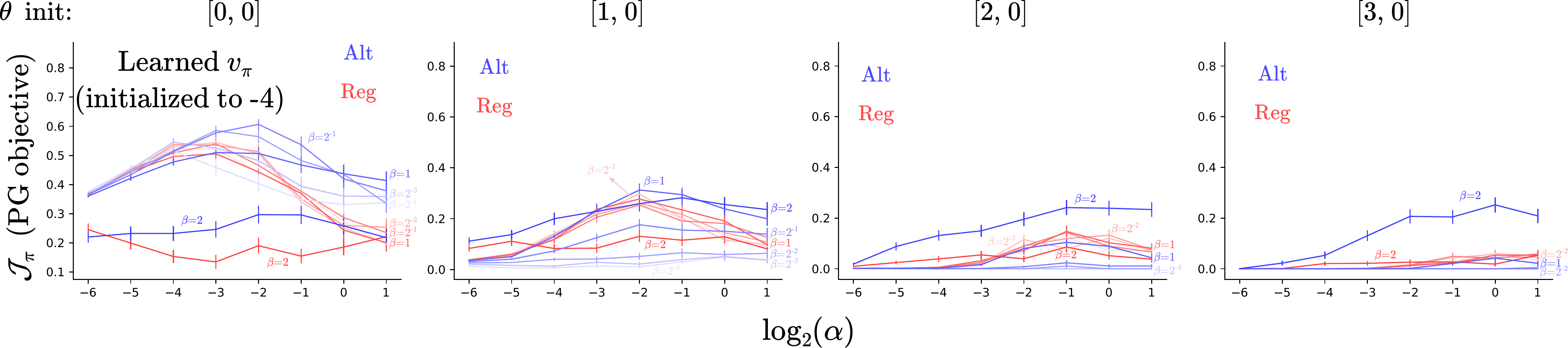}
        \caption{Sensitivity plots for a pessimistically initialized value function estimate. The setting is again similar to that in Figure \ref{fig: multiple_bad_init_sensitivity} except that the baseline is initialized to $-4$.} 
        \label{fig: pessimistic_linear_chain_sensitivity}
      \end{figure}
      
      \subsubsection*{Fixed Value Function Estimate}
      In this experiment, we fixed the value function estimate for both the estimators; which made the alternate, but not the regular, estimator biased. Figure \ref{fig: biased_pg_sensitivity} shows the sensitivity plots for the two estimators. This experiment is the MDP analog of the biased gradient bandit experiment (see Figure \ref{fig: bias_alternate_param_study}) and suggests that Theorem \ref{thm: fixed_pt_biased_pg_attract_repulse} might also hold for MDPs: since, we can show that the policy (for the optimistic case) converges to the fixed point $\pi^*(a) = \frac{1}{r(a) - b} \left( \sum_c \frac{1}{r(c) - b} \right)^{-1}$ given by Lemma \ref{thm: fixed_pt_biased_pg}. For MDPs, we need to compute this fixed point separately for each state: $\pi(\texttt{left} | s) = \pi(\texttt{right} | s) = 0.5$ for $s \in \{s_1, s_2, s_3, s_4\}$, and $\pi(\texttt{left} | s_5) = 0.43$ and $\pi(\texttt{right} | s_5) = 0.57$. Then, we analytically calculate $\mathcal{J}_{\pi^*} = v_{\pi^*}(s_3) = 0.28$ which approximately matches the value shown in Figure \ref{fig: biased_pg_sensitivity} (left) where the policy for the optimistic fixed baseline actually converged to; therefore, our prediction matches with the actual results. For a pessimistic baseline, the policy converged to the nearest corner for each state (in this case action \texttt{left}) and obtained zero return.
      
      \begin{figure}[!tbp]
        \centering
        \includegraphics[scale=0.6]{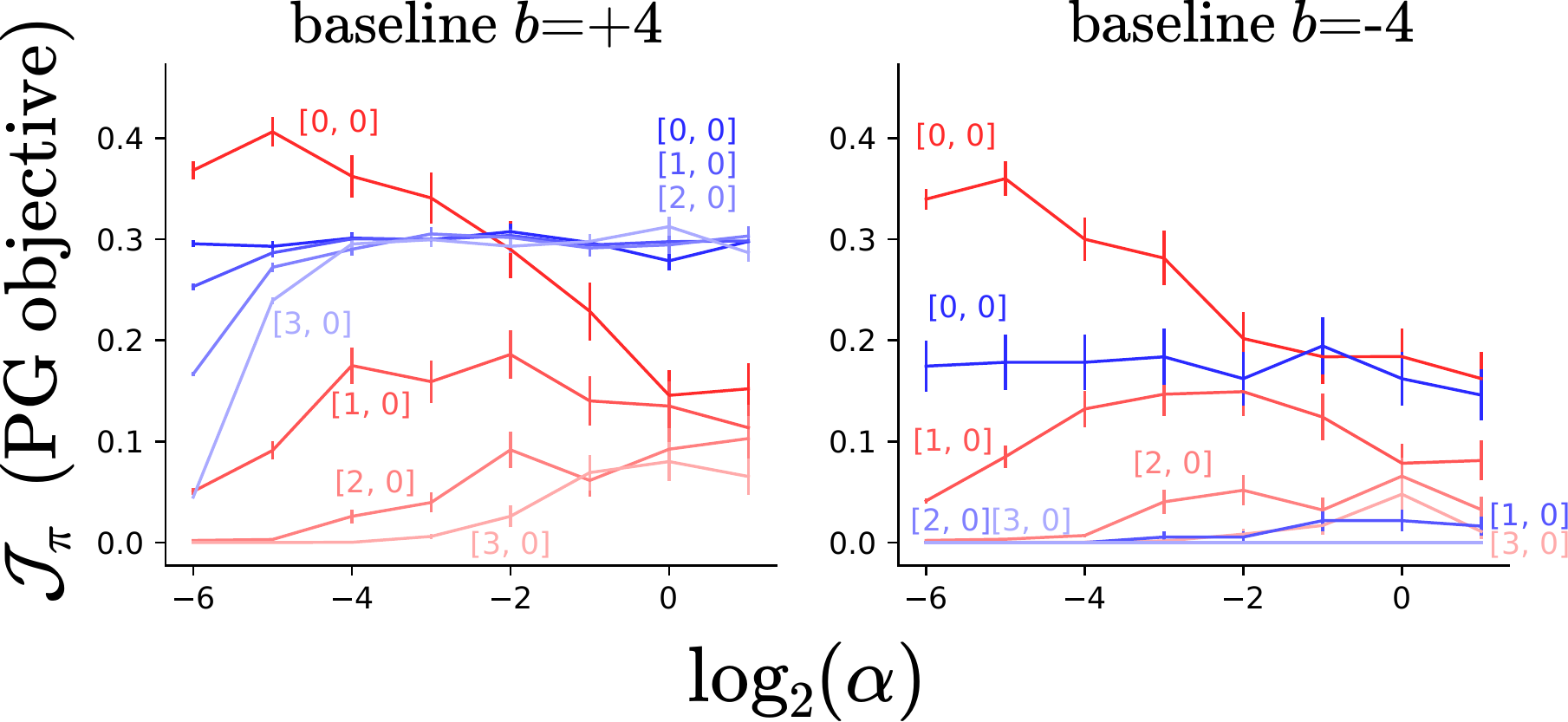}
        \caption{Bias in the alternate PG estimator for MDPs. Sensitivity plots for the REINFORCE agents using regular or alternate estimator with fixed baselines on the chain environment. The value function estimates were fixed to $+4$ (for left subplot) and $-4$ (for right subplot) and were not subsequently updated. The labels for the individual curves show the action preference initializations $[\theta_{\texttt{left}}, \theta_{\texttt{right}}]$ for all the five states.}
        \label{fig: biased_pg_sensitivity}
      \end{figure}

      \subsection{Effect of adding Gradient Noise to the Estimators} \label{sec: grad_noise_mdp}
      
      This experiment studies the efficacy of randomly perturbing the PG estimates, with a normally distributed noise while updating the action preferences, in helping the agent escape saturated policy regions. The preferences, for a given state $s$ and action $a$, were updated using $\theta_a^{\text{new}}(s) = \theta_a^{\text{old}}(s) + \alpha \big[ g_a(s) + \xi \big]$ with $\xi \sim \mathcal{N}(0, \texttt{grad-noise})$. We swept over the noise magnitude $\texttt{grad-noise} \in \{0, 0.1, 0.5, 1, 2\}$. Both the estimators learned a value function estimate which was initialized to zero\footnote{We do not present the results for expected PG and the regular and alternate REINFORCE estimators using the true $v_\pi$, since they were qualitatively similar to the results for the learned baseline case presented here.}. Figure \ref{fig: competitor_linear_chain_sensitivity} shows the sensitivity of the final performance for the two estimators. 

      These results show that adding gradient noise doesn't significantly help the regular estimator in dealing with saturated policies for MDPs: the performance of the regular estimator with gradient noise was only slightly better than its performance without gradient noise (cf. Figure \ref{fig: multiple_bad_init_sensitivity}). More importantly, the alternate estimator had a superior final performance compared to the regular estimator for both the uniform and the saturated policy initialization cases. This observation is in contrast to what we saw in the bandit case (\S \ref{sec: gradient_noise_bandits}), where the gradient noise seemed to help more. We explain this result by the fact that random noise merely increases the variance of the gradient estimate without providing any effective signal. And while this added variance helped the regular estimator in the simpler bandit problem, it is apparently not sufficient for the MDPs. The alternate estimator on the other hand relies on both the noise from the sampled return and also a much stronger signal from the bias in the critic estimate to escape the saturated policy regions.
      
      \begin{figure}[!tbp]
        \centering
        \includegraphics[scale=0.32]{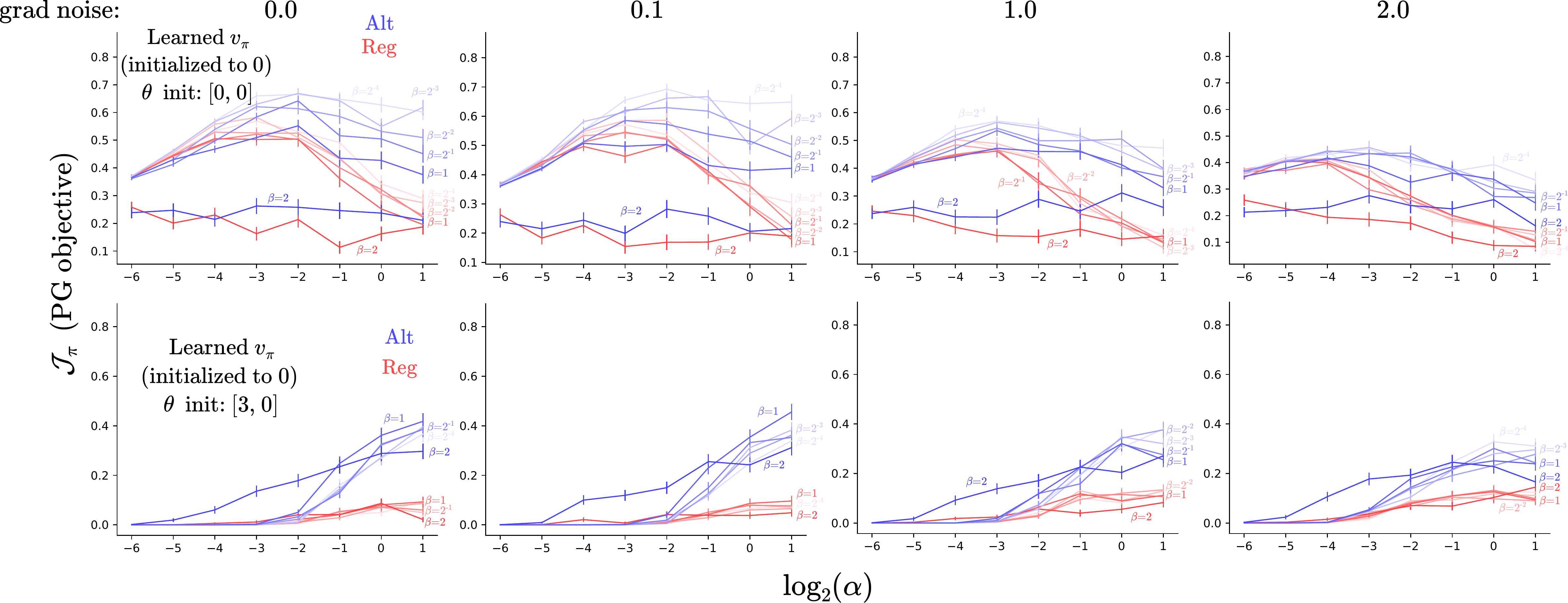}
        \caption{Adding gradient noise to the REINFORCE estimators. The top row shows the sensitivity plots for a uniform policy and the bottom row shows the sensitivity plots for a policy saturated as $\theta_{\text{left}} = 3$ and $\theta_{\text{right}} = 0$. Different columns correspond to different gradient noise magnitudes.} 
        \label{fig: competitor_linear_chain_sensitivity}
      \end{figure}       

      \subsection{A Uniform Policy Initialization can be Sub-optimal too!} \label{app: uniform_policy_bad}
      In all the previous experiments, we found that an optimistically initialized baseline helped the performance of the alternate estimator (by making the policy more uniform). In this section, we consider an MDP, where a uniform policy is bad for exploration, and therefore an optimistic baseline by making the agent's policy more uniform hurts its performance. Figure \ref{fig: chain_messed} shows this MDP, which is a modified form\footnote{This MDP was inspired from the ``Vanishing gradient example'' given in  Chapter 12 of Agarwal et al. (2019).} of the five state chain environment: instead of having just two actions, this modified chain has four actions for each state, three of which take the agent towards the left direction and the fourth towards the right direction.
      
      \begin{figure}[!hbp]
        \centering
        \includegraphics[scale=0.8]{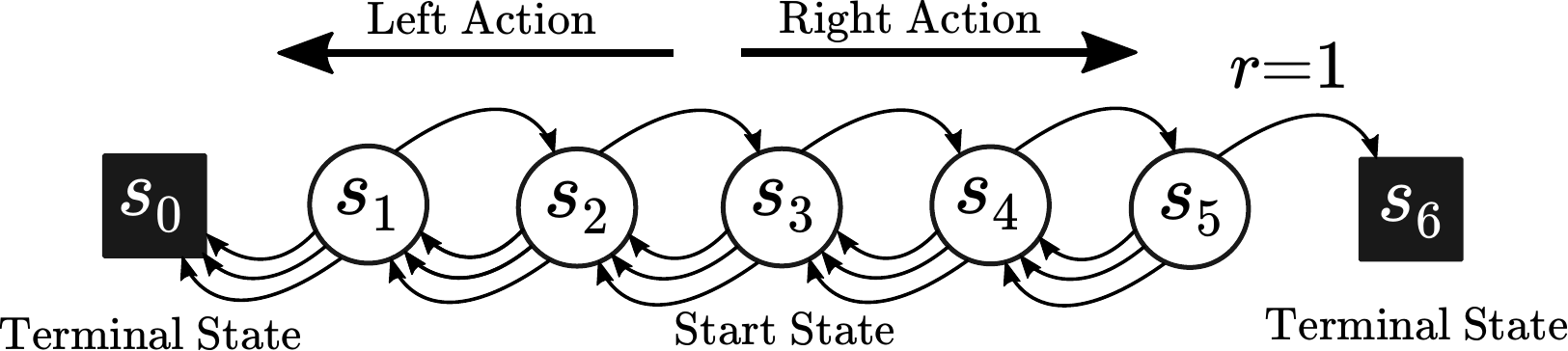}
        \caption{Linear chain environment where a uniform policy will explore poorly.} 
        \label{fig: chain_messed}
      \end{figure}

      Figure \ref{fig: chain_messed_sensitivity} shows the sensitivity plots for the REINFORCE agents with the regular and alternate estimators for four different settings: using the true $v_\pi$, and a learned baseline initialized to either of $0$, $+4$, or $-4$. The policy for all these settings was initialized uniformly, i.e. $\theta = 0$ for all states and action. The results show that in each case, the alternate estimator had a poorer performance as compared to the regular estimator. In particular, even an optimistically initialized baseline could not improve the alternate estimator's performance. We reconcile this fact by noting that in this problem the uniform random policy will actually lead to a sub-optimal exploration and an optimistic baseline will hamper the performance of the alternate estimator.
      
      \begin{figure}[!tbp]
        \centering
        \includegraphics[scale=0.32]{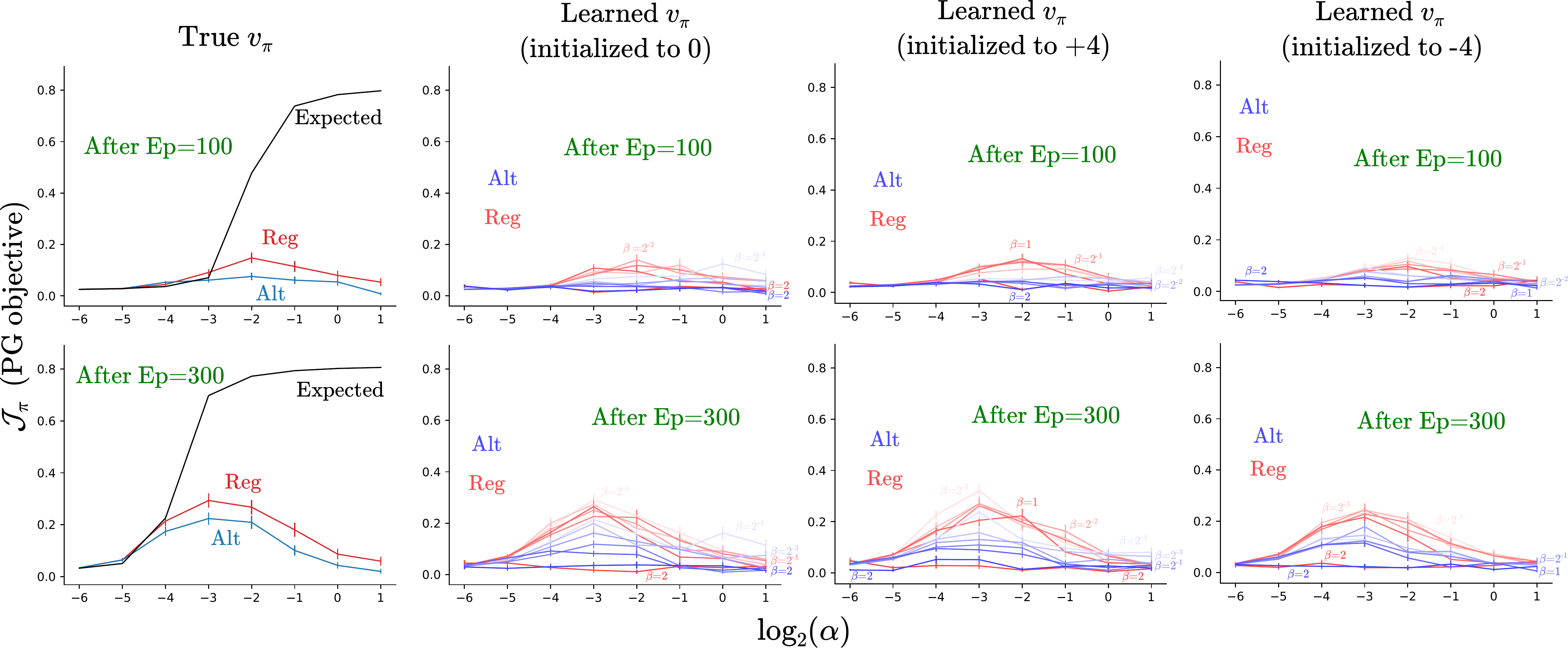}
        \caption{Sensitivity plots for different methods, initialized with a uniform policy, on the difficult chain environment (described in Figure \ref{fig: chain_messed}). Top row shows the mean final performance  (averaged over the last 10 episodes) of the estimators after 100 episodes and the bottom row shows the mean final performance (again averaged over the last 10 episodes) after 300 episodes. Different columns refer to either a true baseline or a learned baseline with different initializations. We see that even after 300 episodes, the algorithms are not able to solve the task (i.e. they are unable to achieve the maximum return of $0.81$). Further, even with an optimistic baseline, the alternate estimator performs quite poorly as compared to the regular estimator.} 
        \label{fig: chain_messed_sensitivity}
      \end{figure}

      \section{Online Actor-Critic with Linear Function Approximation} \label{app: experiments_linear}

      In this section, we demonstrate that the desirable properties of the alternate estimator continue to hold with function approximation. We ran experiments with online actor-critic (see Algorithm \ref{alg: online_ac}) on various MDPs using linear function approximation (with tile-coding) and neural networks. The agent was trained using either the regular or the alternate estimator and a critic learned using TD-style updates; in these experiments, we do not have access to the underlying dynamics, and therefore we did not conduct experiments with the true value function $v_\pi$. We considered two different types of experiments: one with artificially introduced policy saturation, and another with non-stationary MDPs which naturally give rise to sub-optimally saturated policies. Experiments where we explicitly saturate the agent's policy at the beginning of learning are helpful in understanding the behavior of the estimators by providing a controlled environment. The non-stationary MDPs serve as a proof-of-concept that sub-optimally saturated policies are not merely pathological constructs, but can occur in practical problems that a user might care about.

      \begin{algorithm}[!hbt]
        \caption{Online Actor-Critic}
        \label{alg: online_ac}
        \begin{algorithmic}
          \State Input a policy $\pi_{\wvec}(a|s)$ and a value function estimate $\hat{v}_{\omegavec}(s)$ 
          \State Input the policy stepsize $\alpha$ and the value function stepsize $\beta$
          \State Initialize $\wvec$ and ${\omegavec}$
          \For{each episode (in total number of episodes)}
          \State Sample $S \sim \mu$
          \State Initialize the geometric discounting term $I_\gamma \leftarrow 1$
          \While{the state $S$ is not terminal, i.e. $S \neq s_\times$}
          \State Sample $A \sim \pi_{\wvec}(\cdot|S)$ and $S', R \sim p(\cdot, \cdot | S, A)$
          \State $\delta \leftarrow R + \gamma \hat{v}_{\omegavec}(S') - \hat{v}_{\omegavec}(S)$ \hfill (if $S'$ is terminal, then $\hat{v}_{\omegavec}(S') = \hat{v}_{\omegavec}(s_\times) := 0$)
          \State \texttt{\textcolor{purple}{\# Calculate the gradient estimator}}
          \State $\hat{\gvec}^{\textrm{OAC}} = I_\gamma \delta \nabla_{\wvec} \log \pi_{\wvec}(A|S)$
          \State \texttt{\textcolor{purple}{\# Update the policy}}
          \State $\wvec \leftarrow \wvec + \alpha \cdot \hat{\gvec}^{\textrm{OAC}}$
          \State $I_\gamma \leftarrow \gamma I_\gamma$
          \State \texttt{\textcolor{purple}{\# Update the value function estimate}}
          \State ${\omegavec} \leftarrow {\omegavec} + \beta \cdot \delta \nabla_{\omegavec} \hat{v}_{\omegavec}(S)$
          \State $S \leftarrow S'$
          \EndWhile
          \EndFor
        \end{algorithmic}
      \end{algorithm}

      \subsection{Experimental Details for Linear Function Approximation}
      For experiments with linear function approximation, we used the environments\footnote{We used the environment implementations available here: \url{https://github.com/andnp/PyRlEnvs/tree/main/PyRlEnvs/domains}. In this repository, MountainCar is referred to as GymMountainCar.} Acrobot (Sutton, 1996) and MountainCar (\S 4.3, Moore, 1990; Singh and Sutton, 1996), which are depicted in Figure \ref{fig: environments}. Both of these environments are episodic in nature with a continuous state space and a discrete action space consisting of three actions: \texttt{left}, \texttt{do-nothing}, and \texttt{right}. In Acrobot, the goal is to control an under-actuated double pendulum such that its end effector reaches a certain height, by applying torque on the second joint. In MountainCar, the goal is to move an under-powered car from the bottom to the top of a hill. For both the environments, we fixed the discount factor $\gamma = 1$ and timed out the episode after 1000 timesteps if the agent was unable to solve it by then\footnote{For these experiments we bootstrap at the final state of the episode, i.e. if the episode times out, we bootstrap using $\hat{v}_\pi(s_{\text{timeout}})$, and if the episode actually ends, we use $v(s_\times) = 0$ for bootstrapping. This justifies the use of timeouts since (1) this process follows the recommendation made by Pardo et al. (2020), (2) this is an episodic task and as the agent's performance improves, the episodes will be timed out less, and (3) from our experiments, we observed that 1000 timesteps gave the agent enough time to encounter a true terminal state. Encountering true terminal states is important when learning with TD-style updates, since only at terminal state do we bootstrap from a true value (of zero); for all the other states, we bootstrap from an estimated value. Note that this argument is different from the numerical precision argument given in Footnote \ref{footnote: timeout_rambling} of \S \ref{app: experiments_tabular}. Also note that with $\gamma = 1$ and a timeout of 1000, for the agents that are unable to solve the task within 1000 timesteps, the value function estimates will diverge to negative infinity; for instance, see the value of $\hat{v}_\pi(s_0)$ against time for the regular estimator in Figure \ref{fig: linear_realmc_learning_curve}.}.
      
      \begin{figure}[!tbp]
        \centering
        \includegraphics[scale=0.4]{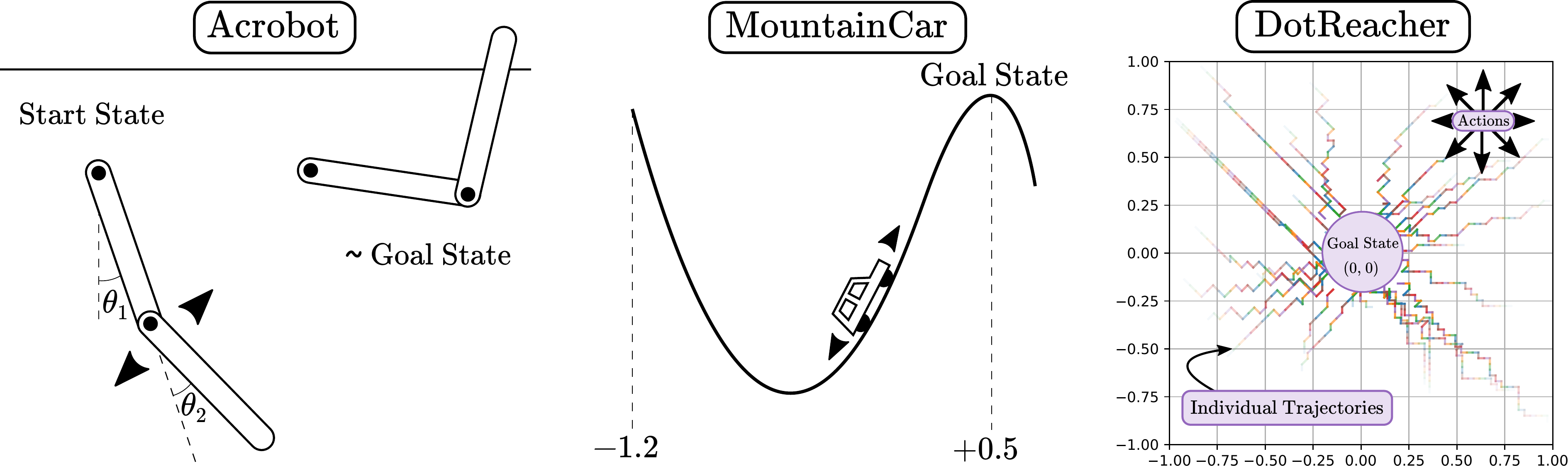}
        \caption{Schematic diagrams of Acrobot, MountainCar, and the DotReacher environments. For the DotReacher environment, we also plotted multiple trajectories for a particular instance of online actor-critic that learned to solve this task.} 
        \label{fig: environments}
      \end{figure}
      
      We performed sweeps over the policy stepsize $\alpha \in \{2^{-13}, 2^{-11}, \ldots, 2^5\}$ and the critic stepsize $\beta \in \{0.01, 0.05, 0.1, 0.5, 1, 2\}$. The state for Acrobot is $[\cos(\theta_1)\;\sin(\theta_1)\;\cos(\theta_2)\;\sin(\theta_1)\;\dot{\theta}_1\;\dot{\theta}_2]$, with $\dot{\theta}_1 \in [-4 \pi, +4 \pi]$ and $\dot{\theta}_2 \in [-9 \pi, +9 \pi]$. The state for MountainCar is $[\texttt{position}, \texttt{velocity}]$ (of the car), with $\texttt{position} \in [-1.2, +0.5]$ and $\texttt{velocity} \in [-0.07, +0.07]$. Both the policy (for details, see \S \ref{app: additional_details_online_ac}) and the value function used by the agent were linear in the state features. And the state features were constructed using tile-coding (\S 9.5.4, Sutton and Barto, 2018). In particular, we used four tiles across each dimension of the state space, and eight different tilings with randomly chosen offsets. We also added an ``always on'' bias feature and normalized the tile-coding features (by dividing them with nine) to have a unit sum.

      \subsection{Non-stationary MountainCar and Acrobot} \label{sec: nonstationary_tasks_alt_wins}
      This experiment investigates the performance of the online AC algorithm, with the regular and the alternate estimators, on non-stationary Acrobot and MountainCar. To induce non-stationarity in either task, we switched the \texttt{left} and \texttt{right} actions after half-time. The policy was initialized uniformly by setting all the action preferences to zero. The critic weights were also initialized to zero. Each agent was trained for $400k$ timesteps for Acrobot, or $200k$ timesteps for MountainCar.

      Figure \ref{fig: linear_nonstationary_combined} (left) shows the learning curves for the two estimators. For each estimator, we selected two different sets of stepsize configurations: one that had the best performance right before the non-stationarity hit, and another that had the best performance at the end of the experiment (in the figure, we denote this latter set by putting a \texttt{final} in the subscript of the corresponding estimators). For MountainCar, the alternate estimator was superior to the regular estimator for both sets of stepsizes. Remarkably, the regular estimator (\textcolor{red}{red curve}), with the best performing parameters at timestep $100k$, was unable to recover from the non-stationarity; whereas the alternate estimator (\textcolor{blue}{blue curve}) was able to recover, despite having a similar performance as the regular estimator at the end of $100k$ timesteps. On Acrobot, the alternate estimator was again superior to the regular estimator, however the difference was less evident.

      \begin{figure}[t]
        \centering
        \includegraphics[scale=0.31]{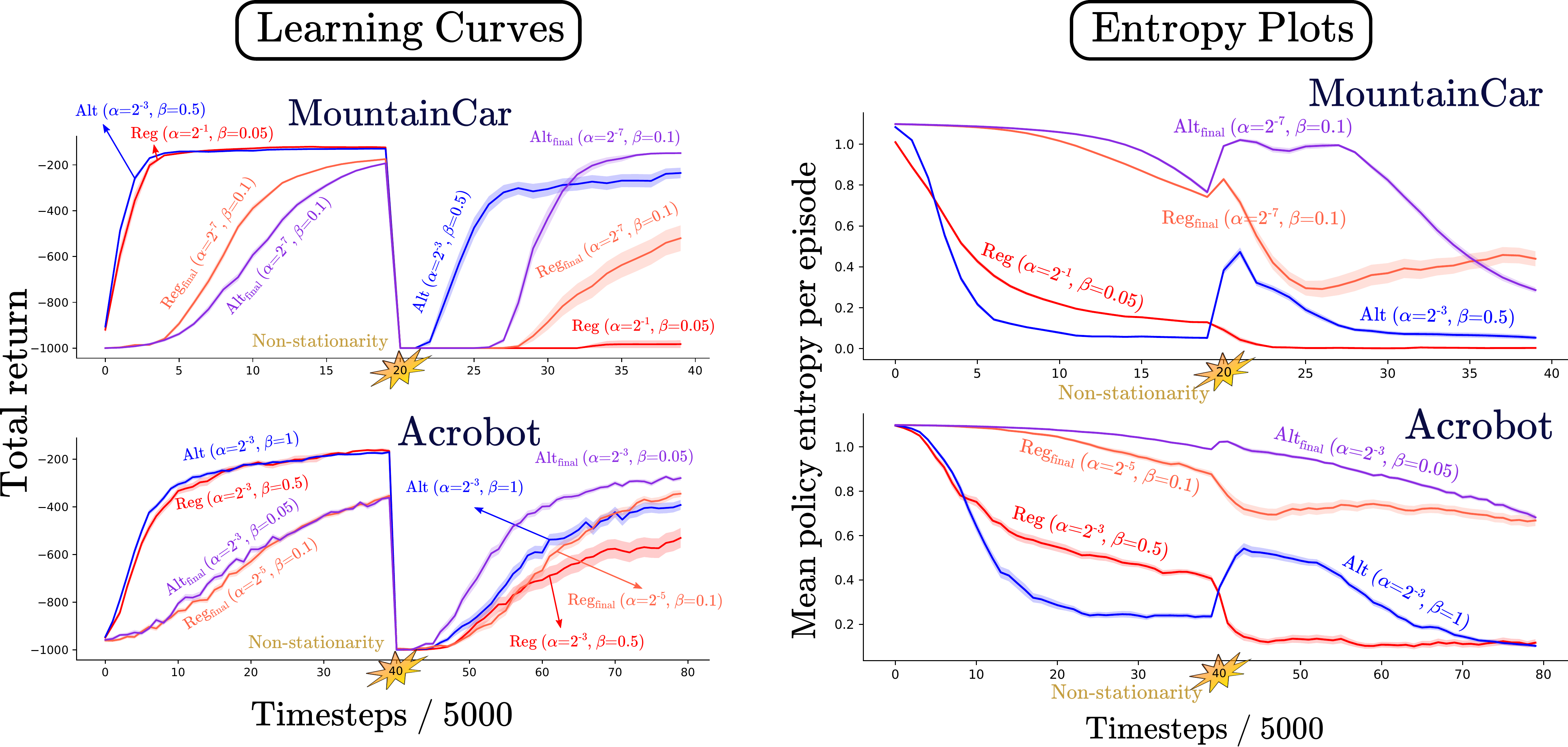}
        \caption{Learning curves and entropy plots for online linear AC with tile-coding on Acrobot and MountainCar. Both the estimators have one learning curve each for the best performing set of stepsizes right before the non-stationarity hit and at the end of learning. The right subplots show the entropy of the policy on the states encountered during each episode. All the curves show the mean values (thick line) and the standard error (shaded region) over 50 independent runs.}
        \label{fig: linear_nonstationary_combined}
      \end{figure}

      We attribute the superior performance of the alternate estimator to the bias in the critic estimate. For instance, consider what would happen in MountainCar for an online AC agent using a well performing stepsize configuration. At $100k$ timesteps, the agent would be able to learn a good policy and therefore the critic would learn to predict a return that is greater than $-200$. Now, when the non-stationarity hits, the agent's policy suddenly becomes sub-optimally saturated, and thus the agent starts receiving returns much lower than $-200$. This in turn means that at this time, the critic estimate has become optimistic. As a result, for the alternate estimator, the critic would encourage exploration by pushing the policy towards a more uniform distribution. Figure \ref{fig: linear_nonstationary_combined} (right) corroborates this point: the policy entropy for the alternate, but not for the regular, estimator has a spike right around the timestep when the non-stationarity was introduced.

      \begin{figure}[h]
        \centering
        \includegraphics[scale=0.31]{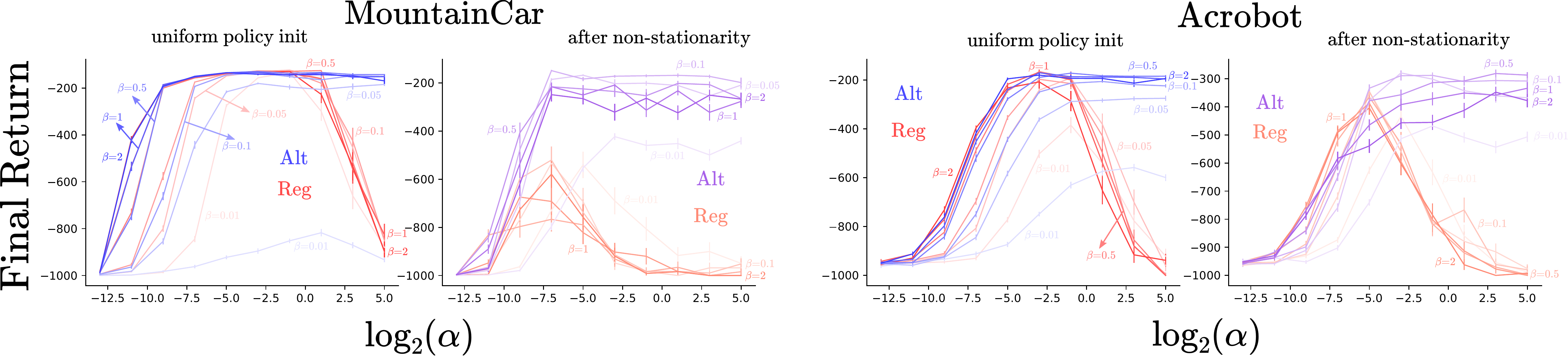}
        \caption{The sensitivity plots for online Actor-Critic with linear function approximation (using tile-coded state features) on MountainCar and Acrobot. The sensitivity plots show the mean performance during the last $5k$ timesteps, averaged over 50 independent runs.}
        \label{fig: linear_nonstationary_sensitivity}
      \end{figure}
      
      Figure \ref{fig: linear_nonstationary_sensitivity} shows the parameter sensitivity plots for these tasks. From the sensitivity plots, we see that the alternate estimator is much less sensitive to the stepsizes and achieves a superior or competitive final performance compared to the regular estimator, both before and after the non-stationarity was introduced. 

      \subsection{Online AC with Entropy Regularization and Escort Transform on Non-stationary Tasks} \label{sec: escort_entropy}
      We now study how the alternate estimator compares to entropy regularization and escort transform, two techniques that are designed to help policy gradient methods deal with policy saturation. In these experiments, we train the agent with the online Actor-Critic algorithm on the non-stationary MountainCar and Acrobot tasks (task description given in the previous section). For entropy regularization, we considered the softmax policy parameterization with either the regular or the alternate PG estimator. For the escort transform, we considered the simple online AC algorithm (without any entropy bonus). 
      
      \begin{figure}[h]
        \centering
        \includegraphics[scale=0.23]{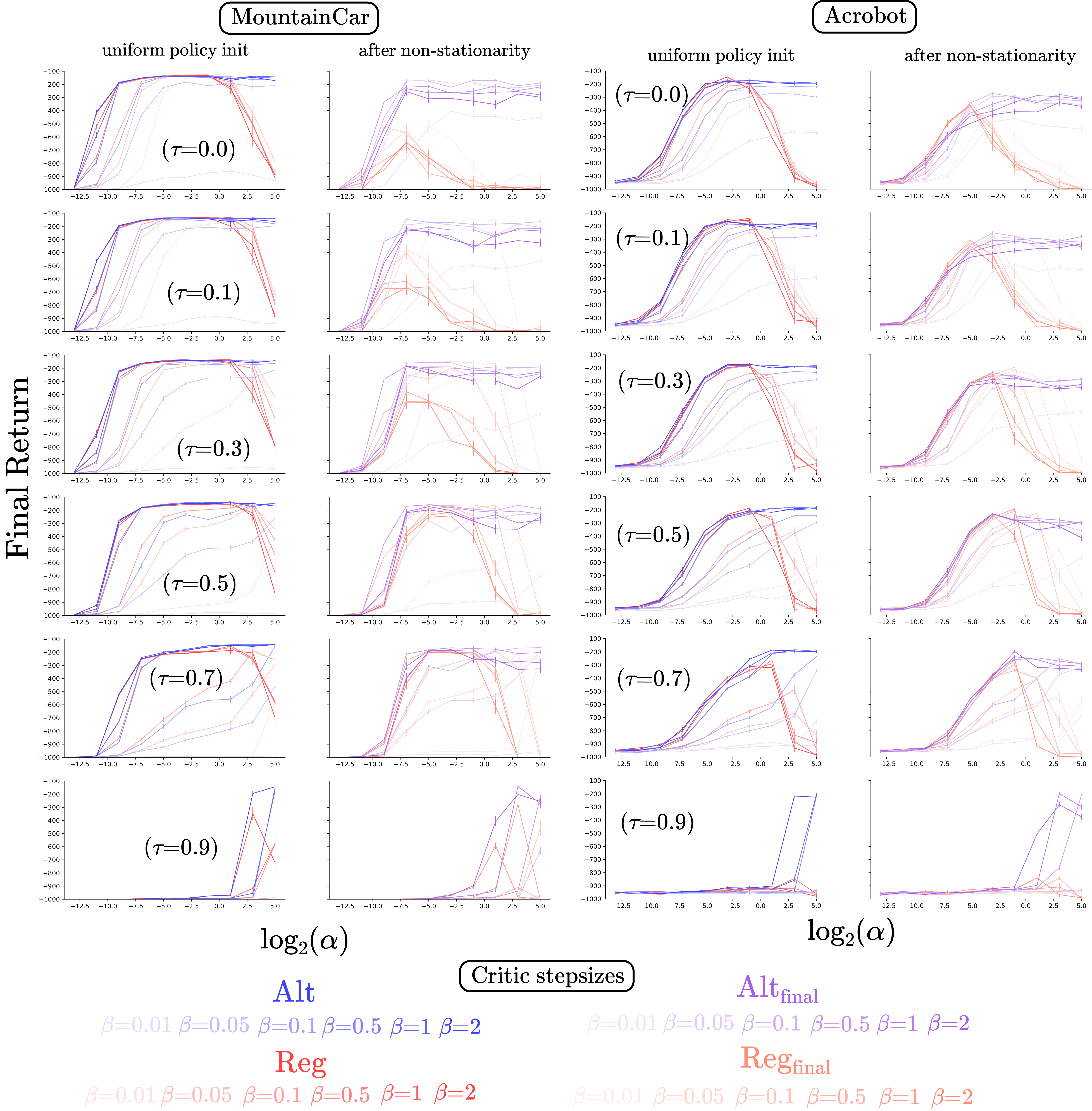}
        \caption{The sensitivity plots for online Actor-Critic with linear function approximation (using tile-coded state features) and entropy regularization on non-stationary MountainCar and Acrobot. The sensitivity plots show the mean performance during the last $5k$ timesteps, averaged over 50 independent runs. The $x$-axis refer to the policy stepsize $\alpha$ and the differently shaded lines correspond to different critic stepsizes $\beta$s. The different rows correspond to different values of the entropy regularization parameter $\tau$.}
        \label{fig: entropy_sensitivity}
      \end{figure}

      Since, both entropy regularization and escort transform have additional hyperparameters: the regularization factor $\tau$ for the first, and the escort power parameter $p$ for the latter, we swept over $\tau \in \{0.0, 0.1, 0.3, 0.5, 0.7, 0.9\}$ and $p \in \{1, 2, 3, 5, 10\}$. We initialized the value function weights to zero for both the experiments, and the policy uniformly randomly. In particular, we set the actor weights so that the initial action preference for softmax policy was $\thetavec^{\text{init}} = [0 \; 0 \; 0]^\top$ and for the escort policy was $\thetavec^{\text{init}} = [1 \; 1 \; 1]^\top$ (since the action preferences for escort transform cannot all be initialized to zero).
      
      Let us now look at the parameter sensitivity plots for this experiment. (We discussed the learning curves for this experiment in Figure \ref{fig: lc_escort_entropy} of the main paper.) Figure \ref{fig: entropy_sensitivity} shows the parameter sensitivity plots for entropy regularization. From this figure, we note that adding an entropy bonus to the alternate softmax estimator improved its performance slightly (especially for Acrobot) for small values of $\tau$. However, for larger values of $\tau$, the performance of the alternate estimator dropped. On the other hand, adding an entropy bonus to the regular estimator improved its sensitivity to the policy stepsize $\alpha$ quite a lot. Entropy regularization also significantly improved the performance of the regular estimator after non-stationarity: regular with entropy had higher final returns at the end of the experiment when compared with regular without entropy. Also, note that the alternate estimator (with or without the entropy regularization) remains superior (or competitive) to the regular estimator with entropy regularization. More importantly, recall from the learning curves (Figure \ref{fig: lc_escort_entropy}) that the entropy bonus was not able to help the best performing regular estimator (\textcolor{red}{red curve}) to recover from the non-stationarity, whereas the alternate estimator (\textcolor{blue}{blue curve}) was able to recover from it.
      
      \begin{figure}[h]
        \centering
        \includegraphics[scale=0.23]{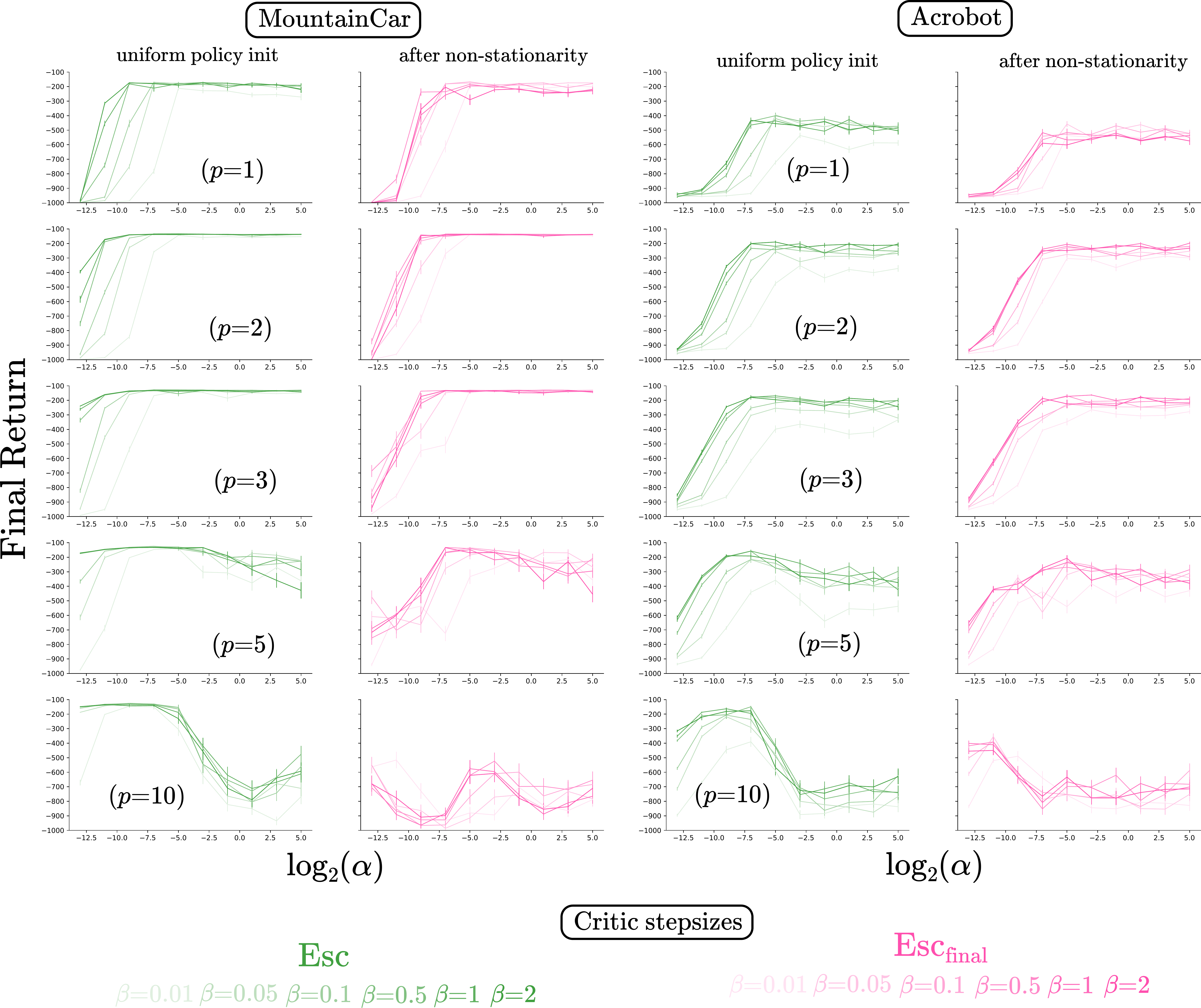}
        \caption{The sensitivity plots for escort transform with online Actor-Critic and linear function approximation (using tile-coded state features) on non-stationary MountainCar and Acrobot. The sensitivity plots show the mean performance during the last $5k$ timesteps, averaged over 50 independent runs. Different rows correspond to different values of the escort power parameter $p$.}
        \label{fig: escort_sensitivity}
      \end{figure}
      
      We show the parameter sensitivity plots for escort transform (ET) in Figure \ref{fig: escort_sensitivity}. These plots show that ET worked really well for $p=2$ and $p=3$. For these values of the escort power parameter, ET had a very low sensitivity to both the policy stepsize $\alpha$ and the critic stepsize $\beta$, and achieved a good performance for most of the parameter settings. However, for $p=1$ and $p=10$, its performance was much worse. These results suggest that ET is a great choice for handling non-stationarities in the environment. In fact, for the top performing choices of $p$, ET was able to achieve a better performance on Acrobot as compared to the alternate estimator. However, the results from the learning curves (Figure \ref{fig: lc_escort_entropy}) showed that the best performing parameter setting (before the non-stationarity hit) for ET (\textcolor{green}{green curve}) was not able to recover well from the non-stationarity. And, Figure \ref{fig: escort_sensitivity} hints that ET might be sensitive to the value of the parameter $p$, and a more extensive empirical evaluation with various PG algorithms and different environments is needed to study this phenomenon further.
      
      \subsection{Experiments with Saturated Linear Softmax Policies} \label{sec: saturated_linear_softmax_policies}
      We now consider the performance of the two PG estimators on Acrobot and MountainCar with artificially saturated policies. Our results show the same trends for alternate and regular estimators in the linear function approximation setting as we saw for the bandit and the tabular setting. For obtaining saturated policies, we initialized the action preference corresponding to the \texttt{do-nothing} action with a value of either $0, 5, 10$, or $100$ depending on the degree of policy saturation, whereas the preferences for the actions \texttt{left} and \texttt{right} were always initialized with zero. The baseline was initialized with different values ($0, +500$, or $-500$) imparting an optimistic or a pessimistic effect to the estimators. Each agent was trained for $50k$ timesteps for both the environments

      Figure \ref{fig: linear_ac_mc_learning_curve} shows the learning curves for for both the environments where the policy was initialized uniformly and the critic was initialized to zero. From this figure, we note that both the estimators were able to learn a good policy (a return of more than $-200$ is considered good on these environments). Further, the alternate estimator learned faster as compared to the regular estimator. We explain this behavior by noting that a baseline initialization of zero is effectively optimistic for a random policy (which will achieve large negative returns). Therefore, with an optimistic baseline, the alternate baseline was probably able to get to the fixed point of the biased update at a fast rate. Once it was near this fixed point, as its critic estimate improved, it slowly moved towards the optimal policy. Whereas, with the regular estimator, the agent just slowly moved towards the optimal policy, and thus had a lower sample efficiency.
      
      \begin{figure}[h]
        \centering
        \includegraphics[scale=0.45]{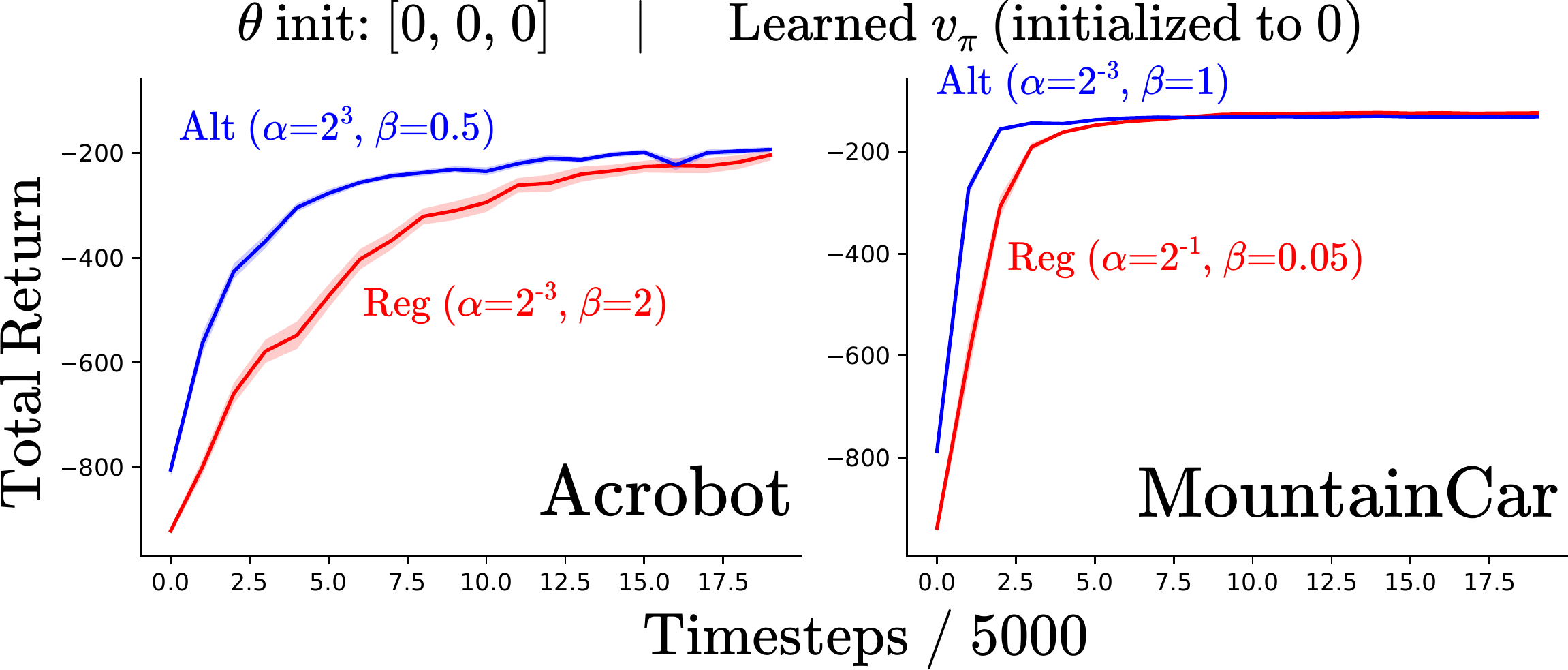}
        \caption{Learning curves for the best stepsize parameter configurations based on final performance on the two environments. The policy was initialized uniformly randomly and the baseline was initialized to zero. The curve show the mean performance and standard error over 50 runs.} 
        \label{fig: linear_ac_mc_learning_curve}
      \end{figure}
      
      We show the sensitivity plots for these estimators, with different policy saturations and baseline initializations, on both Acrobot and MountainCar in Figure \ref{fig: linear_ac_mc_stepsize_sensitivity}. The results follow similar trends that we observed for 3-armed bandits and the tabular chain environment. And impressively, these plots show that the alternate estimator can effectively escape saturated policies, even with a high policy saturation of $\thetavec^{\text{init}} = [0\;100\;0]^\top$.

      \begin{figure}[!tbp]
        \centering
        \includegraphics[scale=0.29]{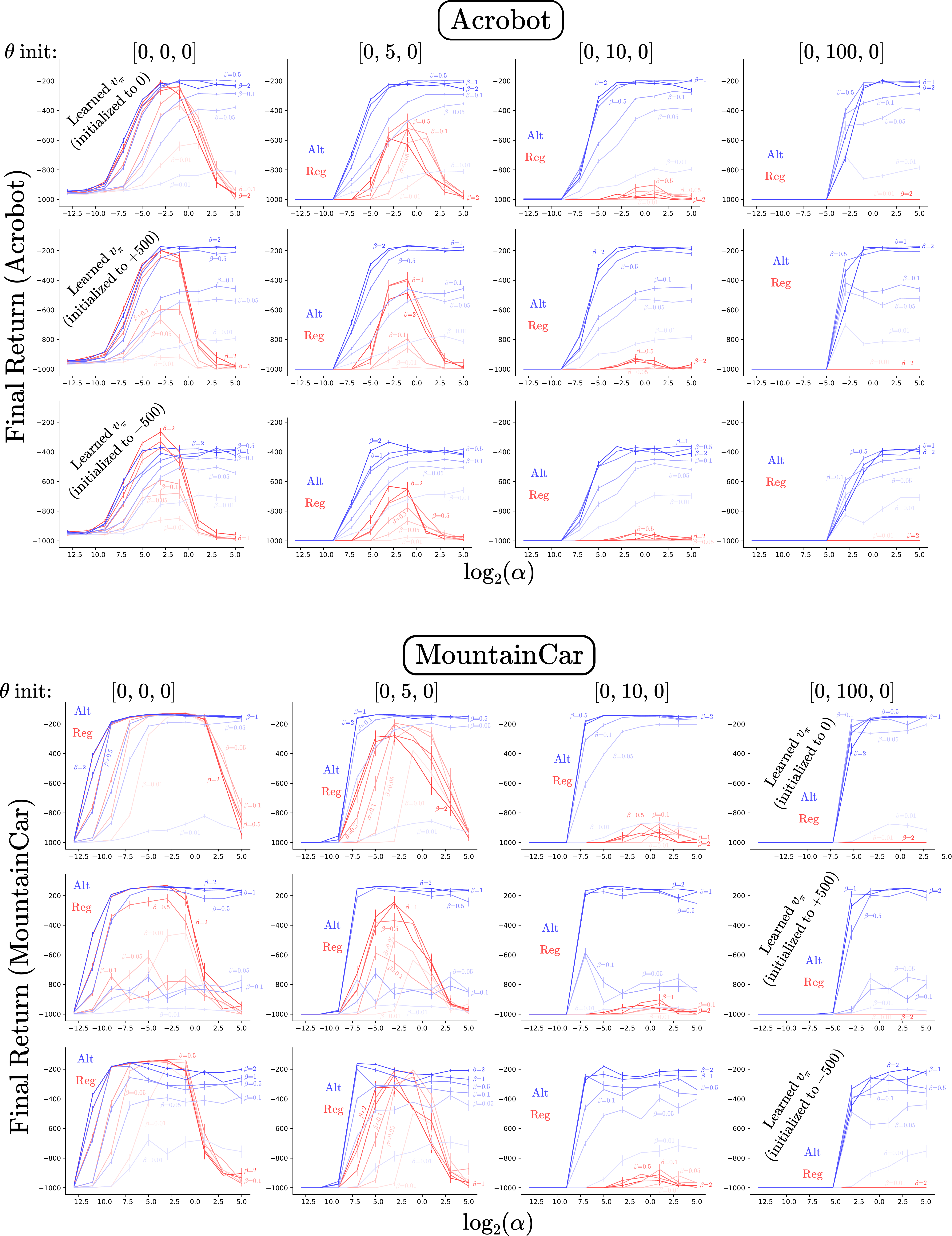}
        \caption{Parameter sensitivity plots for online AC with the two estimators on Acrobot (top) and MountainCar (bottom). The columns correspond to policies with different saturations and the rows correspond to differently initialized critic estimates. Each agent was run for $50k$ timesteps. The plots show the mean final performance during the last $5k$ steps, averaged over 50 runs.} 
        \label{fig: linear_ac_mc_stepsize_sensitivity}
      \end{figure}

      From the results in this figure, we observe that the alternate estimator had a lower parameter sensitivity and achieved a superior final performance compared to the regular estimator, for almost all the policy saturations and baseline initializations considered in this experiment; except for the (uniform policy and $b = -500$) initialization case, where the regular estimator had a better performance. We also observe that as the baseline was initialized to more negative values (to $-500$ instead of $0$), the performance of the alternate estimator became worse. The baseline initialization affects the regular estimator much less. And as the policy initialization became more saturated, the performance of both the methods worsened. Another thing to note is that the regular estimator had a lower sensitivity to the critic stepsize $\beta$; the alternate estimator performed poorly for lower critic stepsizes than it did for higher critic stepsizes. We believe that the reason for this is that with lower critic stepsizes, the critic takes more time to track the value function due to which the alternate estimator remains biased for longer; and this probably affects the final performance in a negative way. We make one final comment about this experiment: the initializations ($0, +500$, and $-500$) are all effectively optimistic for a saturated policy (which gets a return of less than $-1000$); this probably explains why the alternate estimator had such a good final performance, for most settings, in both the tasks.
      
      \subsection{Experiments on RealisticMountainCar}
      We now consider the performance of the two estimators on a difficult variant of the usual MountainCar environment which we call the RealisticMountainCar\footnote{RealisticMountainCar was proposed by Andy Patterson. We used the implementations available here: \url{https://github.com/andnp/PyRlEnvs/tree/main/PyRlEnvs/domains}. In this repository, MountainCar is referred to as GymMountainCar and RealisticMountainCar is referred to as MountainCar.}. Unlike the original environment that had an instantaneous dynamics model (\S 4.3, Moore, 1990), i.e. an acceleration applied to the car instantaneously changed its velocity, RealisticMountainCar introduces a delay between when the acceleration is applied and when the car's velocity increases, thereby making the dynamics more realistic. Moreover, as we find from our experiments, this delay makes the environment more difficult to solve for the regular PG estimator. The state-action space for this environment remains the same as the original MountainCar environment. The discount factor used was $\gamma = 1$ and the episode was cutoff after 1000 timesteps if the agent was unable to solve it by then. We performed sweeps over the policy stepsize $\alpha \in \{2^{-8}, 2^{-6}, \ldots, 2^{16}\}$ and the critic stepsize $\beta \in \{0.01, 0.05, 0.1, 0.5, 1, 2\}$. All the agents were trained for $100k$ timesteps and 50 independent runs.
      
      \begin{figure}[!hbp]
        \centering
        \includegraphics[scale=0.4]{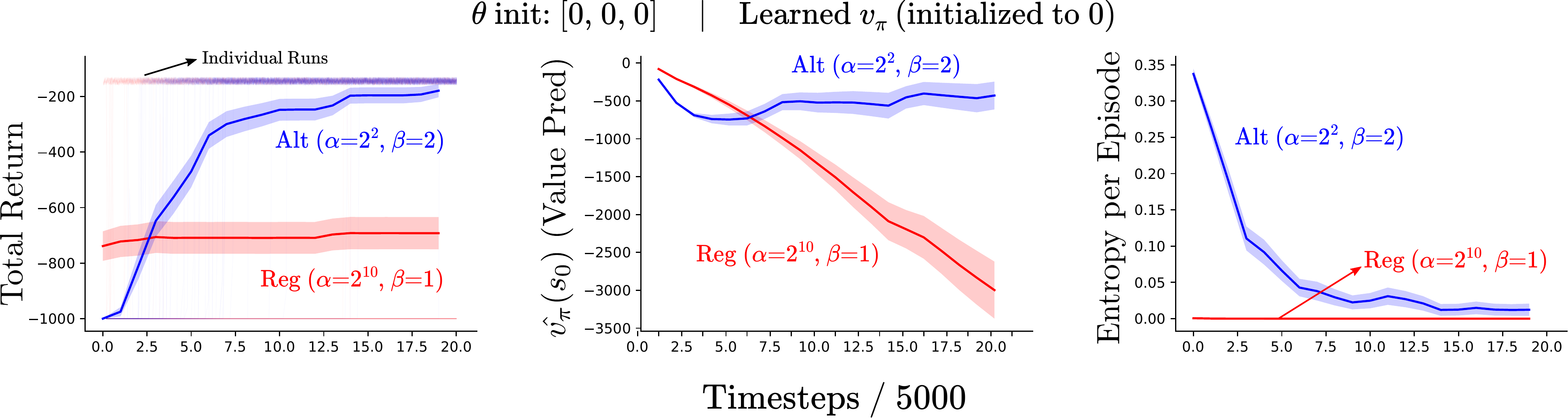}
        \caption{Learning curves for online AC on the RealisticMountainCar task for the best performing parameters. We also plot the value estimate of the initial state of the episode, and the mean policy entropy on the states encountered during different episodes. Even though the alternate estimator was able to learn an effective policy, the regular estimator failed to do so. Note that the regular estimator achieved its best performance at a very high stepsize that almost instantaneously pushed the policy in a highly saturated sub-optimal region; hence the learning curve for the regular estimator is flat (corresponding to an unchanging policy) and the entropy plot for the regular estimator has zero entropy throughout the learning period.} 
        \label{fig: linear_realmc_learning_curve}
      \end{figure}

      Figure \ref{fig: linear_realmc_learning_curve} shows the performance curves, the value function estimate of the initial states, and the policy entropy for the two estimators as the agent learned. From these results, we see that the alternate estimator was effectively able to solve this task, whereas the regular estimator was unable to learn a good policy. From the individual runs for the regular estimator, we see that some of them resulted in a good policy (achieving a return of $>-200$), whereas most of them didn't learn (achieving the maximum possible return of $-1000$), giving it a net average return of about $-700$. We also included the plots for $\hat{v}_\pi(s_0)$ where $\hat{v}_\pi$ is the critic estimate and $s_0$ is the start state observed at the beginning for each episode. From these plots, we see that the value function of the regular estimator diverges to negative infinity because we used a discount factor $\gamma = 1$ and the agent, for most runs, was unable to solve this task within 1000 timesteps.

      We show the sensitivity plots for the mean final performance of the two estimators during the last $5k$ timesteps in Figure \ref{fig: linear_realmc_stepsize_sensitivity}. These results show that the alternate estimator learns a good policy for a large range of stepsizes. Most curiously, for this experiment, the alternate estimator with a baseline initialization of $+500$ had a worse performance than the initialization of $0$ and $-500$. We attribute this result to the fact that overcoming an optimistic initialization of $+500$ takes the longest time (during which the estimator remains biased). Further, since $0$ and $-500$ are already optimistic for this problem, they had a good performance as well. This observation leads us to conclude that the critic for the alternate estimator should be initialized optimistically, but not too optimistically.

      \begin{figure}[t]
        \centering
        \includegraphics[scale=0.4]{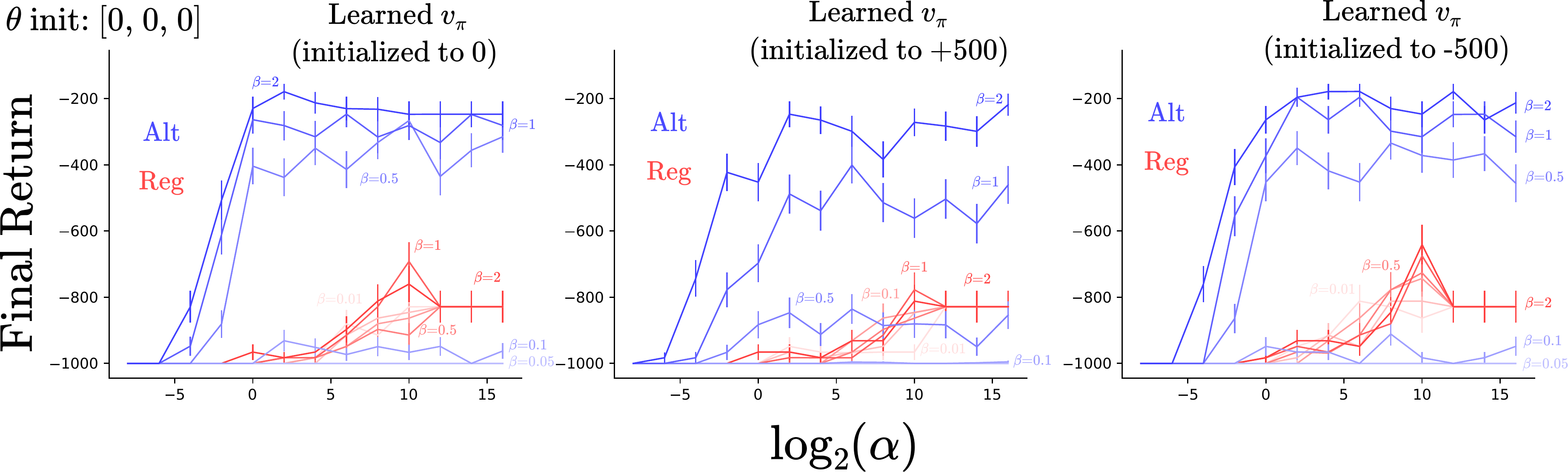}
        \caption{Parameter sensitivity of different PG estimators with online AC on RealisticMountainCar. The policy was initialized uniformly, and the value function was initialized to different value in each column. We see that the alternate estimator is fairly robust across parameters whereas the regular estimator exhibits poor performance consistently across different parameter settings.} 
        \label{fig: linear_realmc_stepsize_sensitivity}
      \end{figure}

      \section{Online AC with Neural Networks} \label{app: experiments_neural}
      In this section, we study the online AC algorithm with a neural policy and a neural critic on an environment called DotReacher. We again experimented separately with artificially saturated policy initializations and a non-stationary environment. In either case, we found that the results were similar to the linear function approximation case. We now specify the experimental details and present the empirical results.
      
      \subsection{Details about Experiments with Neural Networks}
      DotReacher\footnote{DotReacher was proposed by Rupam Mahmood. We used the implementation available here: \url{https://github.com/svmgrg/rl_environments/tree/main/DotReacher}.}, shown in Figure \ref{fig: environments} is a grid world environment where the agent is spawned uniformly randomly in a $2 \times 2$ unit squared arena. For the stationary version of DotReacher, the goal is to reach the $(0, 0)$ state within a tolerance of $0.1$. The agent has 9 different actions: eight corresponding to eight different directions and one for no movement. Each action deterministically moves the agent $0.03$ units in the corresponding direction. The agent gets a reward of $-0.01$ until it reaches the goal state. We fixed $\gamma = 1$ and the episode cutoff length as 1000. For the non-stationary variant, we initially set the goal state to $(-1, -1)$ and then shifted it to $(+1, +1)$ after $50k$ steps. 

      For both the experiments involving the stationary the non-stationary DotReacher, we swept the policy stepsize $\alpha \in \{2^{-17}, 2^{-15}, \ldots, 2^{-3}\}$, the critic stepsize $\beta \in \{2^{-17}, 2^{-15}, \ldots, 2^{-7}\}$, and used the Adam optimizer (Kingma and Ba, 2014) with the default PyTorch (Paszke et al., 2019) configuration\footnote{Refer here: \url{https://pytorch.org/docs/stable/generated/torch.optim.Adam.html}.}. We maintained two separate neural networks  (both with \texttt{ReLU}) for the policy and the critic: the policy network dimensions were $2 \times 10 \times 10 \times 9$ and the critic network dimensions were $2 \times 10 \times 10 \times 1$. For saturating the policy, we increased the action preference corresponding to the no movement action by modifying the bias units of the last layer of the neural network.

      \subsection{Non-stationary DotReacher} \label{sec: non_stationary_dotreacher}
      Figure \ref{fig: neural_nonstationary_dotreacher_stepsize_sensitivity} shows the learning curves and the stepsize sensitivity plots for the online AC algorithm with the two estimators on the DotReacher environment. The policy was initialized uniformly and the critic estimate was initialized to zero. The results follow the same pattern as the experiments on the non-stationary Acrobot and MountainCar. The alternate estimator is still competitive or superior than the regular estimator, although the difference is less pronounced. 
      
      \begin{figure}[t]
        \centering
        \includegraphics[scale=0.35]{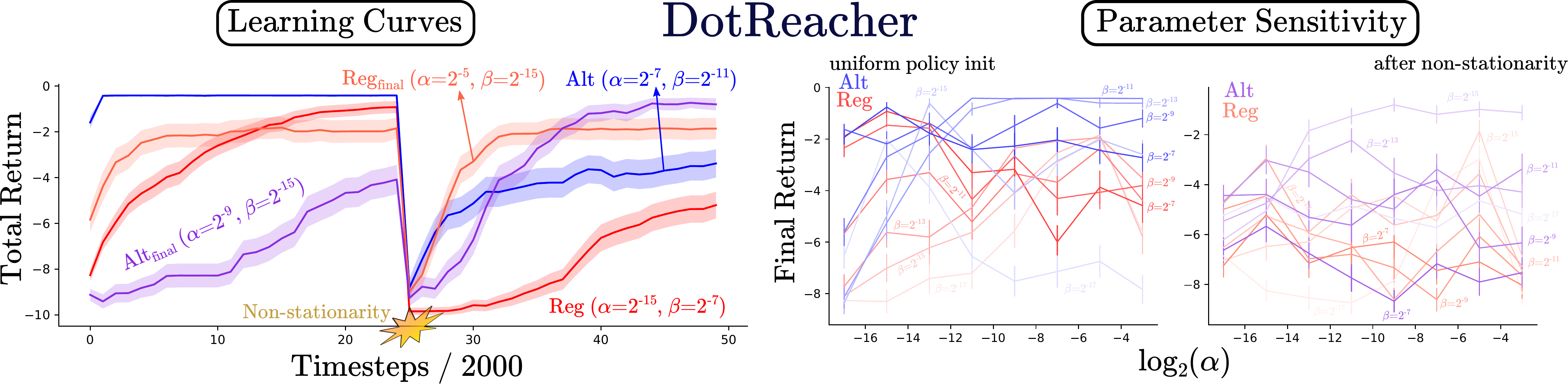}
        \caption{\textbf{(Left)} Learning curves for online AC with neural networks on DotReacher ($100k$ timesteps, 50 runs) for two sets of best performing parameter configuration: one that obtained the best performance right before the non-stationarity occurred, and the other that obtained the best final performance at the end of learning. \textbf{(Right)} The parameter sensitivity plots showing the mean final performance (during the last $2k$ timesteps) at the end of $50k$ timesteps and the mean final performance at the end of $100k$ timesteps. The mean final performance was computed by averaging over the last 2000 timesteps either at half-time or at the end of training.} 
        \label{fig: neural_nonstationary_dotreacher_stepsize_sensitivity}
      \end{figure}
      
      \subsection{Experiments with Saturated Neural Policies on DotReacher} \label{sec: opti_pessi_neural_pg}
      We now consider the experiments with artificially induced policy saturations on the DotReacher environment. Figure \ref{fig: neural_dotreacher_stepsize_sensitivity} shows the sensitivity plots for both the estimators for different policy saturations and critic initializations. In this problem, we consider a return of more than $-1$ as being good. The results are same as those discussed in \S \ref{sec: saturated_linear_softmax_policies}: the alternate estimator had a superior performance and exhibited less sensitivity to the stepsize $\alpha$ as compared to the regular estimator. Interestingly, for these results, we observe that for lower amounts of policy saturations ($\theta_{\text{no-action}} \in \{5, 10\}$), the regular estimator could learn policies almost as good as the alternate estimator. But for the highest saturation case, only the alternate estimator was able to escape the saturation and obtain high returns. We also see that an optimistic baseline initialization (for this problem, $b = 0$ and $b = +25$) helped the alternate estimator achieve a good performance, whereas the pessimistic baseline initialization ($b = -25$) made the performance of the alternate estimator even worse than the regular estimator (except in the case of super high policy saturation).
      
      \begin{figure}[h]
        \centering
        \includegraphics[scale=0.3]{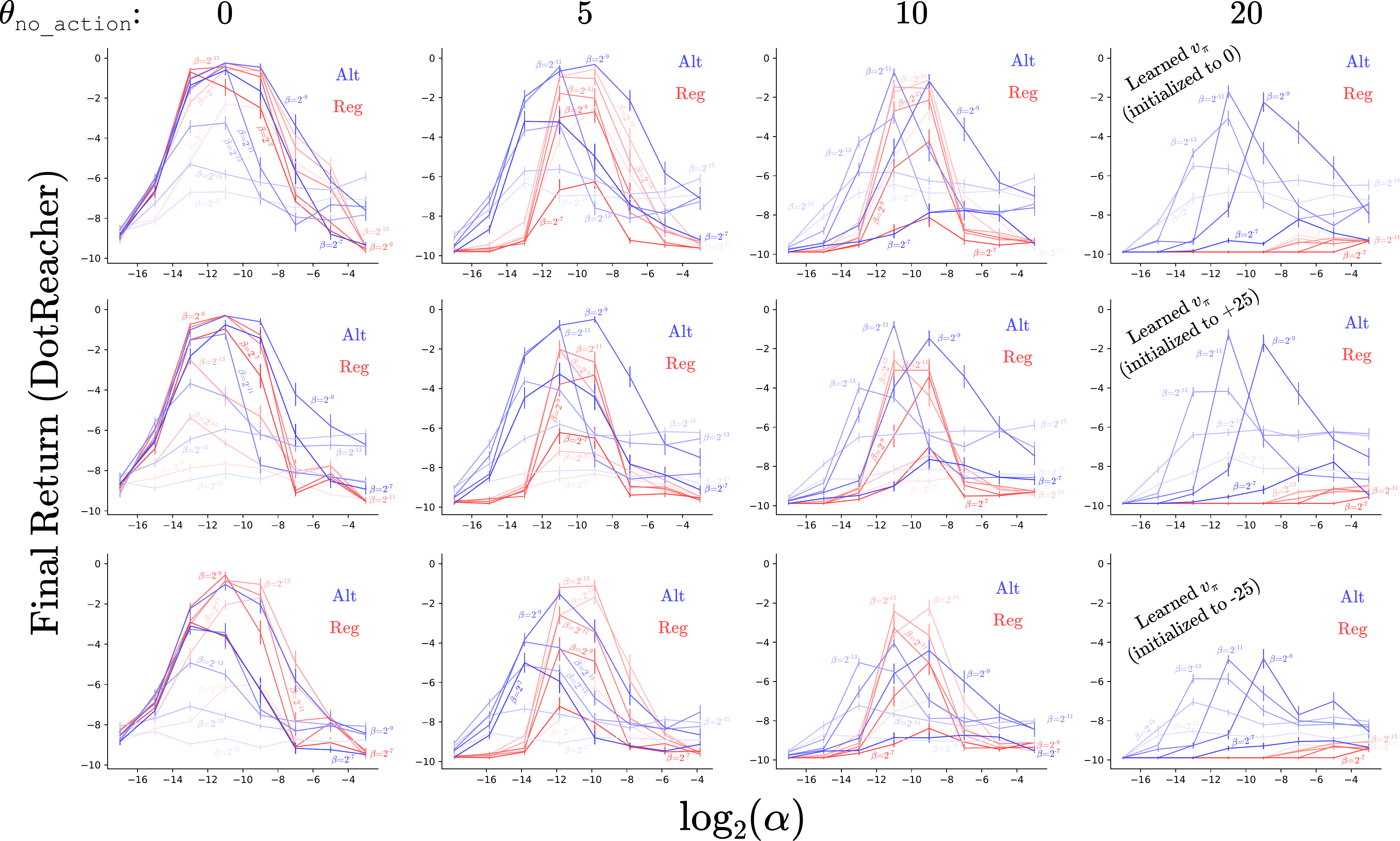}
        \caption{Online AC + neural networks with different baseline initializations and policy saturations on the DotReacher task. The plots show the mean final performance during the last $5k$ timesteps of the agent's $50k$ timestep lifetime. The figure header shows the action preference initialization for the no movement action; the rest of the preferences were initialized to zero. Different rows show differently initialized value function estimates ($0$, $+25$, or $-25$). We observe that for almost all the settings, the alternate estimator performed better than the regular estimator in recovering from saturated policies, especially for higher saturation magnitudes.} 
        \label{fig: neural_dotreacher_stepsize_sensitivity}
      \end{figure}

      \section{A Digressive Remark: Softmax Policies Learn even with Extremely Large Stepsizes}
      In this section, we mention a peculiar observation: online AC with the alternate estimator was able to learn a good policy even with extremely large stepsizes. The reader might have noticed from the parameter sensitivity plots presented in this section that the alternate estimator has a fairly good performance for a big range of the policy stepsize $\alpha$. In particular, the sensitivity plots, given in Figure \ref{fig: linear_ac_mc_stepsize_sensitivity}, for Acrobot and MountainCar show that for the stepsize range we tried (with the maximum value of $\alpha$ being $2^5 = 32$), the alternate estimator had a good performance throughout. More interestingly, for RealisticMountainCar (sensitivity plots given in Figure \ref{fig: linear_realmc_stepsize_sensitivity}), we tried stepsizes as large as $2^{16} = 65536$, and the alternate estimator still received a fairly high final return (only slightly lower than the what the alternate estimator with the best parameter setting achieved and much higher than what the regular estimator received). Even if we take into account the normalization factor from tile-coding (which was 9), the effective stepsize still remains quite high: $\nicefrac{2^{16}}{9} = 7282$. To put this number in perspective, the default choice of the stepsize for training neural networks with the Adam optimizer (Kingma and Ba, 2014) is $0.0003$. One might wonder, whether this only happens for online AC with linear function approximation setting, and the answer is no. We show the stepsize sensitivity plots for REINFORCE on the chain environment (the same experiment as discussed in \S \ref{sec: reinforce_chain_different_policy_saturations_exp} and \S \ref{sec: opti_pessi_pg_experiments}) with a wide range of the policy stepsize in Figure \ref{fig: tabular_00_v0_large}. And these plots display a similar pattern: the alternate estimator learns for extremely large values of the policy stepsize as well. The first two rows of Figure \ref{fig: tabular_00_v0_large} are exactly the same as Figure \ref{fig: multiple_bad_init_sensitivity} except that it includes a larger sweep for the policy stepsize $\alpha \in \{2^{-6}, 2^{-4}, \ldots, 2^{14}\}$. Similarly, the third and fourth rows are identical to Figure \ref{fig: optimistic_linear_chain_sensitivity} and Figure \ref{fig: pessimistic_linear_chain_sensitivity} respectively. We verified that none of these results were due to bugs in our experiments and now discuss some plausible explanations for this behavior.

      \begin{figure}[t]
        \centering
        \includegraphics[scale=0.33]{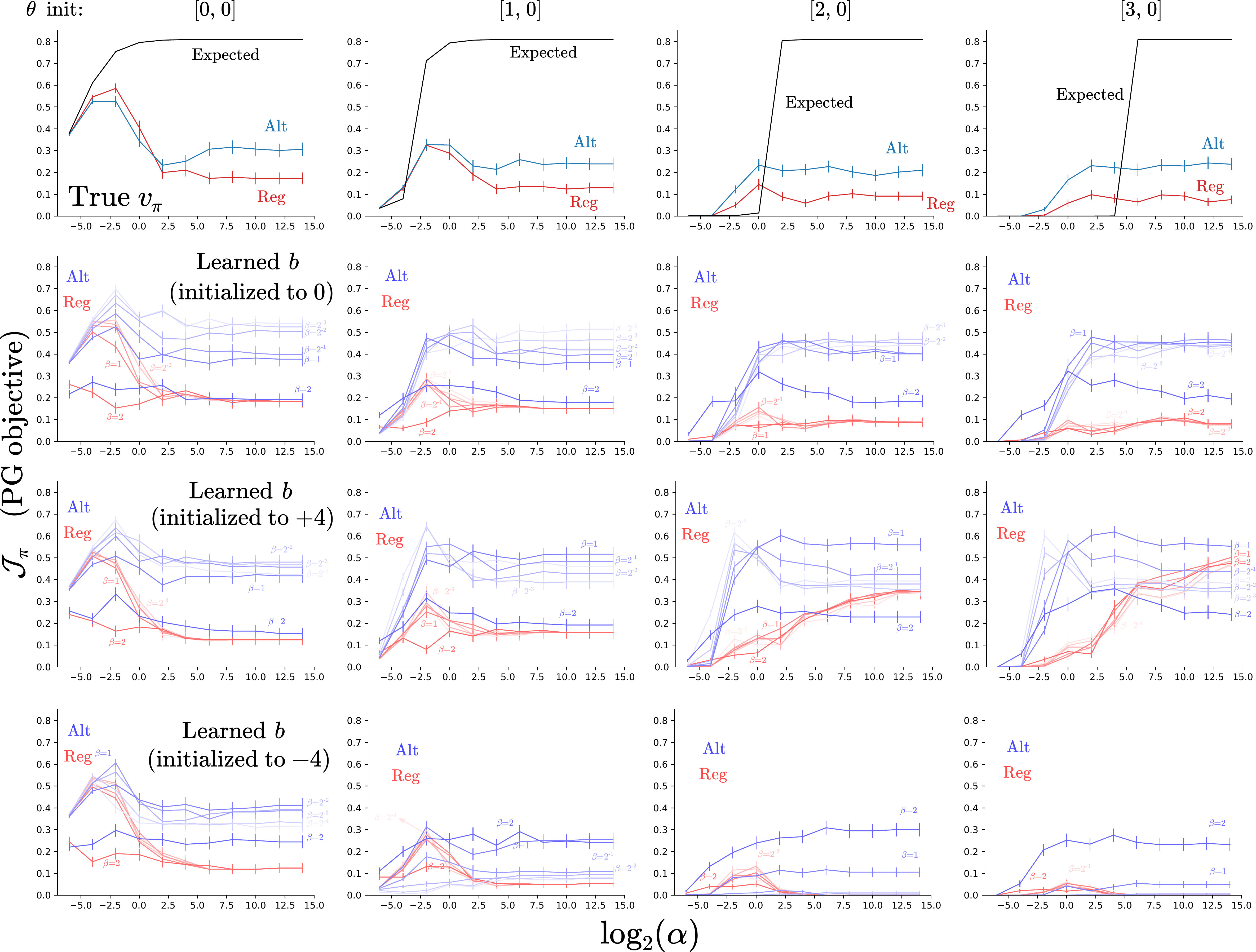}
        \caption{Sensitivity plots showing the final performance of \textbf{tabular} REINFORCE on the chain environment. The different columns correspond to an increasing degree of policy saturation and the figure headers show the action preference initializations for the \texttt{left} and \texttt{right} actions respectively. The first row shows the results for the true baseline $v_\pi$, and the next three rows show results for a learned baseline, which was initialized to either zero, $+4$, or $-4$ at the beginning of learning.} 
        \label{fig: tabular_00_v0_large}
      \end{figure}

      We ascribe these observations to the nature of softmax function and the policy optimization problem over the action preferences. Note that (1) the softmax function transforms any set of action preferences, irrespective of how large they are, to a proper probability distribution; and (2) for softmax policies, the optimal point which maximizes the PG objective, lies at infinity on the action preference landscape. The second point is true because for MDPs, the optimal policy is greedy, which requires the action preference for the greedy action to be infinitely large. This makes the optimization problem over the action preferences very different from, say, an unconstrained convex optimization problem. For convex functions, an update step with a high stepsize, even in the correct direction, can make the function value worse than its present value. In contrast, a larger stepsize with softmax functions will only drive the function value closer to the optimum at a faster rate. We illustrate this point in Figure \ref{fig: softmax_optim}. Therefore, as long as the update direction is correct, having large stepsizes should not be a concern for softmax policies.
      
      \begin{figure}[!tbp]
        \centering
        \includegraphics[scale=0.6]{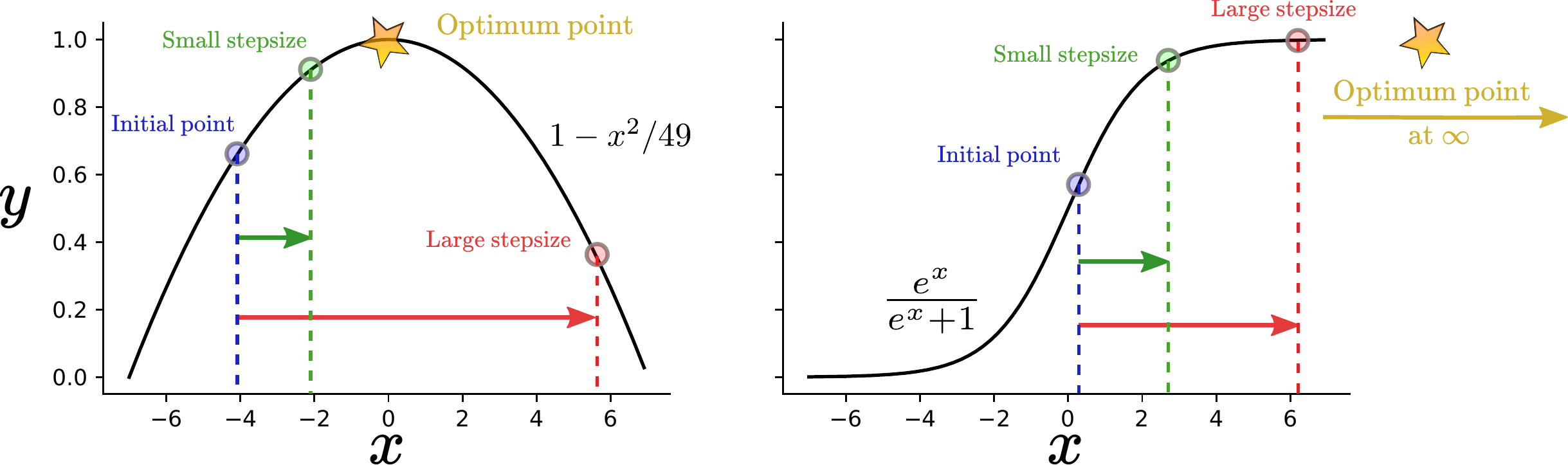}
        \caption{Effect of the stepsize magnitude on optimizing a function value. We consider the optimization problem on a convex function and the logistic function. Note that the logistic function $y = \frac{e^x}{e^x + 1}$ is essentially equivalent (except for a translation along the $x$-axis) to the softmax function $\pi(a) = \frac{e^{\theta_a}}{e^{\theta_a} + \sum_{c \in \mathcal{A} \setminus {a}} e^{\theta_c}}$. This figure shows that given the correct update direction (positive direction for both of these problems), an update step with a very large stepsize on the convex function can make the function value worse. Whereas, in the case of a softmax function, the larger the stepsize, the more the function value increases with a single update step.} 
        \label{fig: softmax_optim}
      \end{figure}
      
      For concreteness, let us consider the stepsize sensitivity behavior of expected PG in Figure \ref{fig: tabular_00_v0_large}: a larger stepsize for expected PG update resulted in a faster convergence to the optimal policy and consequently a higher final return. From gradient based optimization literature, we know that the expected (policy) gradient is the the update direction that results in the fastest local improvement in the (policy) objective, and therefore using the reasons given in the previous paragraph, we would expect a large stepsize to drive the policy towards this local optimum faster. We can also observe this trend for the expected gradient bandit algorithm. Figure \ref{fig: intuition_bandits_reward123_noise1} (bottom-left) shows that for a bandit problem, a large policy update stepsize should still make the policy converge towards the optimal (top) corner, albeit via a different route. But for that bandit problem, there is only one optimum that is also globally optimal. Also, intuitively, the policy gradient update in the (tabular) MDP setting can be considered as (kind of) making separate gradient bandit updates for each state.

      However, due to the sampling noise, the stochastic PG estimators might not always give an update direction that points towards the local optimum. Consequently, a very large stepsize could saturate the policy in a sub-optimal manner. We posit two reasons for why the alternate estimator might be better suited to handle very large policy stepsizes than the regular estimator. The first reason is that the alternate estimator can handle and escape sub-optimally saturated policies much better than the regular estimator. With very large stepsizes, the action preferences for both the estimators would be driven to large values and the policy would becomes saturated. For such cases, the regular gradient estimator becomes vanishingly small: $\hat{\gvec}^{\text{REG}}(S, A) = (R - b) (\evec_A - \pivec) \approx 0$, since near the corners $\evec_A \approx \pivec$ with a high probability. Therefore, even with a very large stepsize $\alpha$, the net update to the action preferences $\thetavec^{\text{new}} = \thetavec^{\text{old}} + \alpha \hat{\gvec}^{\text{REG}}(S, A)$ remains small, and the policy doesn't change much. In contrast, for the alternate estimator $\hat{\gvec}^{\text{ALT}}(S, A) = (R - b) \evec_A$, irrespective of the degree of policy saturation, the estimator remains non-zero, which when combined with a large stepsize can be effective in changing even very large action preferences. Although this explanation used the gradient bandit estimators, similar arguments should hold for the MDP policy gradient estimators as well.

      The second point is that, since the alternate estimator updates the preference corresponding to only the sampled action, there is much less generalization in the preference values across actions when using the alternate estimator. In contrast, the regular estimator updates all the action preferences and therefore has a high generalization across the action space. This distinction becomes important when the agent updates the policy using an incorrect direction and a large stepsize. With the alternate estimator, only some action preferences for some states get saturated, whereas with the regular estimator, all the preferences for a set of states get saturated. This generalization argument is also supported by the experiments presented in this document. With large stepsizes, the performance of the estimators, especially the regular version, becomes worse as we increase the generalization of the policies: compare the performance of the two estimators in Figure \ref{fig: tabular_00_v0_large} (tabular policy: no generalization); Figures \ref{fig: linear_ac_mc_stepsize_sensitivity} and \ref{fig: linear_realmc_stepsize_sensitivity} (linear policy with tile-coding: generalization across 8 tilings); and Figure \ref{fig: neural_dotreacher_stepsize_sensitivity} (neural policy: high generalization). With neural networks, none of the estimators  performed well with large stepsizes because updating even a single action preference (in the wrong direction) changes all the weights lying in the initial layers of the network, ultimately affecting all the action preferences.

      One experiment to test the above hypothesis could be to run online AC with both the estimators on Acrobot and MountainCar with a tabular policy. This could be achieved with a tile-coding that has super-fine meshes (that is a large number of tiles per dimension) and only a single tiling. If the performance of the regular estimator with high stepsizes improves as compared to Figure \ref{fig: linear_ac_mc_stepsize_sensitivity}, then that would serve as a strong verification for this hypothesis. Unfortunately, a detailed investigation of this phenomenon is out of scope of this work, and therefore we do not pursue it any further.

\end{document}